\documentclass[12pt,english]{article}
\usepackage{amsmath,mathtools,amssymb,mathrsfs,dsfont,amsthm}
\usepackage[margin=1in]{geometry}
\usepackage{algpseudocode}
\usepackage{algorithm}
\usepackage{appendix}
\usepackage{caption}
\usepackage{cite}
\usepackage[T1]{fontenc}
\usepackage{babel}
\usepackage{graphicx}
\usepackage{float}
\usepackage{color}
\usepackage{subcaption}
\usepackage[colorlinks]{hyperref}
\hypersetup{citecolor=blue}
\usepackage{enumitem}
\usepackage{authblk}
\usepackage{lineno}
\usepackage[normalem]{ulem}
\usepackage{placeins}

\newcommand{\numtype}{section}
\relax

\newtheorem{theorem}{Theorem}[\numtype]
\newtheorem{lemma}[theorem]{Lemma}
\newtheorem{proposition}[theorem]{Proposition}

\theoremstyle{remark}
\newtheorem{remark}[theorem]{Remark}
\newtheorem{example}[theorem]{Example}

\numberwithin{equation}{section}

\providecommand{\keywords}[1]{\textbf{\textit{\noindent  Keywords: }} #1}

\definecolor{darkred}{rgb}{0.6,0.1,0.1}
\definecolor{darkgreen}{rgb}{0.1,0.6,0.1}
\definecolor{darkblue}{rgb}{0.1,0.1,0.6}
\def\toTLp{\stackrel{\TLp}{\to}}
\def\weakstarto{\stackrel{*}{\rightharpoonup}}
\def\dd{\mathrm{d}}
\def\Id{\mathrm{Id}}
\def\bbR{\mathbb{R}}
\def\cL{\mathcal{L}}
\def\cP{\mathcal{P}}

\def\diag{\mathrm{diag}}
\def\KL{\mathrm{KL}}
\def\Gra{\mathrm{Gra}}
\def\TLp{\mathrm{TL}^p}
\def\LTLp{\mathrm{L}\TLp}
\def\dTLp{d_{\TLp}}
\def\Lp{\mathrm{L}^p}
\def\Wkp{\mathrm{W}^{k,p}}
\def\dWp{d_{\mathrm{W}^p}}
\def\Wp{\mathrm{W}^p}
\def\dWinfty{d_{\mathrm{W}^\infty}}
\def\dLOT{d_{\LWp,\sigma}}
\def\dLTLp{d_{\mathrm{L}\TLp,(\sigma,h)}}
\def\l{\left(}
\def\r{\right)}
\def\lb{\left\{}
\def\rb{\right\}}
\def\eps{\varepsilon}

\def\TWkp{\mathrm{TW}^{k,p}}
\def\TLtwo{\mathrm{TL}^2}
\def\Div{\mathrm{div}}
\def\n{\mathrm{\mathbf{n}}}
\def\RR{\mathrm{RR}}
\def\MR{\mathrm{MR}}
\def\SPY{\mathrm{SPY}}
\def\esssup{\mathrm{ess}\,\mathrm{sup}}
\def\rmR{\mathrm{R}}
\def\rmf{\mathrm{f}}
\def\cN{\mathcal{N}}
\def\mean{\mathrm{mean}}
\def\PnL{\mathrm{PnL}}
\def\stdev{\mathrm{stdev}}
\def\PPT{\mathrm{PPT}}
\def\COR{\mathrm{COR}}
\def\Corr{\mathrm{Corr}}
\def\Wtwo{\mathrm{W}^2}
\def\LWp{\mathrm{L}\Wp}

\usepackage{multirow,bigdelim}
\usepackage{array}     

\newcommand{\sign}{\mathrm{sign}}

\title{A Linear Transportation $\Lp$ Distance for Pattern Recognition}

\author[1,2]{Oliver M. Crook\thanks{omc25@cam.ac.uk}}
\author[3,4]{Mihai Cucuringu}
\author[5]{Tim Hurst}
\author[2,4]{Carola-Bibiane Sch\"{o}nlieb}
\author[6]{Matthew Thorpe\thanks{matthew.thorpe-2@manchester.ac.uk}}
\author[4,5]{Konstantinos C. Zygalakis}
\affil[1]{Department of Biochemistry,\protect\\ University of Cambridge,\protect\\ Cambridge, CB2 1GA, UK \vspace{\baselineskip}}
\affil[2]{Department of Applied Mathematics and Theoretical Physics,\protect\\ University of Cambridge,\protect\\ Cambridge, CB3 0WA, UK \vspace{\baselineskip}}
\affil[3]{Department of Statistics,\protect\\ University of Oxford,\protect\\ Oxford, OX1 3LB, UK \vspace{\baselineskip}}
\affil[4]{The Alan Turing Institute,\protect\\ London, NW1 2DB, UK \vspace{\baselineskip}}
\affil[5]{School of Mathematics,\protect\\ University of Edinburgh,\protect\\ Edinburgh, EH9 3FD, UK \vspace{\baselineskip}}
\affil[6]{Department of Mathematics,\protect\\ University of Manchester,\protect\\ Manchester, M13 9PL, UK}

\newcommand\Mark[1]{\textsuperscript{#1}}

\date{September 2020}

\begin{document}


\begingroup
\centering
{\LARGE A Linear Transportation $\Lp$ Distance for Pattern Recognition \\[1.5em]
\large Oliver M. Crook\Mark{1,2}, Mihai Cucuringu\Mark{3,4}, Tim Hurst\Mark{5}, Carola-Bibiane Sch\"{o}nlieb\Mark{2,4}, Matthew Thorpe\Mark{6}, Konstantinos C. Zygalakis\Mark{4,5}}\\[1em]
\begin{tabular}{*{2}{>{\centering}p{.45\textwidth}}}
\Mark{1}Department of Biochemistry, & \Mark{2}Department of Applied Mathematics \tabularnewline
University of Cambridge, & and Theoretical Physics, \tabularnewline
Cambridge, CB2 1GA, UK & University of Cambridge, \tabularnewline
 & Cambridge, CB3 0WA, UK
\end{tabular}\\[1em]
\begin{tabular}{*{2}{>{\centering}p{.45\textwidth}}}
\Mark{3}Department of Statistics and Mathematical Institute,  & \Mark{4}The Alan Turing Institute, \\  London, NW1 2DB, UK \tabularnewline 
University of Oxford, &    \tabularnewline
Oxford, OX1 3LB, UK & 
\end{tabular}\\[1em]
\begin{tabular}{*{2}{>{\centering}p{.45\textwidth}}}
\Mark{5}School of Mathematics, & \Mark{6}Department of Mathematics, \tabularnewline
University of Edinburgh, & University of Manchester, \tabularnewline
Edinburgh, EH9 3FD, UK & Manchester, M13 9PL, UK
\end{tabular}\\[1em]

\begin{center}
\large September 2020    
\end{center}
\endgroup

\begin{abstract}
The transportation $\Lp$ distance, denoted $\TLp$, has been proposed as a generalisation of Wasserstein $\Wp$ distances motivated by the property that it can be applied directly to colour or multi-channelled images, as well as multivariate time-series without normalisation or mass constraints.
These distances, as with $\Wp$, are powerful tools in modelling data with spatial or temporal perturbations.
However, their computational cost 
can make them infeasible to apply to even moderate pattern recognition tasks.
We propose linear versions of these distances and show that the linear $\TLp$ distance significantly improves over the linear $\Wp$ distance on signal processing tasks, 
whilst being several orders of magnitude faster to compute than the $\TLp$ distance. 
\end{abstract}

\noindent




\section{Introduction}

Optimal transport has gained in recent popularity because of its ability to model diverse data distributions in the signal and image processing fields~\cite{Kolouri:2017}. Transportation-based methods have been successfully applied to image analysis, including medical images~\cite{Basu:2014} and facial recognition~\cite{Kolouri:2015}, as well as cosmology~\cite{Frisch:2002, Frisch:2006} and voice recognition. Machine learning and Bayesian statistics have also benefited from transport-based approaches~\cite{Frogner:2015, Minsker:2017, Solomon:2015, Srivastava:2015, Srivastava:2018, Tarek:2012}.
\vspace*{\baselineskip}

The popularity of optimal transport is, in part, due to the rise in the number of problems in the experimental and social sciences in which techniques are required to compare signal perturbations across spatial or temporal domains. For example, optimal transport based methods for image registration and warping~\cite{Haker:2004} and image morphing~\cite{Zhu:2007} have existed for many years. Transportation techniques provide non-linear methods that jointly model locations and intensities, making transportation based approaches a powerful tool in many problems.
\vspace*{\baselineskip}

Optimal transport methods, in particular Wasserstein $\Wp$ distances, are grounded in a wealth of mathematical theory.
Excellent introductions to the advanced mathematical theory of optimal transport are presented in \cite{Villani:2003, Villani:2008}, whilst \cite{Santambrogio:2015} presents the theory with a more applied perspective.
Many technical aspects of optimal transport have been explored, including geometric properties \cite{Gangbo:1996} and links to evolutionary PDEs \cite{Ambrosio:2008}.
\vspace*{\baselineskip}

There is much interest in developing efficient methods to compute optimal transport distances and maps. For discrete measures, the optimal transport problem can be solved using linear programming approaches~\cite{Santambrogio:2015}. Since solving a linear programme can be costly,~\cite{Oberman:2015} proposed a multi-scale linear programme for efficient computations. Other approaches include flow minimisation techniques~\cite{Angenent:2003, Haker:2004, Ambrosio:2008, Haber:2010}, and gradient descent approaches~\cite{Chartrand:2009}, as well as multi-scale methods~\cite{Merigot:2011, Gerber:2017}. More recently, Cuturi proposed entropy-regularisation based approaches to compute approximations of the optimal transport problem~\cite{Cuturi:2013}. These methods have been explored in-depth with many extensions~\cite{Cuturi:2014, Benamou:2015, Solomon:2015, Aude:2016, Altschuler::2017, Abid::2018, Alaya::2019, Lin::2019}. In order to efficiently compute pairwise optimal transport distances on a large data set, a framework called linear optimal transport was proposed in~\cite{Wang:2013,Kolouri:2016b} for the $\Wp$ distance. 
In~\cite{Wang:2013} the linear transportation distance was applied to classification tasks for medical, facial and galaxy images, \cite{Kolouri:2016b} applied the distance to classification problems in medical, facial, galaxies and bird images, \cite{park18} used the framework to generate new images, and \cite{cai2020linearized} used the distance to classify jets in collider data.
\vspace*{\baselineskip}

Higher-order transportation methods have been proposed in the mathematical analysis literature~\cite{Trillos:2016, Thorpe:2017}. 
In~\cite{Trillos:2016} Garc\'{i}a Trillos and Slep\v{c}ev proposed the transportation $\Lp$ ($\TLp$) distance to define discrete-to-continuum convergence of variational problems on point clouds. This was further extended in~\cite{Thorpe:2017} to a transportation $\Wkp$ distance (where the notation relates to Sobolev spaces). 
These transportation distances have several key advantages over the optimal transport distance; these include~\cite{Thorpe:2017}:
\begin{enumerate}
\item They have no need for mass normalisation.
\item They are not restricted to non-negative measures.
\item They more accurately model translations compared to $\Lp$.
\item They can compare signals with different discretisations.
\end{enumerate}
Additionally, the transportation $\Wkp$ distance can include information on the derivative of the signal.
The transportation distances $\TLp$ and $\Wkp$ were proposed in \cite{Thorpe:2017} to tackle problems in signal analysis; for example, they were able to apply transportation methods to colour images. They experimentally show that the $\TLp$ distance outperforms both $\Lp$ distances and $\Wp$ distances in classification tasks. However, these higher-order transportation distances require the computation of an optimal transport map (on a higher dimensional space) and this thwarts its application to larger scale pattern recognition tasks, where pairwise distances are needed. 
In this paper we propose a linear approximation of the $\TLp$ distance, which is orders of magnitude faster to compute than the full $\TLp$ distance first proposed in~\cite{Trillos:2016}, whilst retaining some of its favourable properties
\vspace*{\baselineskip}

Our proposed method can be seen as an extension of the linear Wasserstein framework ($\LWp$)~\cite{Wang:2013} to higher order transportation distances.
The linearisation is essentially a projection onto the tangent manifold at a given reference point.
The geodesic distance (in this case corresponding to the Wasserstein distance) is approximated by the Euclidean distance in the tangent space. 
Suppose we wish to compare $N$ signals/images, then application of optimal transport methods would require computation of $N(N-1)/2$ distances. In the linear optimal transport framework only $N$ distances need to be computed. From here signals/images are then embedded into Euclidean space allowing linear statistical methods to be applied, whilst preserving much of the geometry of the original optimal transport space.
In~\cite{Wang:2013} the optimal transport distance of choice is the $\Wp$ distance, here we will choose the $\TLp$ distance. This choice allows a more general set of unnormalised, not necessarily non-negative, multi-channel signals to modelled, which is not possible in the linear $\Wp$ framework.
\vspace*{\baselineskip}

This manuscript begins by reviewing optimal transport and the $\Wp$ and $\TLp$ distances.
The $\LWp$ framework is reviewed in Section~\ref{section:LOT} and we propose our extension to $\TLp$ in Section~\ref{section:LTLP}.
We then give an overview on interpolation in the $\TLp$ space.
A background on numerical methods, more precisely how existing methods for $\Wp$ can be adapted to $\TLp$ is included in the appendix.
In Section~\ref{sec:Results} we apply our method to classification problems in 
Australian sign language (Section~\ref{subsec:Results:Auslan}), breast cancer histopathology (Section~\ref{subsec:Results:Cancer}) and financial time series (Section~\ref{subsec:Results:Financial}).
We show that linear $\TLp$ ($\LTLp$) outperforms $\LWp$ and that it has similar performance to $\TLp$, but is several orders of magnitude faster.
Other applications to synthetic data sets and cell morphometry are given in the appendix.

\section{Methods}

\subsection{Optimal Transport and the Wasserstein Distance}

To fix notation, we review the modern Monge-Kantorovich formulation of optimal transport and we refer to the excellent monographs~\cite{Villani:2003, Villani:2008,Santambrogio:2015,peyre18} for a thorough exposition. Let $\mu$ and $\nu$ be probability measures on measure spaces $X$ and $Y$ respectively, i.e. $\mu,\nu\in\cP(X)$. Further, let $\mathcal{B}(X)$ denote the Borel $\sigma$-algebra on $X$. We define the \emph{pushforward} of a measure $\mu\in\cP(X)$ by a function $h:X\to Z$ by $h_{*}\mu(A) := \mu(h^{-1}(A))$ for all $A\in\mathcal{B}(Z)$. The inverse of $h$ is understood as being in the set theoretic sense, i.e. $h^{-1}(A) := \{x \,:\, h(x) \in A\}$. We denote by $\Pi(\mu, \nu)$ the set of all measures on $X\times Y$ such that the first marginal is $\mu$ and the second marginal is $\nu$.
To be precise, if $P^X : X \times Y \to X$ and $P^Y : X \times Y \to Y$ are the canonical projections then $P^X_* \pi = \mu$ and $P^Y_* \pi = \nu$.
We call any $\pi\in\Pi(\mu,\nu)$ a \emph{transportation plan} between $\mu$ and $\nu$ (also called a coupling between $\mu$ and $\nu$).

The \emph{Kantorovich optimal transport problem} is the following variational problem
\begin{equation}\label{equation::Kantorovich}
K(\mu, \nu) = \inf_{\pi \in \Pi(\mu, \nu)}\int_{X\times Y}c(x, y)\, \dd \pi(x, y),
\end{equation}
where $c(x, y)$ is a cost function. The minimiser of this problem is called the \emph{optimal transport plan} $\pi^\dagger$ and such a minimiser exists when $c$ is lower semi-continuous (see, for example \cite{Santambrogio:2015}). 
The prototypical example for $c$ (when $X=Y=\bbR^d$) is $c(x,y) = |x-y|_p^p:=\sum_{i=1}^d |x_i-y_i|^p$, for this choice of $c$ one can define the \emph{Wasserstein distance} by $\dWp(\mu,\nu) = \sqrt[p]{K(\mu,\nu)}$ (see also~\eqref{eq:Methods:dWp} below).
When $c$ is a metric then~\eqref{equation::Kantorovich} is also known as the earth mover's distance. 
\vspace*{\baselineskip}


Now, considering a different formulation, let $T: X \to Y$ be a Borel measurable function such that $T_*\mu = \nu$. 
The \emph{Monge optimal transport problem} is to solve
\begin{equation}\label{equation::monge}
M(\mu,\nu) = \inf_{T\,:\, T_*\mu = \nu} \int_{X\times Y} c(x, T(x)) \, \dd \mu(x)
\end{equation}
We call any $T$ that satisfies $T_*\mu=\nu$ a \emph{transport map} between $\mu$ and $\nu$, and the solution to the optimisation problem $T^\dagger$ is called the \emph{optimal transport map}.
\vspace*{\baselineskip}

It is worthwhile noting that the formulation of the optimal transport problems in equations~\eqref{equation::Kantorovich} and~\eqref{equation::monge} are not, in general, equivalent.
However, if the optimal transport plan $\pi^\dagger$ can be written in the form $\pi^\dagger = (\Id\times T^\dagger)_*\mu$ then it follows that $T^\dagger$ is an optimal transport map and the two formulations are equivalent, i.e. $K(\mu,\nu)=M(\mu,\nu)$.
A sufficient condition to show that such an optimal transport plan exists is to require that $\mu$ is absolutely continuous with respect to the Lebesgue measure on a compact domain $\Omega\subset\bbR^d$ and, in addition, $c(x,y) = h(x-y)$ where $h:\Omega \to [0,\infty)$ is strictly convex and superlinear, see~\cite[Theorem~2.44]{Villani:2003}.
Note that it is easy to find examples where there do not exist transport maps at all.
For example, if $\mu = \frac13 \delta_{x_1} + \frac13 \delta_{x_2} + \frac13 \delta_{x_3}$ and $\nu = \frac12 \delta_{y_1} + \frac12 \delta_{y_2}$.
\vspace*{\baselineskip}

Let $\Omega \subset \mathbb{R}^n$ and $\cP_p(\Omega)$ be the set of Radon measures on $\Omega$ with finite $p^{\mathrm{th}}$ moment.
For $p \in [1,\infty)$, we define the Wasserstein distance between $\mu$ and $\nu$ in $\cP_p(\Omega)$ by
\begin{equation}\label{eq:Methods:dWp}
\dWp(\mu, \nu) = \left(\inf_{\pi \in \Pi(\mu, \nu)}\int_{\Omega\times \Omega}|x - y|^p_p \, \dd \pi(x, y)\right)^{1/p}.
\end{equation}
The Wasserstein space $\Wp$ is the metric space $(\cP_p(\Omega),\dWp)$.
For the case $p = \infty$, we can define a distance on $\cP_\infty(\Omega)$ by
\[ \dWinfty(\mu,\nu) := \inf_{\pi \in \Pi(\mu, \nu)}\esssup_{\pi}\{|x - y|: (x,y) \in \Omega \times \Omega\}. \]



We briefly review the features that make optimal transport particularly suited to signal and image processing. For an extended survey of these ideas see~\cite{Kolouri:2017}.
Optimal transport is able to provide generative models which can represent diverse data distributions and can capture signal variations as a result of spatial perturbations. 
Furthermore, there is a well formulated theoretical basis (particularly for $\Wtwo$) with interesting geometrical properties, such as existence of minimisers~\cite{gangbo95,Villani:2008}, the Riemannian structure of Wasserstein spaces when $p=2$~\cite{Ambrosio:2008,Otto:2001} and characterisation as the weak$^*$ convergence when $\Omega$ is compact~\cite{Santambrogio:2015}. 
The Riemannian structure allows the characterisation of geodesics (shortest curves) on the space $\cP_2(\Omega)$.
In addition, there are many methods to compute optimal transport distances, for convenience, we review a selection in the appendix in the context of computing the $\TLp$ distance.


\subsection{The Transportation \texorpdfstring{$\Lp$}{Lp} Distance} \label{section::TLP}

The $\TLp$ distance was first introduced in~\cite{Trillos:2016} to define a discrete-to-continuum convergence on point clouds. This tool has been been extensively used to study similar statistical problems, e.g.~\cite{Dunlop:2018,Slepcev:2017,Cristoferi:2018,Thorpe:2016,osting17,garciatrillos18, garciatrillos16a,garciatrillos17,garciatrillos17a,garciatrillos18aAAA,garciatrillos17bAAA,garciatrillos17cAAA, garciatrillos18bAAA}, 
and recently has been shown to be a valuable tool in signal analysis~\cite{fitschen17, Thorpe:2017}. 
In this section, we review the definitions and properties of the $\TLp$ distance and space.
\vspace*{\baselineskip}

Given an open and bounded domain $\Omega\subset\bbR^d$ we define the $\TLp$ space as the set of pairs $(\mu, f)$ such that $f \in \Lp(\mu; \mathbb{R}^m)$ and $\mu \in \cP_p(\Omega)$. 
We do not make any assumption on the dimension $m$ of the range of $f$.  Working in this formal and abstract framework of measure theory allows us to formulate our methods for both discrete and continuous signals, simultaneously. Importantly, similarly to the Wasserstein distance but unlike the $\Lp$ distance, this framework allows us to compare signals with different discretisations since $\mu$ and $\nu$ need not have the same support. We define the $\TLp$ space as
\[ \TLp : = \left\{(\mu, f): \mu \in \cP_p(\Omega), f \in \Lp(\mu; \mathbb{R}^m)\right\}. \]
We construct the $\TLp$ metric between pairs $(\mu, f)\in \TLp$ and $(\nu, g)\in \TLp$ as follows:
\begin{equation} \label{eq:Methods:TLp}
\dTLp((\mu, f), (\nu, g)) = \left(\inf_{\pi \in \Pi(\mu, \nu)}\int_{\Omega\times \Omega}|x - y|^{p}_p + |f(x)- g(y)|^{p}_p \, \dd\pi(x,y)\right)^{1/p}.
\end{equation}
Intuitively, we see that $\TLp$ optimal transport plans (that is plans in $\Pi(\mu,\nu)$ that achieve the minimum in the above variational problem) strike a balance between matching spatially, i.e. minimising $\int_{\Omega\times \Omega}|x - y|^{p}_p \, \dd \pi(x,y)$, and matching signal features, i.e. minimising $\int_{\Omega\times \Omega}|f(x) - g(y)|^{p}_p \, \dd \pi(x,y)$.
\vspace*{\baselineskip}

\begin{example}
	Let us explain here how images can be represented in $\TLp$.
	Let $\{x_i\}_{i=1}^n$ be the location of pixels (which usually form a grid over $[0,1]\times[0,1]$).
	We apply $\TLp$ by choosing a base measure $\mu$, this is commonly the uniform measure over $\{x_i\}_{i=1}^n$, i.e. $\mu = \frac{1}{n}\sum_{i=1}^n \delta_{x_i}$.
	An image is then represented by the pair $(\mu,f)$ where $f:\{x_i\}_{i=1}^n \to \bbR^3$ for RGB images and $f(x_i)$ are the RGB values for the pixel at location $x_i$.
	Similarly, for greyscale images one would have $f:\{x_i\}_{i=1}^n\to \bbR$ where $f(x_i)$ is now the greyscale value for the pixel at location $x_i$.
	Of course, one can make different choices for $\mu$ in order to emphasise regions/features of the images.
	We note that, as we represent the image as a function $f$, then we do not need to assume that the image has unit mass.
	To apply $\Wp$ we would represent the image as a probability measure $\mu$, and we would therefore have to assume that the image has unit mass (or renormalise). This strong assumption limits the applicability of optimal transport in image processing because it requires ad-hoc renormalisation which may suppress features of the image.
\end{example}
\vspace*{\baselineskip}

To further understand the $\TLp$ distance, we reformulate it as a $\Wp$ distance supported on the graphs of functions. Recall that the graph of a function is defined as:
\[ \Gra(f) := \left\{(x, f(x)): x \in \Omega\right\}. \]
Note that the $\Gra(f)\subset \Lambda := \Omega \times \mathbb{R}^m$. We define the following lifted measure on $\Gra(f)$: 
\[ \tilde{\mu}(A\times B) = (\Id\times f)_*\mu(A\times B) = \mu(\{x: x \in A, f(x) \in B\}), \]
where $A \times B \subset \Lambda$. It is clear that $\tilde{\mu}$ is a well-defined measure on $\Gra(f)$. We can characterise the $\TLp$ distance as a $\Wp$ distance in the following way~\cite{Trillos:2016}: 
\begin{align*}
\dTLp^p((\mu, f), (\nu, g)) & = \inf_{\pi \in \Pi(\mu, \nu)}\int_{\Omega\times \Omega}|x - y|^{p}_p + |f(x)- g(y)|^{p}_p \, \dd \pi(x,y) \\
 & = \inf_{\tilde{\pi} \in \tilde{\Pi}(\tilde{\mu}, \tilde{\nu})} \int_{\Lambda\times \Lambda}|\boldsymbol{x} - \boldsymbol{y}|^p_p \, \dd \tilde{\pi}(\boldsymbol{x}, \boldsymbol{y}) \\
 & = \dWp^p(\tilde{\mu}, \tilde{\nu}).
\end{align*}
Thus, we can see that the $\TLp$ distance is the $\Wp$ distance between the appropriate measures on the graphs of function.
This allows us to make the following identification between $\Wp$ and $\TLp$ through the mapping. 
\begin{align}
\TLp(\Omega) & \to \Wp(\Omega\times\bbR^m) \label{eq:Methods:TLpIdentWp1} \\
(\mu, f) & \mapsto \tilde{\mu} = (\Id \times f)_* \mu. \label{eq:Methods:TLpIdentWp2}
\end{align}
This connection of the $\TLp$ distance and the $\Wp$ distance facilitates the transfer of certain Wasserstein properties to the $\TLp$ setting; for example, metric properties and existence of minimisers.

It is easy to see that, for any $\mu,\nu\in \cP_p(\Omega)$ and $f,g\in L^p(\mu)$:
\begin{align*}
\dTLp((\mu,f),(\mu,g)) & \leq \|f-g\|_{L^p(\mu)} \\
\dWp(\mu,\nu) & \leq \dTLp((\mu,\mathds{1}),(\nu,\mathds{1})).
\end{align*}
In fact, one can also prove the converse inequalities (up to a constant) and hence $\TLp$ can be seen to generalise both weak$^*$ convergence of measures and $\Lp$ convergence of functions (see~\cite{Trillos:2016} or Proposition~\ref{prop:Methods:TLpProp} below). 

The $\TLp$ distance can be seen as a special case of optimal transport by observing that $\dTLp((\mu,f),(\nu,g))$ coincides with the Kantorovich optimal transport problem between two measures $\mu$ and $\nu$ with cost function $c(x,y;f,g) = |x - y|^p_p + |f(x) - g(y)|^p_p$. 

For reference, we also state a Monge-type formulation of the $\TLp$ distance as follows:
\begin{equation} \label{eq:Methods:MongeTLp}
\dTLp^p((\mu, f), (\nu, g))  = \inf_{T: T_{*}\mu = \nu}\int_{\Omega}|x - T(x)|^{p}_p + |f(x)- g(T(x))|^{p}_p \,\dd\mu(x),
\end{equation}
where $T$ is a transportation map.
When we write the Monge formulation of optimal transport we are assuming that there is an equivalence between~\eqref{eq:Methods:TLp} and~\eqref{eq:Methods:MongeTLp}.
This is in general difficult to verify, since the application of Brenier's theorem does not lead to natural conditions.
(Assuming that $\mu$ does not give mass to small sets and both $\mu$ and $\nu$ have a sufficient number of bounded moments then one can apply Brenier's theorem to $K(\mu,\nu)$ where $c(x,y) = |x-y|_p^p + |f(x) - g(y)|_p^p$ and $K$ is defined by~\eqref{equation::Kantorovich}, if $c$ is strictly convex; practically this is not reasonable.)
However, when $\mu$ and $\nu$ are discrete uniform measures with supports of equal size the Monge formulation~\eqref{eq:Methods:MongeTLp} coincides with the Kantorovich formulation~\eqref{eq:Methods:TLp} (see the proposition below).

Let us recall the identification~(\ref{eq:Methods:TLpIdentWp1}-\ref{eq:Methods:TLpIdentWp2}) and the identity $\dWp(\tilde{\mu},\tilde{\nu}) = \dTLp((\mu,f),(\nu,g))$.
Then, there is a corresponding equivalence between transport maps.
That is (assuming all transport maps exist and are unique) let $T^\dagger$ achieve the minimum in~\eqref{eq:Methods:MongeTLp} and $\tilde{T}^\dagger$ achieve the minimum in $M(\tilde{\mu},\tilde{\nu})$ where $M$ is given by~\eqref{equation::monge} with $c(x,y) = |x-y|_p^p$.
It follows that $\tilde{T}^\dagger(\boldsymbol{x}) = (T^\dagger(x),g(T^\dagger(x))$ for $\mu$-almost every $x\in \Omega$.

\begin{proposition}
\label{prop:Methods:TLpProp}
Let $\Omega\subset\bbR^d$ be open and bounded and $p\in (1,\infty)$.
Then the following holds
\begin{enumerate}
\item \cite[Remark 3.4]{Trillos:2016} $(\TLp,\dTLp)$ is a metric space;
\item \cite[Theorem 5.10]{Santambrogio:2015} $\mu_n\weakstarto \mu$ if and only if $(\mu_n,\mathds{1}) \toTLp (\mu,\mathds{1})$;
\item \cite[Proposition 3.12]{Trillos:2016} $f_n\to f$ in $\Lp(\mu)$ if and only if $(\mu,f_n)\toTLp (\mu,f)$;
\item \cite[Proposition 3.4]{Thorpe:2017} for any $(\mu,f),(\nu,g)\in\TLp$ there exists a transport plan $\pi^\dagger\in\Pi(\mu,\nu)$ realising the minimum in $\dTLp((\mu,f),(\nu,g))$, i.e. 
\[ \dTLp^p((\mu,f),(\nu,g)) = \int_{\Omega\times\Omega} |x-y|_p^p + |f(x) - g(y)|_p^p \, \dd \pi^\dagger(x,y); \]
\item \cite[Proposition 3.5]{Thorpe:2017} if $\mu = \frac{1}{n} \sum_{i=1}^n \delta_{x_i}$ and $\nu = \frac{1}{n} \sum_{j=1}^n \delta_{y_j}$ then for any $f\in L^p(\mu)$ and $g\in L^p(\nu)$ there exists $T^\dagger:\{x_i\}_{i=1}^n\to \{y_j\}_{j=1}^n$ such that $T^\dagger_*\mu=\nu$ and
\[ \dTLp^p((\mu,f),(\nu,g)) = \int_\Omega |x-T^\dagger(x)|_p^p + |f(x) - g(T^\dagger(x))|^p_p \, \dd \mu(x), \]
i.e. the Monge and Kantorovich formulations of $\TLp$ (given by~\eqref{eq:Methods:MongeTLp} and~\eqref{eq:Methods:TLp} respectively) are equivalent for point masses.
\end{enumerate}
\end{proposition}


Note that not all properties of $\Wp$ carry through to $\TLp$.
For example $(\TLp, \dTLp)$ is not complete.
Indeed, following~\cite{Trillos:2016}, let $\Omega = (0,1)$ and note that $f_{n+1}(x) = \mathrm{sign} \sin(2^n \pi x)$, $\mu_n = \cL\lfloor_{(0,1)}$ (the Lebesgue measure on $(0,1)$) is a Cauchy sequence in $(\TLp, \dTLp)$.
However, $\{f_n\}$ does not converge in $\Lp$ and therefore $\{(\mu_n, f_n)\}$ cannot converge in $\TLp$, by part 3 of the above proposition. 
The completion of $\TLp$ can be identified with the set of Young measures, and therefore the space $\cP(\Omega\times \bbR^m)$, see~\cite[Remark 3.6]{Trillos:2016}. 
We note also that there do not exist geodesics in $\TLp$, however we develop an approach to interpolate in $\TLp$ (see section \ref{subsec:Methods:Int}).
\vspace*{\baselineskip}


By part 3 in the above proposition $\TLp$ inherits some sensitivity to high frequency perturbations from the $\Lp$ norm.
In contrast, as $\Wp$ metricizes the weak* convergence (in compact Euclidean spaces) then $\Wp$ is insensitive to high frequency perturbations, see~\cite[Section 2.2]{Thorpe:2017}.

We also can deduce that the $\TLp$ distance inherits translation sensitivity from $\Wp$.
In particular, we can see that $\Lp$ distances are insensitive to translations if the supports of the images are disjoint.
On the other hand, the $\Wp$ distance scales linearly with the size of the translation no matter how large the translation.
The $\TLp$ distance, although not scaling linearly with translation, is monotonically non-decreasing as a function of translation.
\vspace*{\baselineskip}

The $\TLp$ appears an excellent tool to exploit in pattern recognition problems such as images or times series, however it is as computational demanding as $\Wp$ and despite recent advances in computation of optimal transport, e.g.~\cite{Cuturi:2013}, it is still challenging to apply it to large scale problems. In the next section, we review the linear Wasserstein ($\LWp$) framework, which was introduced to allow application of optimal transport methods to classification problems~\cite{Wang:2013}. In later sections, we apply the ideas of linear optimal transport to $\TLp$.

\subsection{A Linear \texorpdfstring{$\Wp$}{Wp} Framework}\label{section:LOT}

The $\LWp$ framework was proposed in~\cite{Wang:2013}, as a way to apply optimal transport techniques (in particular $\Wp$) to large scale classification problems for image analysis. Given a set of $N$ images, one would need to compute all pairwise $\Wp$ distances in order to use methods such as $k$-nearest neighbour classifiers. The $\LWp$ framework was developed so that the Wasserstein distance needs to be computed only $N$ times. In particular, it is the optimal Wasserstein transport maps between signals that are computed. From here, the images are embedded in Euclidean space therefore allowing linear statistical techniques to be applied~\cite{Wang:2011b}. This technique was successfully applied in~\cite{Basu:2014} to detect morphological difference in cancer cells. The technique has been further refined and extended to super-resolution images~\cite{Kolouri:2015, Kolouri:2016b}. In this section, we briefly review the ideas of $\LWp$.
\vspace*{\baselineskip}

The idea behind the $\LWp$ framework is to find the optimal Wasserstein transport maps with respect to one (reference) measure. 
For simplicity we assume that the Monge problem is equivalent to the Kantorovich problem and, in particular, there exists optimal transport maps.
Via an embedding of the transport map into Euclidean space the $\Wp$ distance between any two pairs is estimated. More precisely, the $\LWp$ framework provides a linear embedding for $\cP_p(\Omega)$ with respect to a fixed measure $\sigma \in \cP_p(\Omega)$~\cite{Kolouri:2017}. This means the Euclidean distance of the embedded measure and the fixed measure $\sigma$ is equal to the $\Wp$ distance of the measure and the fixed measure. The Euclidean distance between any two measures is then an approximation to the Wasserstein distance between these measures. These linear embeddings then facilitate the application of standard statistical techniques such as PCA, LDA and K-means. The $\LWp$ framework is also invertible and so synthetic, but physically possible signals, can be realised~\cite{park18}.
\vspace*{\baselineskip}

Let $\mu_1,\mu_2\in\cP(\Omega)$ and $\sigma\in \cP(\Omega)$ is our reference measure.
Throughout this section, we assume that optimal transport maps $T^{\mu_i}:\Omega\to\Omega$ exist between $\sigma$ and $\mu_i$, i.e.
\[ \dWp(\sigma, \mu_i) = \sqrt[p]{\int_{\Omega} |x - T^{\mu_i}(x)|^p_p \, \dd \sigma(x)}, \]
and $T^{\mu_i}_*\sigma = \mu_i$. 
If optimal transport maps do not exist then one can still define the $\LWp$ distance but there is not a natural way to embed this distance into Euclidean space.
We refer to~\cite[Section 2.3]{Wang:2013} on how to define the $\LWp$ distance using generalised geodesics.

In the setting considered here the \emph{Linear Wasserstein Distance}, $\LWp$, is defined by \cite{Wang:2013}:
\[ \dLOT(\mu_1, \mu_2) := \sqrt[p]{\int_{\Omega}|T^{\mu_1}(x) - T^{\mu_2}(x)|^p_p \, \dd \sigma(x)}. \]
We observe that $\dLOT$ is a metric and
\[ \dLOT(\mu_1,\mu_2) = \lVert T^{\mu_1} - T^{\mu_2} \rVert_{\Lp(\sigma)}. \]

Let us assume that $\sigma$ has a density $\rho$ with respect to the Lebesgue measure and define
\begin{equation} \label{eq:Methods:Pc}
P_c(\mu) = (T^{\mu} - \Id) \rho^{\frac{1}{p}}.
\end{equation}
Then 
\begin{equation} \label{eq:Methods:dLOTEmbCts}
\dLOT(\mu_1,\mu_2) = \| P_c(\mu_1) - P_c(\mu_2) \|_{\Lp(\Omega)}.
\end{equation}
The map $P_c$ is our linear embedding from the Wasserstein space to Euclidean space.
We make the following claims on the embedding.

\begin{proposition}
\label{prop:Methods:LOTCts}
Assume $\Omega\subset\bbR^d$ is bounded and $\sigma\in \cP(\Omega)$ has a density $\rho$ with respect to the Lebesgue measure.
Define $P=P_c$ where $P_c$ is given by~\eqref{eq:Methods:Pc}.
Then, the following holds:
\begin{enumerate}
\item $P(\mu) \in \Lp(\Omega)$ for any $\mu\in \cP(\Omega)$,
\item $P(\sigma) = 0$,
\item $\dLOT(\mu_1,\mu_2) = \|P(\mu_1) - P(\mu_2)\|_{\Lp(\Omega)}$ for any $\mu_1,\mu_2\in \cP(\Omega)$,
\item $\dLOT(\sigma,\mu) = \dWp(\sigma,\mu)$ for any $\mu\in \cP(\Omega)$.
\end{enumerate}
\end{proposition}

\begin{proof}
Since $\sigma$ has a density with respect to the Lebesgue measure then transport maps $T^{\mu_i}$, $T^{\mu}$ exist.
Since the $\Wp$ distance is finite (as $\Omega$ is bounded), it follows that $T^{\mu_i}-\Id\in L^p(\sigma)$ which proves (1).
(2) follows directly from $T^{\sigma} = \Id$.
(3) was shown already in~\eqref{eq:Methods:dLOTEmbCts}.
Finally, (4) follows
\[ \dLOT(\sigma,\mu) = \| (T^\mu - \Id) \rho^{\frac{1}{p}} \|_{\Lp(\Omega)} = \| T^\mu - \Id \|_{\Lp(\sigma)} = \dWp(\sigma,\mu) \]
where we use $P(\sigma) = 0$.
\end{proof}

We make a similar definition for discrete measures.
If $\sigma = \sum_{j=1}^n \rho_j \delta_{x_j}$ for some $\{x_i\}_{i=1}^n\subset\bbR^d$ then we define
\begin{equation} \label{eq:Methods:Pd}
[P_d(\mu_i)]_j = (T^{\mu_i}(x_j) - x_j) \rho_j^{\frac{1}{p}}.
\end{equation}
Analogously to the Lebesgue density case we have
\begin{equation*} \label{eq:Methods:dLOTEmbDis}
\dLOT(\mu_1,\mu_2) = |P_d(\mu_1) - P_d(\mu_2) |_p
\end{equation*}
where we recall that $|\cdot|_p$ is the Euclidean $p$-norm: $|x|_p := \sqrt[p]{\sum_{j=1}^n |x_j|^p}$.
For discrete $\sigma$ the map $P_d$ is our linear embedding from the Wasserstein space to Euclidean space.

\begin{proposition}
\label{prop:Methods:LOTDis}
Let $\{x_i\}_{i=1}^n \subset \bbR^d$ and assume $\sigma=\frac{1}{n} \sum_{i=1}^n \delta_{x_i}$.
Define $P=P_d$ where $P_d$ is given by~\eqref{eq:Methods:Pd} with $\rho_j = \frac{1}{n}$.
Then, the following holds:
\begin{enumerate}
\item $P(\mu) \in \ell^p$ for any $\mu=\frac{1}{n} \sum_{j=1}^n \delta_{y_j}$,
\item $P(\sigma) = 0$,
\item $\dLOT(\mu_1,\mu_2) = |P(\mu_1) - P(\mu_2)|_p$ for any $\mu_1=\frac{1}{n} \sum_{j=1}^n \delta_{y_j}$, $\mu_2=\frac{1}{n} \sum_{j=1}^n \delta_{z_j}$,
\item $\dLOT(\sigma,\mu) = \dWp(\sigma,\mu)$ for any $\mu=\frac{1}{n} \sum_{j=1}^n \delta_{y_j}$.
\end{enumerate}
\end{proposition}

\begin{proof}
The particular forms of all the measures $\sigma,\mu,\mu_1,\mu_2$ is enough to guarantee that transport maps $T^{\mu}, T^{\mu_1}, T^{\mu_2}$ all exist.
The proof is then analogous to the proof of Proposition~\ref{prop:Methods:LOTCts}.
\end{proof}

\begin{example}
	Let us consider how to generate a new image using the linear embedding.
	Suppose we have a reference measure $\sigma\in \cP(\bbR^d)$ with density $\rho$ and a set of measures $\{\mu_i\}_{i=1}^N\subset \cP(\bbR^d)$ with optimal transport maps $T^{\mu_i}$ which form the linear embedding through $\alpha_i=P(\mu_i) = (T^{\mu_i}-\Id)\rho^{\frac{1}{p}}$.
	Given a new point $\alpha$ in the linear space we can define a transport map by $T = \alpha \rho^{-\frac{1}{p}} + \Id$.
	We generate a new image by $\mu = T_* \sigma$.
	Note that to generate the new image we only required the reference measure $\sigma$ and a new point $\alpha$ in the linear space.
	However, in order to generate the new point $\alpha$ it will often be sensible to use the statistics of $\{\alpha_i\}_{i=1}^N$, for example see~\cite{park18}.
\end{example}

In both the Lebesgue density and uniform discrete case $P$ preserves the $\Wp$ distance between the reference measure and any given $\mu$ (where for discrete measures $\mu$ is also uniform discrete).
Between $\mu_1$, $\mu_2$ one approximates $\dWp(\mu_1,\mu_2) \approx \dLOT(\mu_1, \mu_2)$.
The next section proposes our extension of the $\LWp$ framework to the $\TLp$ distance.

\subsection{A Linear \texorpdfstring{$\TLp$}{TLp} Framework}\label{section:LTLP}



In this section, we propose a linear $\TLp$ framework.
Recall that the $\TLp$ distance can be defined as an optimal transport distance between measures supported on the graph of a function.
Let $(\sigma,h)\in\TLp$ be the $\TLp$ reference signal and $\tilde{\sigma} = (\Id\times h)_*\sigma\in \cP_p(\Lambda)$ the measure in $\Lambda=\Omega\times\bbR^m$ with support on the graph of $h$.
Let $(\mu_i,f_i)\in\TLp$, $i=1,2$, and define $\tilde{\mu}_i = (\Id\times f_i)_*\mu_i$.
As in the previous section we will assume that optimal transport maps $\tilde{T}^{\tilde{\mu}_i}:\Lambda\to\Lambda$ exist between $\tilde{\sigma}$ and $\tilde{\mu}_i$, i.e.
\[ \dWp(\tilde{\sigma},\tilde{\mu}_i) = \sqrt[p]{\int_\Lambda |\boldsymbol{x} - \tilde{T}^{\tilde{\mu}_i}(\boldsymbol{x})|^p_p \, \dd \tilde{\sigma}(\boldsymbol{x})}. \]
Recall that we can write $\tilde{T}^{\tilde{\mu}_i}$ in the form $\tilde{T}^{\tilde{\mu}_i}(\boldsymbol{x}) = (T^{\mu_i}(x),f_i(T^{\mu_i}(x)))$ where $\boldsymbol{x} = (x,y) \in \bbR^d \times \bbR^m$ and $T^{\mu_i}$ is the optimal plan for the Monge problem~\eqref{equation::monge} between $\mu_i$ and $\sigma$ with cost $c(x,y) = |x-y|_p^p + |f(x) - g(y)|_p^p$.
The \emph{Linear Transportation $\Lp$ Distance} ($\LTLp$) is defined as
\[ \dLTLp((\mu_1,f_1),(\mu_2,f_2)) := \sqrt[p]{\int_\Lambda |\tilde{T}^{\tilde{\mu}_1}(\boldsymbol{x}) - \tilde{T}^{\tilde{\mu}_2}(\boldsymbol{x}) |^p_p \, \dd \tilde{\sigma}(\boldsymbol{x})}. \\
  \]
Simple manipulations of the $\LTLp$ distance imply
\begin{align*}
\dLTLp((\mu_1,f_1),(\mu_2,f_2)) & =  \| \tilde{T}^{\tilde{\mu}_1} - \tilde{T}^{\tilde{\mu}_2} \|_{\Lp(\tilde{\sigma})} \\
 & = \sqrt[p]{\int_\Omega |T^{\mu_1}(x) - T^{\mu_2}(x)|_p^p + |f_1(T^{\mu_1}(x)) - f_2(T^{\mu_2}(x))|_p^p \, \dd \sigma(x)}.
\end{align*}

Following the construction of the embedding in the previous section we go directly to the discrete case (since $\tilde{\sigma}$ has support on the graph it cannot have a density with respect to the Lebesgue measure).
We assume that $\sigma = \sum_{i=1}^n \rho_j \delta_{x_j}$ for some $\{x_j\}_{j=1}^n \subset \bbR^d$ and we define
\begin{align}
[P_d((\mu_i,f_i))]_j & = (T^{\mu_i}(x_j) - x_j)\rho_j^{\frac{1}{p}} \label{eq:Methods:LTLpPd} \\
[Q_d((\mu_i,f_i))]_j & = (f_i(T^{\mu_i}(x_j)) - h(x_j)) \rho_j^{\frac{1}{p}} \label{eq:Methods:LTLpQd} \\
\tilde{P}_d((\mu_i,f_i)) & = (P_d((\mu_i,f_i)),Q_d((\mu_i,f_i)) ). \label{eq:Methods:LTLptildePd}
\end{align} 
Given this definition we can write
\[ \dLTLp((\mu_1,f_1),(\mu_2,f_2)) = | \tilde{P}_d((\mu_1,f_1)) - \tilde{P}_d((\mu_2,f_2)) |_p. \]
The map $\tilde{P}_d$ embeds our signals into Euclidean space.
We have the following properties of the embedding (analogous to Propositions~\ref{prop:Methods:LOTCts} and~\ref{prop:Methods:LOTDis}).

\begin{proposition}
\label{prop:Methods:LTLpDis}
Let $\{x_i\}_{i=1}^n \subset \bbR^d$ and assume $\sigma=\frac{1}{n} \sum_{i=1}^n \delta_{x_i}$.
Define $\tilde{P}=\tilde{P}_d$ where $\tilde{P}_d$ is given by~\textnormal{(}\ref{eq:Methods:LTLpPd}-\ref{eq:Methods:LTLptildePd}\textnormal{)} with $\rho_j = \frac{1}{n}$.
Then, the following holds:
\begin{enumerate}
\item $\tilde{P}((\mu,f)) \in \ell^p$ for any $(\mu,f)\in \TLp$ with $\mu=\frac{1}{n} \sum_{j=1}^n \delta_{y_j}$,
\item $P((\sigma,h)) = 0$,
\item $\dTLp((\mu_1,f_1),(\mu_2,f_2)) = |\tilde{P}((\mu_1,f_1)) - \tilde{P}((\mu_2,f_2))|_p$ for any $(\mu_i,f_i)\in \TLp$ with $\mu_1=\frac{1}{n} \sum_{j=1}^n \delta_{y_j}$, $\mu_2=\frac{1}{n} \sum_{j=1}^n \delta_{z_j}$,
\item $\dTLp((\sigma,h),(\mu,f)) = \dTLp((\sigma,h),(\mu,f))$ for any $(\mu,f)\in\TLp$ with $\mu=\frac{1}{n} \sum_{j=1}^n \delta_{y_j}$.
\end{enumerate}
\end{proposition}

\begin{proof}
By Proposition~\ref{prop:Methods:TLpProp}(5) the transport maps $T^{\mu}, T^{\mu_1}, T^{\mu_2}$ exist.
The rest of the proof follows as in the proof of Proposition~\ref{prop:Methods:LOTCts}.   
\end{proof}

As for $\LWp$ we have that $\LTLp$ is exactly $\TLp$ when comparing with the reference measure, i.e. $\dLTLp((\sigma,h),(\mu,f)) = \dTLp((\sigma,h),(\mu,f))$.
When we are comparing two measures, neither of which are the base measure, then we make the approximation
\[ \dLTLp((\mu_1,f_1),(\mu_2,f_2)) \approx \dTLp((\mu_1,f_1),(\mu_2,f_2)). \]
When $(\mu_i,f_i)\neq (\sigma,h)$ then the approximation is only formal.
In particular, to derive quantitative bounds between $\LTLp$ and $\TLp$ requires a detailed analysis of the $\TLp$ space, including quantitative estimates on curvature.
To the authors knowledge there is not yet such a bound between the $\Wp$ and $\LWp$ distances, although recently \cite{Moosmuller::2020} have obtained bounds for some perturbations.

\begin{example}
	Let us consider how to generate a new $\TLp$ image from the linear space.
	We recall that colour images can be represented by $(\mu,f)$ where $\{x_i\}_{i=1}^n$ are the locations of pixels (which are uniform across $[0,1]\times [0,1]$), $\mu = \frac{1}{n}\sum_{i=1}^n \delta_{x_i}$ and $f:\{x_i\}_{i=1}^n \to \bbR^3$ represents the RGB values for each pixel.
	We take a reference image $(\sigma,h)$ of the same form (in particular $\sigma = \frac{1}{n}\sum_{i=1}^n \delta_{x_i}$), and note that the embedding (mapping $\tilde{\sigma}=(\Id\times h)_*\sigma$ to $\tilde{\mu}=(\Id\times f)_*\mu$) is given by $\alpha^{\tilde{\mu}} = (\alpha_1^{\tilde{\mu}},\dots,\alpha_n^{\tilde{\mu}})\in \bbR^{5n}$ where
	\[ \alpha_i^{\tilde{\mu}} = \frac{1}{n^{\frac{1}{p}}}\l \tilde{T}^{\tilde{\mu}}(x_i) - (x_i,h(x_i)) \r \in \bbR^5. \]
	To generate a new image we need to invert this mapping.
	Let $\alpha = (\alpha_1,\dots, \alpha_n) \in \bbR^{5n}$.
	We define $\tilde{T}_i = n^{\frac{1}{p}} \alpha_i + (x_i,h(x_i)) \in \bbR^5$.
	Then $y_i = (\tilde{T}_i)_{1:2}\in\bbR^2$ are the location of the pixels and $c_i = (\tilde{T}_{3:5})$ are the RGB values in the new image.
	In the $\TLp$ space the new image is represented by $(\nu,g)$ where $\nu = \frac{1}{n} \sum_{i=1}^n \delta_{y_i}$ and $g(y_i) = c_i$. 
	This is only well defined if $y_i$ are all unique.
	If not, then we use \emph{barycentric projection} (see also Section~\ref{subsec:Methods:Int}), for example if $y_i = y$ for all $i\in \mathcal{I}$ then we define $g(y) = \frac{1}{|\mathcal{I}|} \sum_{\mathcal{I}} c_i$ to be the empirical average.
\end{example}

\subsection{Geodesics and Interpolation} \label{subsec:Methods:Int}

The space $(\cP_p(\Omega), \dWp)$ is a \emph{geodesic space}, with easily characterisable geodesics. Letting $\pi^\dagger \in \Pi(\sigma, \mu)$ be the optimal transport plan that minimises the transport problem given by~\eqref{equation::Kantorovich}, we define $I_t:\Omega \times \Omega \to \Omega$, where $t \in [0,1]$, to be a linear interpolation, as follows:
\[ I_t(x,y) = (1-t)x + ty. \]
Then the geodesic in $\Wp$ is given by $\mu(t) = [I_t]_*\pi^\dagger$.
When there exists transport maps, i.e. $\pi^\dagger = (\Id \times T^\mu)_* \mu$ then the geodesic can be written $\mu_t = [T_t^\mu]_*\mu$ where $T_t^\mu(x) = I_t(x,T^\mu(x)) = (1-t)x + tT^\mu(x)$.
Let $P$ be defined by $P=P_c$ in~\eqref{eq:Methods:Pc} then since 
\[ P(\mu_t) = (T_t^\mu - \Id) \rho^{\frac{1}{p}} = \l (1-t)\Id + tT^\mu  - \Id \r \rho^{\frac{1}{p}} = t\l T^\mu - \Id \r \rho^{\frac{1}{p}} = t P(\mu) \]
we see that the projection of the geodesic onto the Euclidean space is the geodesic between the projections.
In particular, the geodesic between $P(\sigma) = 0$ and $P(\mu)$ in Euclidean space is simply $t P(\mu)$.
The same argument holds in the discrete case where $P=P_d$ is defined by~\eqref{eq:Methods:Pd}.
Since the projection is invertible (at least in some open neighbourhood of the reference measure) we can map from the Euclidean embedding back to $\Wp$. Notably, this allows one to translate principal eigenvectors in PCA space (of the linear embedding) into modes of variation in $\Wp$, see~\cite{Wang:2013} for more details.

This argument does not directly apply to the $\TLp$ space since, by the following remark, the $\TLp$ space does not permit geodesics. 

\begin{remark}
Consider the measure $\mu = \frac12 \delta_0 + \frac12\delta_1$ and
the functions $f(0) = 0, f(1) = 10, g(0) = 10, g(1) = 0.$ Then the
transport between $(\mu,f)$ and $(\mu,g)$ is from $(0,0)$ to $(1,0)$ and from
$(1,10)$ to $(0,10)$. The "half way" point would be the measure
$\mu_{\frac12} = \delta_\frac12$ and the function that takes the value
$10$ and $0$ at $x=\frac12$, which is not a function.
\end{remark}

However, this does not prevent us from interpolating and visualising modes of variation. Indeed, let $\tilde{T}^{\tilde{\mu}}(\boldsymbol{x}) = (T^{\mu}(x),f(T^{\mu}(x)))$ be the optimal $\TLp$ map pushing $(\sigma,h)$ to $(\mu,f)$, then the map
\[ \tilde{T}_t^{\tilde{\mu}}(\boldsymbol{x}) = \l (1-t)x + tT^{\mu}(x), (1-t)h(x) + t(f(T^{\mu}(x))) \r \]
interpolates between the signals $(\sigma, h)$ and $(\mu, f)$. 
In fact, this is the geodesic in $\Wp(\Omega\times\bbR^m)$; that is $\tilde{\mu}_t = \tilde{T}_t^{\tilde{\mu}}$ is the geodesic in $\Wp(\Omega\times\bbR^m)$ between $\tilde{\sigma}$ and $\tilde{\mu}$.
Although we can invert $\tilde{P}=\tilde{P}_d$ (defined in~\eqref{eq:Methods:LTLpPd}) in the Wasserstein space, i.e. for all $p\in \bbR^{nm}$ there exists $\tilde{\nu}\in \cP(\Omega\times \bbR^m)$ such that $P(\tilde{\nu}) = p$ (note that this is $P$ and not $\tilde{P}$ since we are inverting with respect to $\Wp(\Omega\times\bbR^m)$) we cannot guarantee that $\tilde{\nu}$ can be written in the form $\tilde{\nu} = (\Id \times g)_*\nu$ for some $g\in \Lp(\nu)$.
Hence, we cannot in general invert the linear embeddings from $\TLp$ back into $\TLp$.
Instead we use an approximate inversion.
We define $\tilde{P}^{-1}(p) = (\nu,\bar{g})$ where $\tilde{\nu}$ satisfies $P(\tilde{\nu}) = p$, $\bar{g}(x) = \mathbb{E} \tilde{\nu}_x$ and (where we use disintegration of measures) $\tilde{\nu} = \tilde{\nu}_x \otimes \nu$ with the latter meaning
\[ \tilde{\nu}(A\times B) = \int_A \tilde{\nu}_x(B) \, \dd \nu(x) \qquad \text{for all measureable } A\subset \Omega, B \subset \bbR^m. \]
In other words, we define the "inverse" map from the linear embedding of $\TLp$ back into $\TLp$ as the inverse map in $\Wp(\Omega\times\bbR^m)$ and projected onto $\TLp(\Omega)$:
\begin{align*}
\text{1: } & p \in \left\{ \text{linear } \TLp \text{ space} \right\} \mapsto \tilde{\nu} \in \cP(\Omega\times \bbR^m) \text{ using inverse of } \LWp \text{ in the space } \Omega\times\bbR^m \\
\text{2: } & \tilde{\nu} \mapsto (\nu,g)\in\TLp \text{ by projecting } \Wp(\Omega\times\bbR^m) \text{ onto } \TLp(\Omega).
\end{align*}
The projection onto $\TLp$ is done by taking the mean across each fibre in $x\in\Omega$.
We note that the second step does not have to be performed in the linear Wasserstein setting.
With this definition we are also able to visualise any point in the linear embedding in $\TLp$ space.


\section{Results} \label{sec:Results}

In this section we apply the $\LTLp$ framework to three real world examples to auslan (Australian sign language), breast cancer histopathology and financial time series.
We also include two synthetic examples and a further application to cell morphometry in the appendix.
Throughout we will choose $p=2$.

\subsection{Application to Auslan Data} \label{subsec:Results:Auslan}

We apply the transportation methodology presented in this manuscript to the Australian Sign language (Auslan) dataset~\cite{Kadous:2002}. A native Auslan signer was recorded, using fifth dimension technology gloves, making $95$ different signs repeated over a period of $9$ weeks. The sign was repeated $3$ times at each recording, thus each word was measured $27$ times. This means there are a total of $2565$ signs in the dataset. Each measurement is considered as a multivariate time-series. The measurements taken for each hand are the $x,y,z$ positions, along with roll, pitch and yaw. In addition, the bend of each of the $5$ fingers is recorded. Thus at each frame $22$ measurements are observed. We consider the Auslan data as functions $f_i:\mathbb{R} \to \mathbb{R}^{22}$, for $i = i,..., 2565$. We truncate the number of time frames to $44$ because little variation was observed past this point. The goal of this task is to classify signs given $f_i$ $i = i,..., 2565$ as input to $95$ possible words (labels) as output.

We apply the $\LTLp$ and $\LWp$ frameworks to this dataset. Since $\TLp$ can handle multi-channel signals, no additional pre-processing was needed. However, to apply $\Wp$ additional pre-processing was required. Firstly, all signals were made positive and then the mean was taken so that there was only a signal channel with positive values. We then normalised so the signal integrated to unity. A linear embedding was obtained as described in previous sections. Once this linear embedding is obtained, we use the 1 nearest neighbour (1NN) algorithm to predict the signs from the linear embedding of the signals. 
As an assessment of performance we use the macro-F1 score (the harmonic mean of the precision and recall)~\cite{He:2008}.
We assess performance with a $5$-fold cross-validation framework and repeat $100$ times to produce a distribution of scores. In addition, we compare to the standard $\TLp$ methodology since for this particular data set, even though costly it is possible to perform the full computation. 
We also recorded timings for each of the methods.

Table \ref{table::timings} shows that the linear transportation methods are considerably faster than the full transportation methods.
Indeed the $\LTLp$ distance was on the order of magnitude of ten's of seconds, whilst the full $\TLp$ took several hours.
Figure \ref{figure:AuslanCompare} demonstrates that our proposed $\LTLp$ method significantly outperforms the $\LWp$ approach (T-test, $p < 10^{-4}$) on the Auslan data. 
There is a loss in classification ability of $\LTLp$ versus the $\TLp$ method, which is unsurprising as the linear transportation method is approximate. However, on the Auslan dataset we observe this difference to be small and this minor improvement comes at computational cost orders of magnitude greater.

\begin{table}[ht!]
	\begin{center}
		\begin{tabular}{ |c|c|c|c| } 
			\hline
			Application & $\LWp$ & $\LTLp$ & $\TLp$\\
			\hline
			\hline
			Auslan & 12.1 & 13.0 & 91200\\
			\hline
			Breast Cancer Histopathology & 25407.0 & 2919.8 & $>345600$\\
			\hline
			Financial Time Series & 39.5 & 192.3 & - \\ 
			\hline
		\end{tabular}
	\end{center}
	\caption{CPU times in seconds to compute each transportation method on each dataset. Computation was halted after 4 CPU days ($=345600$ seconds).} 
    \label{table::timings}
\end{table}

\begin{figure}[ht]
	\begin{subfigure}[t]{0.48\textwidth}
		\centering
		\includegraphics[width=0.8\textwidth]{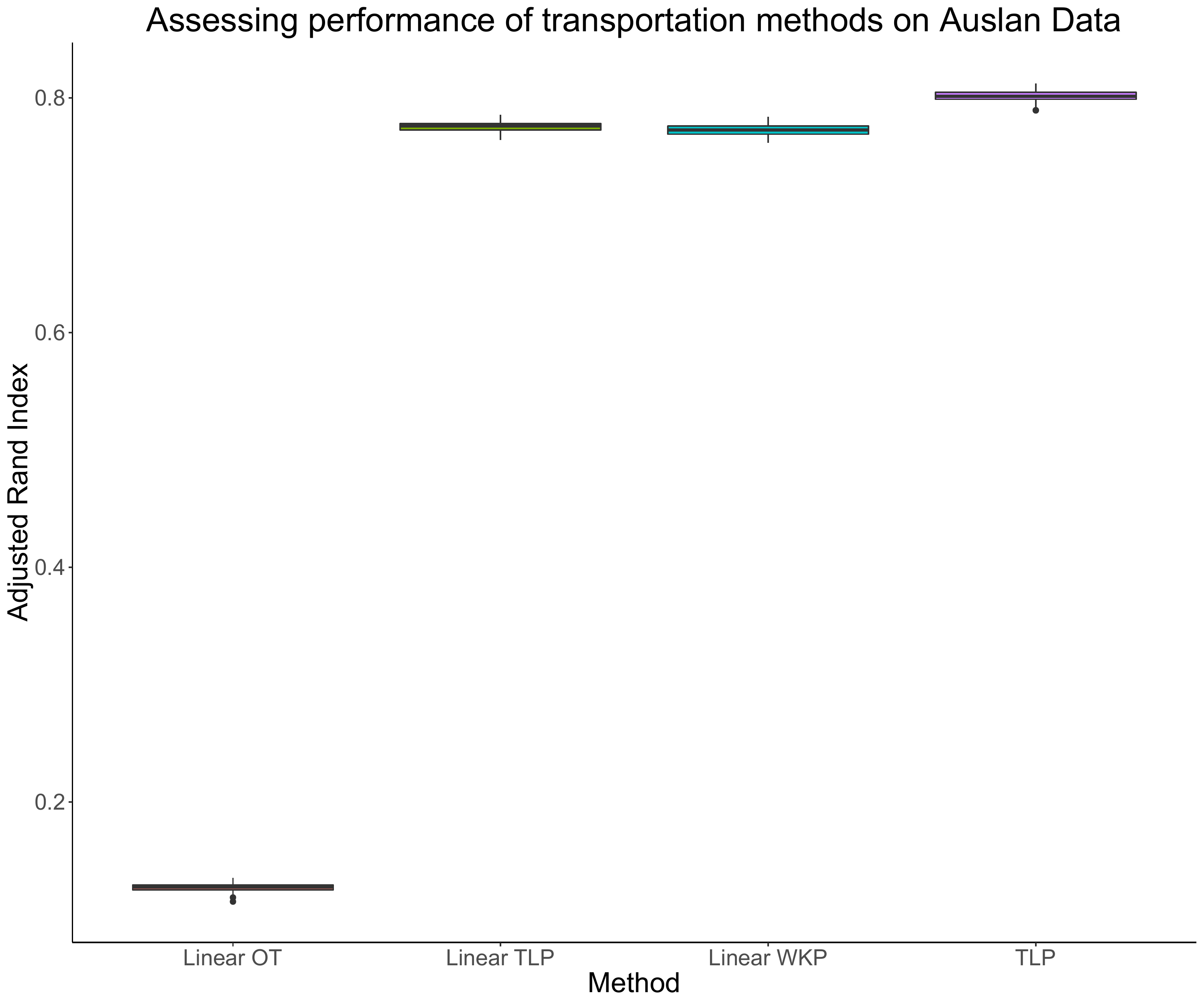}
		\caption{}
	\end{subfigure}
	\begin{subfigure}[t]{0.48\textwidth}
		\centering
		\includegraphics[width=0.95\textwidth]{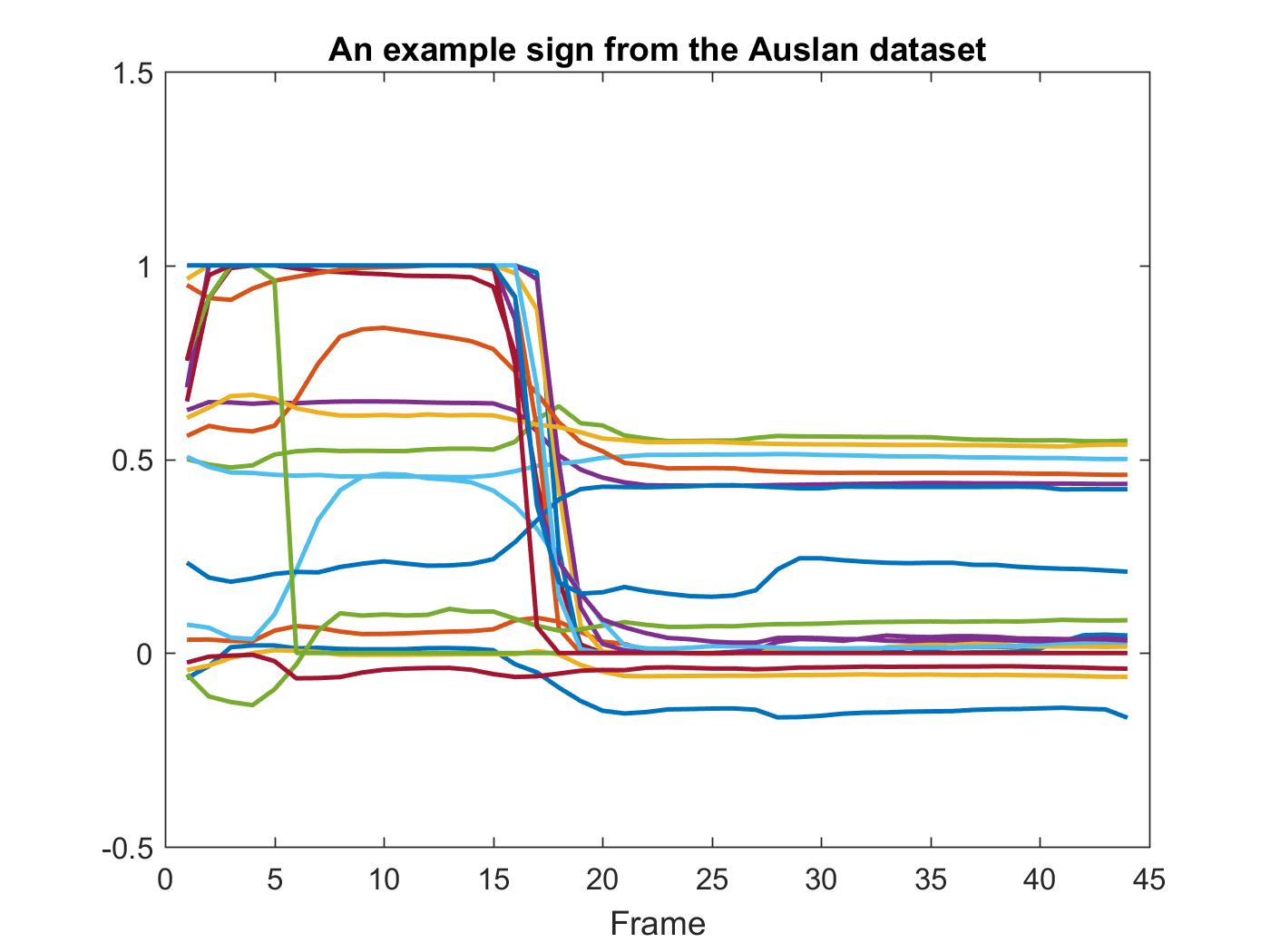}
		\caption{}
	\end{subfigure}
	\caption{(a) Distributions of macro-F1 scores for different transportations methods. Evidently $\LTLp$ significantly outperforms $\LWp$. 
	(b) An example multi-variate signal from the Auslan dataset. This represent a single word from a native signer. Each line represents data from fifth dimension technology gloves. For example one line represents the $x$ position of the left hand for a total of $44$ frames. A total of $22$ different lines represent a $22$ different measurements across the frames}.
	\label{figure:AuslanCompare}
\end{figure}

\subsection{Application to Breast Cancer Histopathology} \label{subsec:Results:Cancer}

In this section, we demonstrate the applicability of $\LTLp$ to images from breast cancer histopathology. Invasive Ductal Carcinoma is an aggressive and common form of breast cancer and deep learning based approaches have been used to construct classifiers to analyse such data \cite{Cruz:2014, Janowczyk:2016}. We analyse these datasets using transportation based approaches. We randomly sample $100$ images each from two patients $1$ healthy and $1$ cancerous, totalling 200 images. Each image is on a $50 \times 50$ pixel grid. To apply $\LWp$ to these images we first convert the images to a single intensity channel and renormalise so that the intensities integrate to unity. We apply the $\LTLp$ approach without any ad-hoc preprocessing, since it can be directly applied to un-normalised multi-channelled images. We compute transport maps in each case using entropy regularised approaches and then linearly embed these images 
as described in earlier sections. We visualise the linear embeddings using PCA. In addition, we again employ the 1NN classifier using the same framework as in the Auslan application. We report distributions of macro-F1 scores for both $\LWp$ and $\LTLp$.

From Table~\ref{table::timings} we see that the computational cost of the $\LTLp$ and $\LWp$ distances differs significantly more than might be expected. 
This is due to the $\LTLp$ distance converging more quickly (and therefore needing fewer iterations than the $\LWp$ distance (and conversely in the financial time series example).
The $\TLp$ distance was not included as a comparison as it was too expensive to compute.

Figure \ref{figure:cancerCompare} demonstrates clear differences between the $\LTLp$ embedding and $\LWp$ embedding. The cancerous and healthy images separate more obviously in the PCA representation of the $\LTLp$ embeddings. This is supported when using the 1NN classifier, where a mean macro-F1 score of $0.70$ is reported in the $\LWp$ case, whilst for $\LTLp$ the mean macro-F1 score is $0.88$ representing a greater than $25\%$ improvement. It is clear from the box plots that $\LTLp$ outperforms $\LWp$ (T-test $p < 10^{-16}$). Using linear interpolation, we visualise perturbations, in units of standard deviation, in principal component space as synthetic images (Figure \ref{figure:cancerInterp}). The interpolation in the embedding produced by $\LTLp$ demonstrates localised mass moving from the centre of the image towards the edges. This corresponds to cancer invading the milk ducts in cancerous tissue with open milk ducts in non-cancerous tissue. Linear interpolation in the $\LWp$ embedding visualises mass moving from the lower right to the upper left of the plot, there is no physical interpretation for this variation. It is clear that the $\LTLp$ synthetic images are more interpretable than the $\LWp$ synthetic images.

\begin{figure}[ht]
	\begin{subfigure}[t]{0.33\textwidth}
		\centering
		\includegraphics[width=0.95\textwidth]{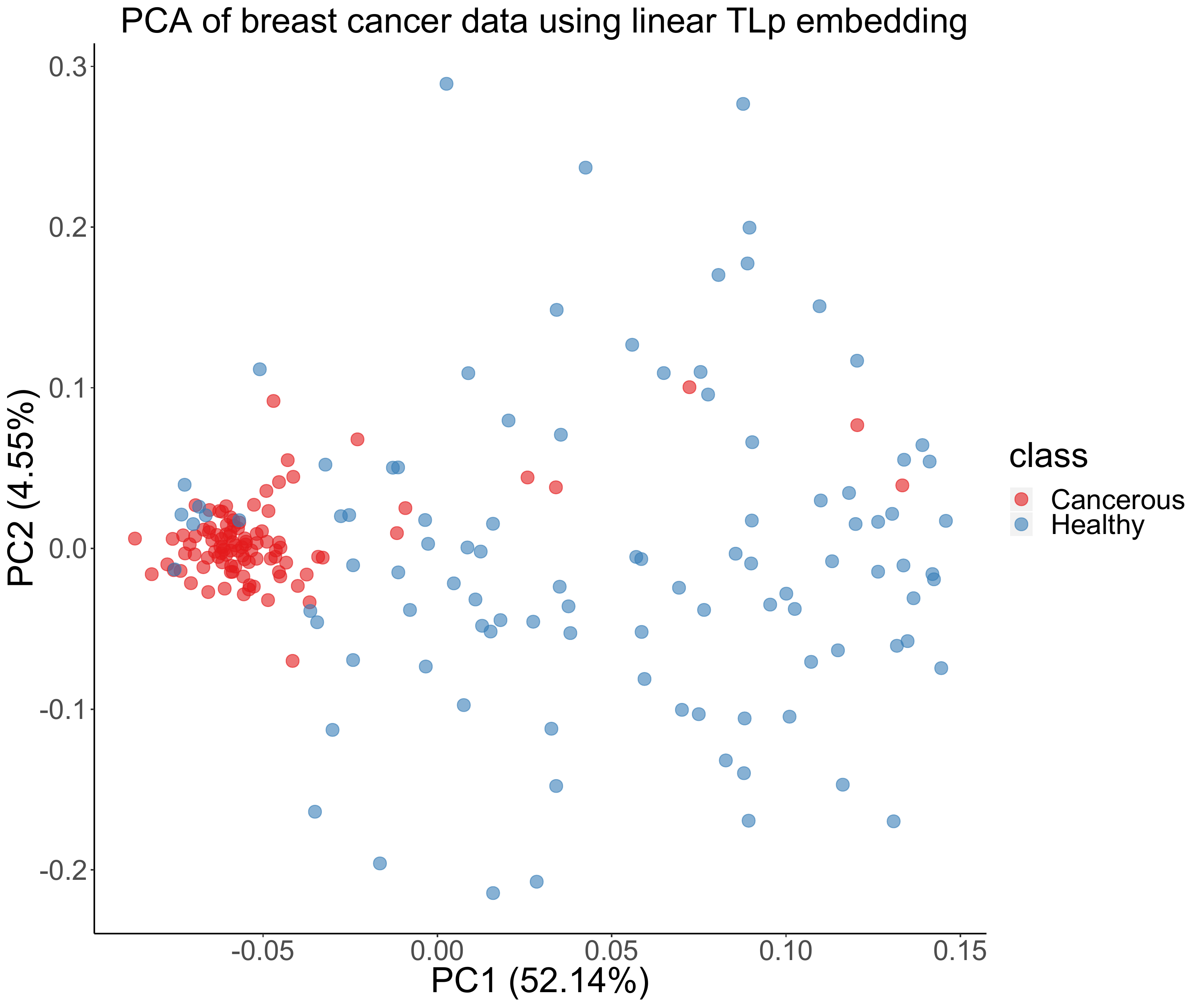}
		\caption{}
	\end{subfigure}
	\begin{subfigure}[t]{0.33\textwidth}
		\centering
		\includegraphics[width=0.95\textwidth]{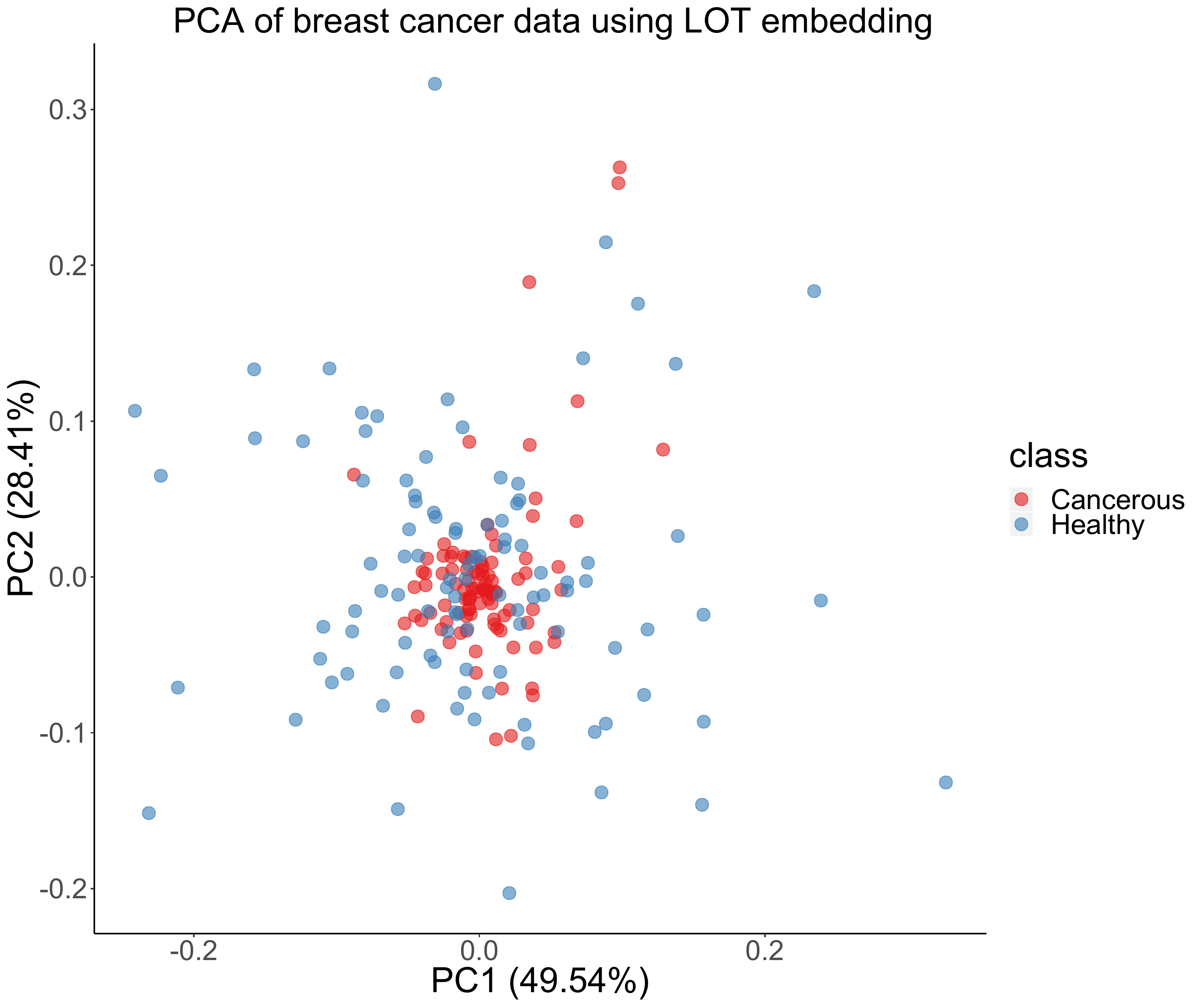}
		\caption{}
	\end{subfigure}%
	\begin{subfigure}[t]{0.33\textwidth}
		\centering
		\includegraphics[width=0.95\textwidth]{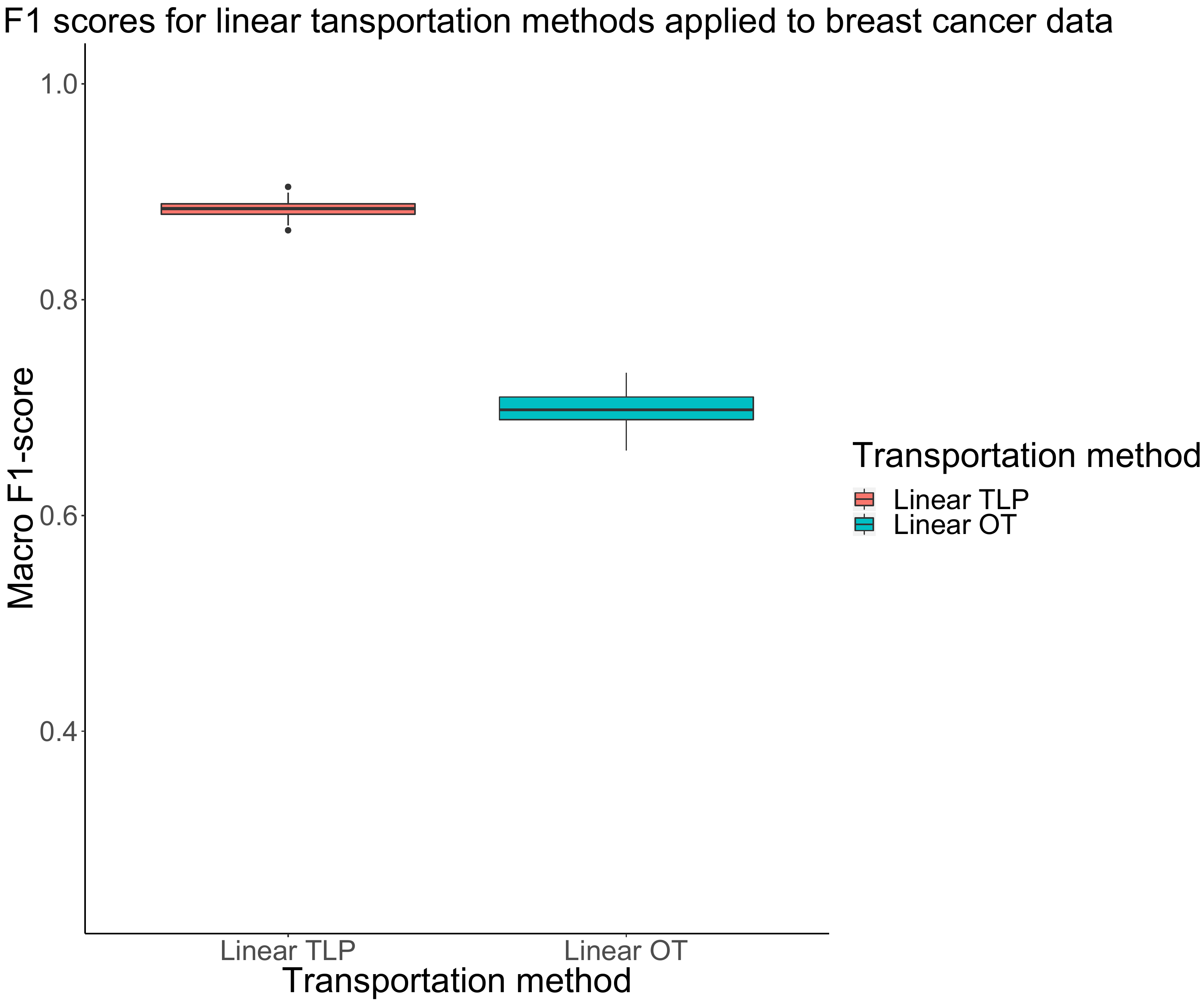}
		\caption{}
	\end{subfigure}%

	\begin{subfigure}[t]{0.49\textwidth}
	\centering
	\includegraphics[width=0.646\textwidth]{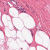}
	\caption{}
\end{subfigure}%
\begin{subfigure}[t]{0.49\textwidth}
	\centering
	\includegraphics[width=0.646\textwidth]{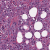}
	\caption{}
\end{subfigure}%
	
	\caption{(a,b) PCA plots using the $\LTLp$ embedding and the $\LWp$ embedding, we observe that $\LTLp$ embedding separates classes. (c)  Distributions of macro F1 scores for $\LWp$ and $\LTLp$ embeddings in the application to breast cancer histopathology using the 1NN classifier. 
	(d) An example image from a healthy patient. (e) An example image from a patient with breast cancer.}
	\label{figure:cancerCompare}
\end{figure}

\begin{figure}[ht!]
	\centering
	\includegraphics[height=4in]{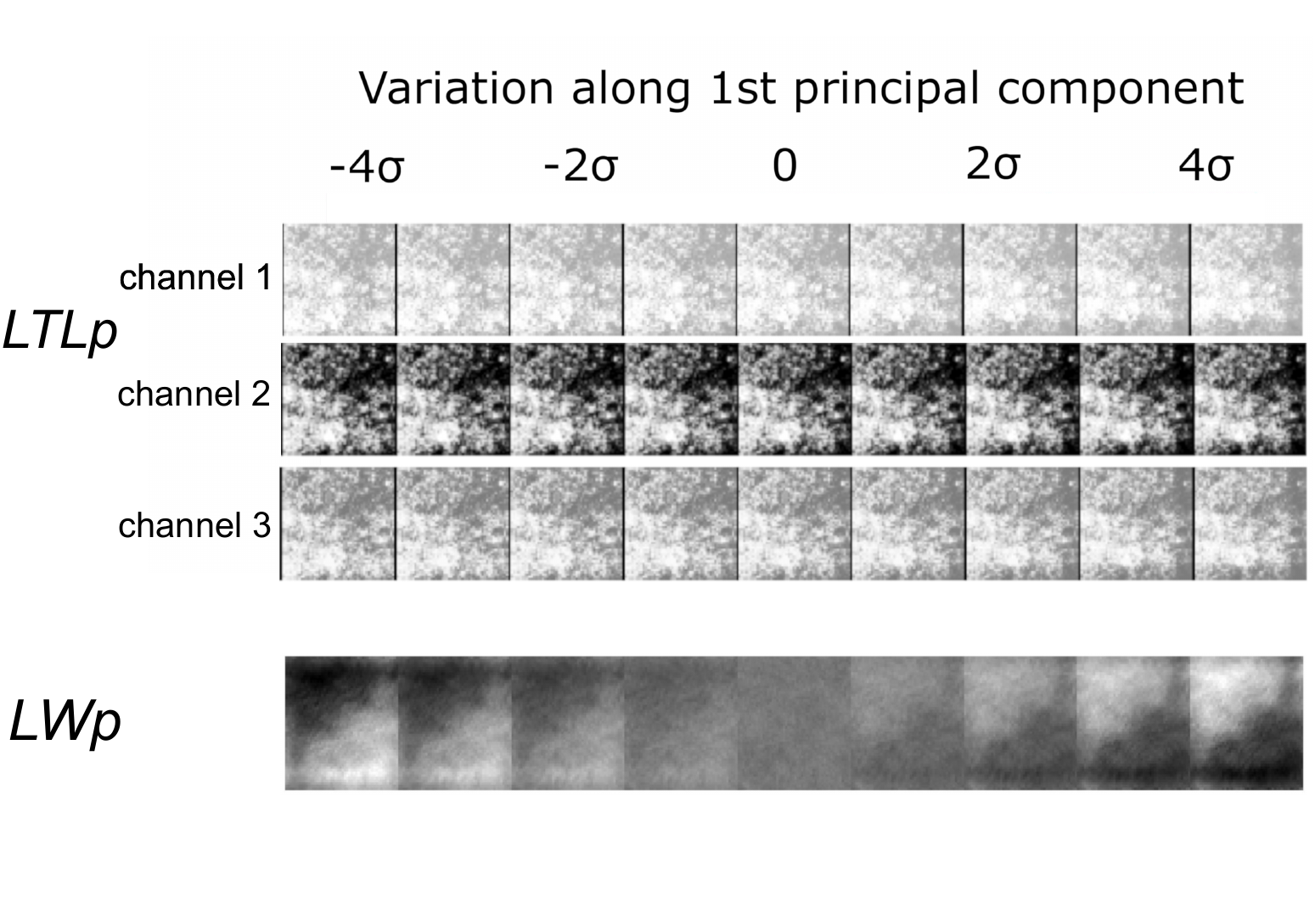}
	\caption{Variations along the first principal component in the application to breast cancer histopathology for both the $\LTLp$ embedding and the $\LWp$ embedding.
	}
	\label{figure:cancerInterp}
\end{figure}


\subsection{Application to Financial Time Series} \label{subsec:Results:Financial}

In this section, we consider an application of transportation distances to financial time series data.
ore specifically, we use daily close prices for constituents of the SP1500 index, for the period 2$^\text{nd}$ January 2004 - 23$^\text{rd}$ March 2020.

\subsubsection{Experimental Setup and Background on Financial Time Series}
We only consider  instruments (stocks) available throughout the entire history, which amounts to approximately $n=1150$. We use daily log-returns, with the return of instrument $i$, between times  $t_1$ and $t_2$, defined as $\rmR_{i}^{(t_1, t_2)} = \log{ \frac{ P_{i, t_2}}{ P_{i,t_1}}},$ where $P_{i, t}$ denotes the price of instrument $i$ at the end of day $t$. On any given day $t$, we consider the matrix $S_t=[P_{t-m+1},\cdots,P_t]$ of size $n \times m$ (where $P_t$ is the vector $[P_{1,t},\cdots P_{1150,t}]^\top$), capturing the daily returns for the past $m$ days (with fixed $m=20$ throughout the experiments). We refer to $S_t$ as the sliding window, as we vary the time component. The goal is to use and compare various techniques for computing the k-nearest-neighbors of $S_t$, which we denote by
\[ \mathcal{N}_k(S_t) = \lb t_j \,:\, j=1,\dots,k \text{ s.t. } S_{t_1}, S_{t_2}, \ldots, S_{t_k} \text{ are the } k \text{ nearest neighbors of } S_t\rb, \]
by pooling together the similarities between the multivariate time series comprising $S_t$ with prior historical time series $ S_{m}, \ldots, S_{t-h}$, where $h$ is the future horizon at which we aim to predict. Note that, at any given time $t$, the available history to query for the k-nearest-neighbors of $S_{t}$ ends at $S_{t-h}$, in order to avoid look forward data snooping. Figure \ref{fig:schematicDigramKnnFinance}  is a schematic diagram of our pipeline process.    
\begin{figure}[h!]
\begin{center}
\includegraphics[width=0.80\columnwidth]{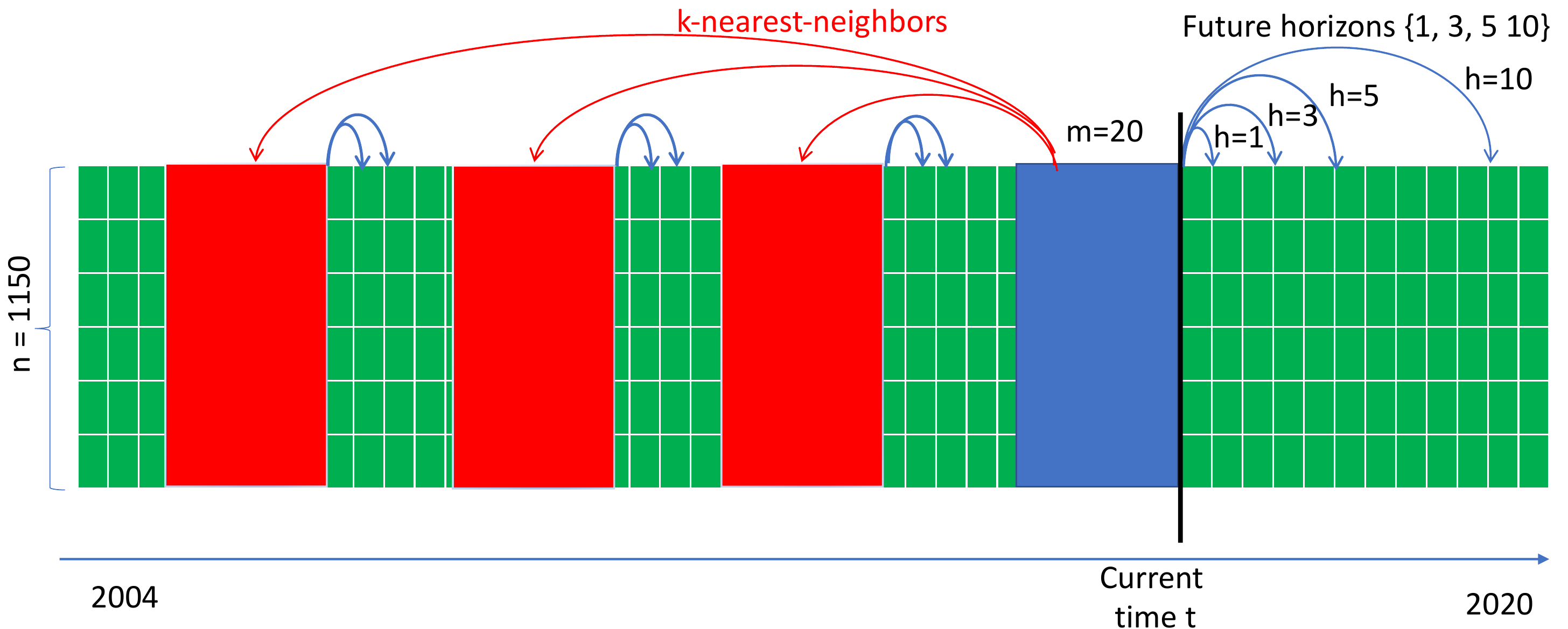}
\end{center}
\vspace{-6mm}
\caption{Schematic diagram of our pipeline. 
We use a sliding window approach, where the current window is shown in blue, and the historical k-nearest-neighbor windows in red.}
\label{fig:schematicDigramKnnFinance}
\vspace{-4mm}
\end{figure}

\paragraph{Future returns} 
We consider forward looking returns (referred to as targets) of two different types, and at various horizons. 
The future return can be the raw returns itself, or various decompositions of it. We let 
$\rmf_{i,t}^{(h,\RR)} = \log{ \frac{ P_{i, t+h}}{ P_{i,t}}}$ denote the raw return of instrument $i$ at time $t$, with future horizon $h$.  
In our setting, we consider the following future horizons $h \in \{1,3,5,10\}$ days. One often uses the S\&P500 as a proxy for the entire market return. The S\&P500 is a stock market index that measures the stock performance of 500 large companies listed on stock exchanges in the US market. 
The index has a corresponding ETF (Exchange Traded Fund), which can be traded much like any other regular stock (with symbol \textsc{SPY}), and we denote its raw return by 
$\rmf_{\SPY,t}^{(h,\RR)} = \log{ \frac{ P_{\SPY, t+h}}{ P_{\SPY,t}}}$. 
With this in mind, for any instrument $i$, one can then consider its future \emph{market-excess return} (MR) as
$ \rmf_{i,t}^{(h,\MR)} = \rmf_{i,t}^{(h,\RR)} - \beta_i \rmf_{\SPY,t}^{(h,\RR)}$.
For simplicity, we assume $ \beta_i = 1$ across all instruments $i=1,\dots,n$, though there are various techniques to infer the individual betas from historical prices  \cite{hollstein_prokopczuk_2016}. We remark that the main reason for benchmarking our predictions against market-excess returns, as opposed to only raw returns, is essentially to hedge away the market risk and increase the Sharpe Ratio score defined further below. 

\paragraph{Estimates of future returns} Once we have identified the k-nearest-neighbor (knn) periods of the current window $S_t$, the prediction made at time $t$, for a given horizon $h$, is a weighted sum of the corresponding historical future returns, where the weights are inversely proportional to the knn distances. More precisely, if we denote by $\hat{\rmf}_t^{(h,\RR)}$ the $n \times 1$ vector of forecasts for the future $h$-day raw returns at time $t$, its values are given by 
\[ \hat{\rmf}_t^{(h,\RR)} = \sum_{t_j\in\cN_k(S_t)} w_j \rmf_{t}^{(h,\RR)}, \]
where the weights $w_i$ are given by $w_j\sim\frac{1}{d(S_t,S_{t_j})}$, normalized such that $\sum_{t_j\in\cN_k(S_t)} w_j = 1$, where $d(S_t, S_{t_j})$ denotes the distance between the current window $S_t$, and its nearest historical neighbors $S_{t_j}, t_j \in  \mathcal{N}_k(S_t)$.
Similarly, we compute estimates for the future $h$-day market-excess returns (MR), by pooling together k-nn historical forward looking market-excess returns.

\paragraph{P\&L} For performance evaluation, we rely on standard metrics from the finance  literature. 
For a given set of forecasts, the corresponding \textsc{PnL} (Profit and Loss) on  day $t$ for a given return $\rmf_t$ (shortly chosen to be the raw return $\hat{\rmf}_t^{(h,\RR)}$ or the market excess return $\hat{\rmf}_t^{(h,\MR)}$) is defined by  
\[ \PnL_t= \sum_{i=1}^n  \sign( \alpha_{i,t}) \cdot  \rmf_{i,t}, \quad  t = 1, \ldots, T, \] 
where $\alpha_{i,t}$ denotes our forecast for stock $i$ on day $t$. Note that the PnL increases if and only if the sign of the forecast $\alpha$ agrees with the sign of the future return $\rmf_{i,t}$, and decreases otherwise. The sum is across all the $n$ instruments, and $\rmf_{i,t}$ is the future return (either raw return (RR) or market-excess return (MR)) of stock $i$ on day $t$. We explore different forward looking horizons $ h \in \{1,3,5,10\}$). 
We add a superscript to the PnL calculation to indicate its dependency on horizon $h$ and the type of return considered (RR or MR). For instance, in our setting, the $h$-day forward looking market-excess return, computed daily, is given by 
%
\[ \PnL_t^{(h,\MR)}= \sum_{i=1}^n  \sign \left(  \hat{\rmf}_{i,t}^{(h,\MR)} \right)  \cdot  \rmf_{i,t}^{(h,\MR)}, \quad  t = 1, \dots, T, \quad  h \in \{1,3,5,10\}. \] 

\paragraph{Sharpe Ratio}
After computing the daily PnL time series for all available days, 
we capture the risk-adjusted performance by computing the corresponding (annualized) Sharpe Ratio 
%
\[ \text{Sharpe Ratio (SR)} := \frac{ \mean(\PnL)}{ \stdev(\PnL)} \times  \sqrt{252}, \]
%
where the scaling is due to the fact that there are 252 trading days within a calendar year. For simplicity, we apply the same scaling $\sqrt{252}$ also to the longer horizons $h>1$, and refer the reader to \cite{benhamou2019testing} for an in depth discussion on 
Sharpe Ratios\footnote{The annualized Sharpe Ratio is calculated from daily observations as $\frac{\mu - r_f}{\sigma} \sqrt{252}$,  where $\mu$ is the average daily PnL return, $r_f$ denotes the risk-free rate, and $\sigma$ the standard deviation of the PnL returns. Since the risk-free rate is 
close to zero over the period of study, we compute the Sharpe Ratio as $\frac{\mu}{\sigma} \cdot \sqrt{252}$.} 
and practical considerations arising from the fact that typical forecasts for equity returns are usually serially correlated.  
We attribute the future $h$-day PnL to each day $t$ (leading to overlapping windows), as opposed to maintaining $h$ parallel portfolios, and computing their daily total PnL. We are mainly interested in relative performance of the methods, in terms of Sharpe Ratio and PnL, and less on the actual magnitudes of these performance metrics, and their practical considerations.

\paragraph{Average PnL in basis points}
In the financial literature, a typical performance measure is the average return per dollar traded, in percentage. For example, one is typically interested in the annualized return of the portfolio. In what follows, we denote by PPT (PnL Per Trade) the average daily PnL per unit of notional (eg., \$1). For instance, if at time $t_0$ the available capital is \$100, and the cumulative PnL at time $t_{252}$ (thus after one year) is  \$10, then the annualized return amounts to 10\%. Recalling that 1\% amounts to 100 basis points (bpts), an annualized return of 10\% translates to approximately PPT = 4 bpts per day (since $4 \times 252 \approx  1000$ bpts, which amounts to 10\%). 
Essentially, the PPT is telling us how much would we earn for each \$1 traded in the markets (excluding transaction costs and fees). For simplicity, we ignore sizing/liquidity effects and assume that each day, we invest \$1 for each of the $n$ instruments, which leads us to the following simplified notion of PPT, averaged over the entire trading period comprised of $T$ days 
%
\[ \PPT = \sum_{t=1}^T  \frac{\PnL_t }{ n } = \frac{ \sum_{t=1}^T  \PnL_t }{ T n }.\]

\paragraph{Quintile Portfolios}
One is often interested in understanding the performance of the forecasts, as a function of their respective magnitudes. To this end, one typically considers only a subset (eg, top $q$\% strongest in magnitude forecasts) of the universe of stocks, usually referred to as \textit{quantile} portfolios in the literature \cite{Fama}. 
A quantile-based analysis simply constructs and evaluates portfolios composed of stocks which fall in a specific quantile bucket, or above a quantile threshold. We choose to use upward-contained quintile buckets, which we denote by $qr_i, i=1, 2, \ldots, 5$ indicating the quintile rank of each stocks, meaning that stocks with quintile rank $qr_i$ correspond to the top $1 - \frac{i-1}{5}$ fraction of  largest-in-magnitude forecasts. The colors in Figures \ref{fig:PS_knn_100_MaxHist_Inf_BARS} denote the quintile portfolios traded based on the magnitude of the forecasts. For example, the red bars $qr_1$ correspond to the full universe of stocks, while the green bars $qr_4$ denote the top 40\% largest in magnitude forecasts.

\subsubsection{Methods Comparison}

We compare the prediction performance of $\Wp$, $\LWp$ and $\LTLp$, and leave out the full $\TLp$ due to its prohibitive computational running cost. In addition, we compare to another  more classical approach, not relying on transportation distance methodology, given by the simple Pearson correlation between the original time series. More precisely, for each window $S_t$ (matrix of size $n \times m, \; t =1, \ldots, T$), we first standardize the returns in each row (eg, corresponding to each stock), and denote the resulting matrix by $\tilde{S}_t, t =1, \ldots, T$. Next, we unwrap each  matrix $\tilde{S}_t$ into a vector $\psi_t  \in  \mathbb{R}^{nm}$, and finally compute the pairwise distance between a pair of time windows $S_i$ and $S_j$, using a correlation-based distance between their corresponding flattened versions: 
$ \COR_{ij} := 1 - \Corr(\psi_i,\psi_j). $ 
Note that, in light of the pre-processing step that standardized each stock, this corresponds, up to a scaling constant, to the squared Euclidean distance between the corresponding vectors  $|| \psi_i - \psi_j  ||_F^2$. 

Figure \ref{fig:DistHistBar} shows numerical results comparing the various methods considered. The left column is a heatmap showing the $ T \times T $ pairwise distance matrix between all days available in history, in the interval 2004 - 2020. We note that $\LTLp$ clearly highlights the financial crisis occurred in 2008, followed by $\LWp$, and to some extent, $\Wp$, while COR show barely visible signs of this event. The middle columns show a distribution of the pairwise distances, while the right columns show the  row sums of the distance matrix. Construing the distance matrix as a network with distance/dissimilarity information, this plot effectively plots the degree of each node (i.e., of each time period corresponding to a sliding window of length 20). The periods of time with the largest dissimilarity degree correspond to the financial crisis in 2008. Note that, for each of visualization, we standardize the degree vector.

Next, we zoom in into the $\LTLp$ pairwise distance matrix, and show the resulting degrees in Figure \ref{fig:TotalDistanceDegree_LTPL_annotated}. After computing the $\LTLp$ distance matrix, we interpret this as a distance network, and compute the total distance degree of each node, after standardization. In this plot, we are able to recognize many of the major financial market events that have happened over the last two decades: the big financial crisis of 2007-2008, the 2010 Flash crash, the August 2011 markets fall (between May-October 2011), the Chinese market crash from January 2016, the period Oct-Nov 2018 (when the stock market lost more than \$2 trillion),  August 2019 (a highly volatile month in the global stock markets), and finally, the February 2020 stock market crash triggered by the COVID-19 pandemic. It is interesting to observe that the distance degree corresponding to the COVID-19 pandemic is matched in magnitude only by the 2007-2008 financial crisis.
Furthermore, in the top right of Figure~\ref{fig:TotalDistanceDegree_LTPL_annotated} we plot the top $50$ eigenvalues of the $\LTLp$ distance matrix, 
highlighting the usual market mode top eigenvector, with the second eigenvector highly localized on the 2007-2008 financial crisis and the February-March 2020 Covid-19 pandemic period (plots of the top 5 eigenvectors are shown in the Appendix, see Figure~\ref{fig:eigenLTLP}).

\begin{figure}[h!]
\begin{center}
\includegraphics[width=0.6\columnwidth]{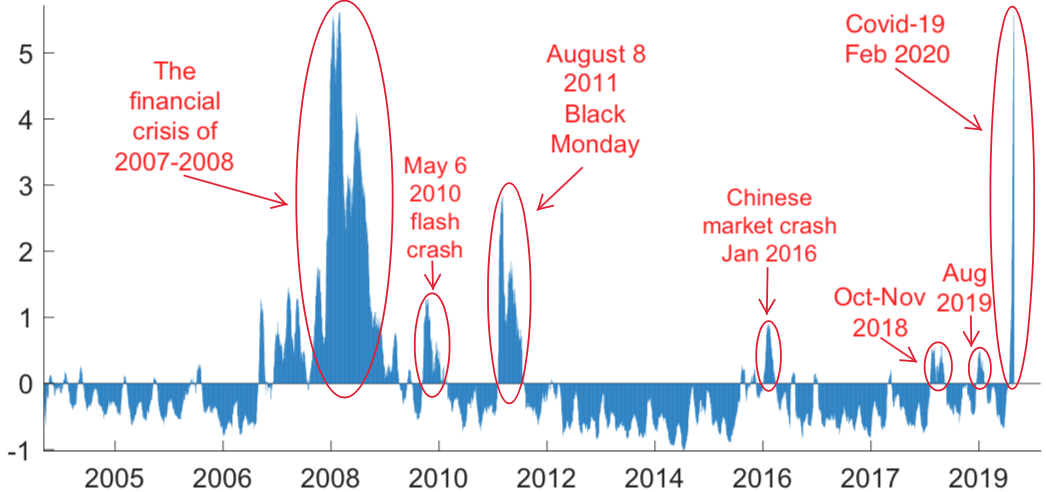}
\includegraphics[width=0.39\columnwidth]{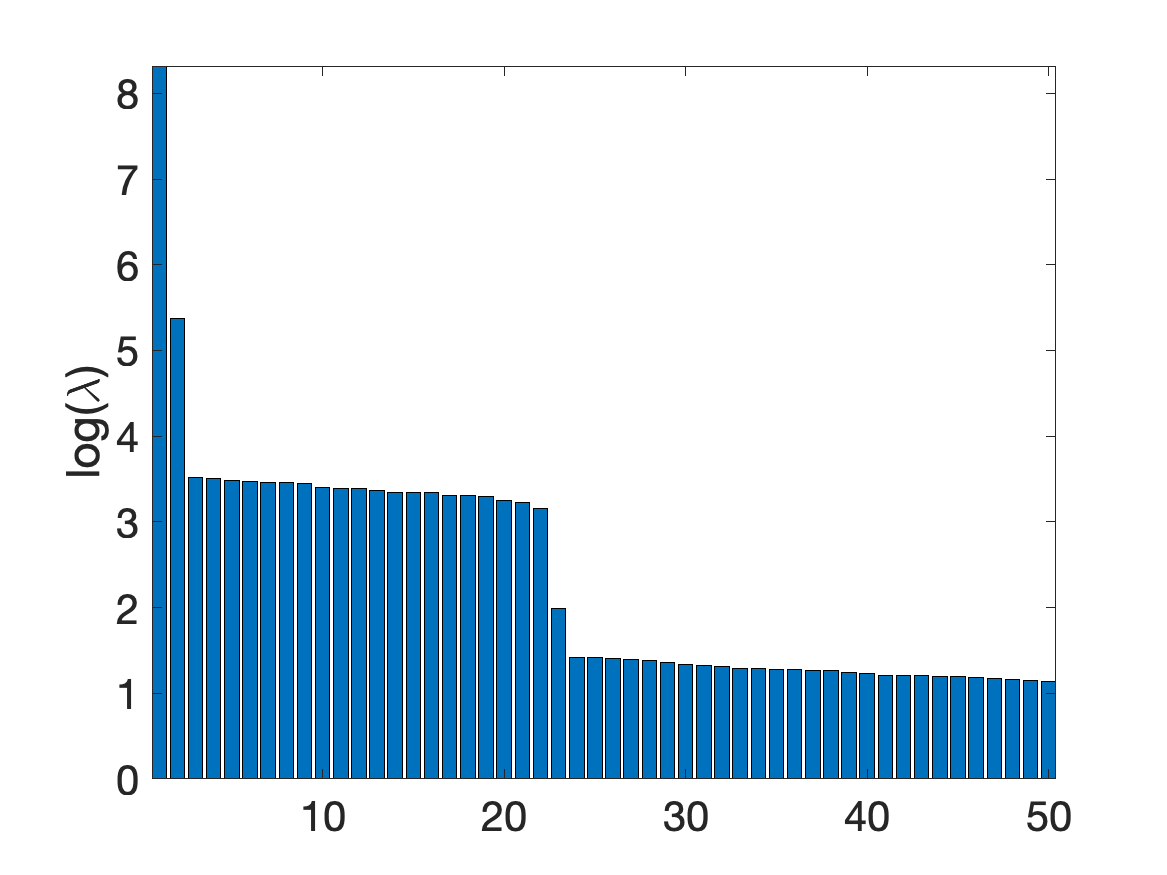}
\\ 
\includegraphics[width=0.32\columnwidth]{{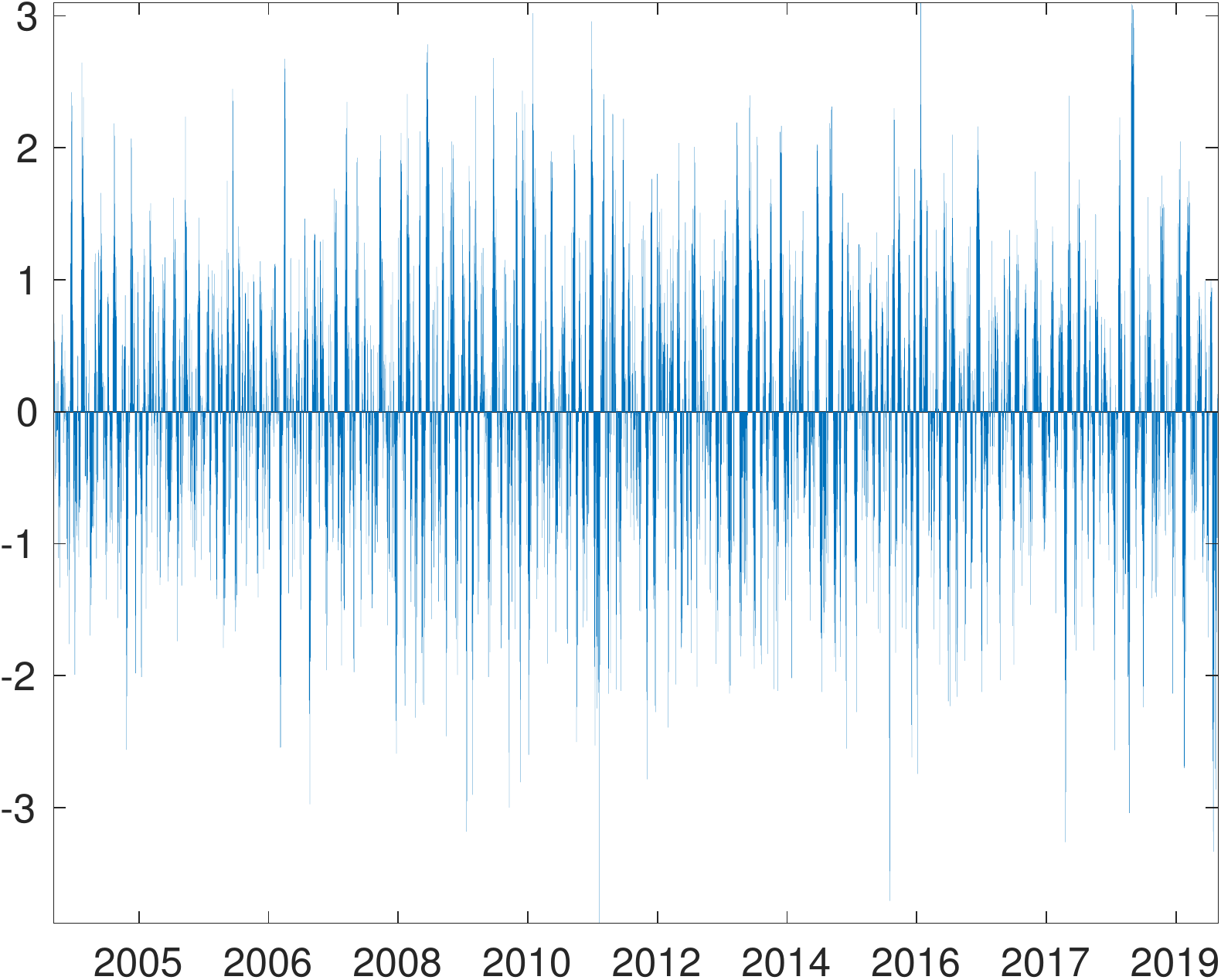}}
\includegraphics[width=0.32\columnwidth]{{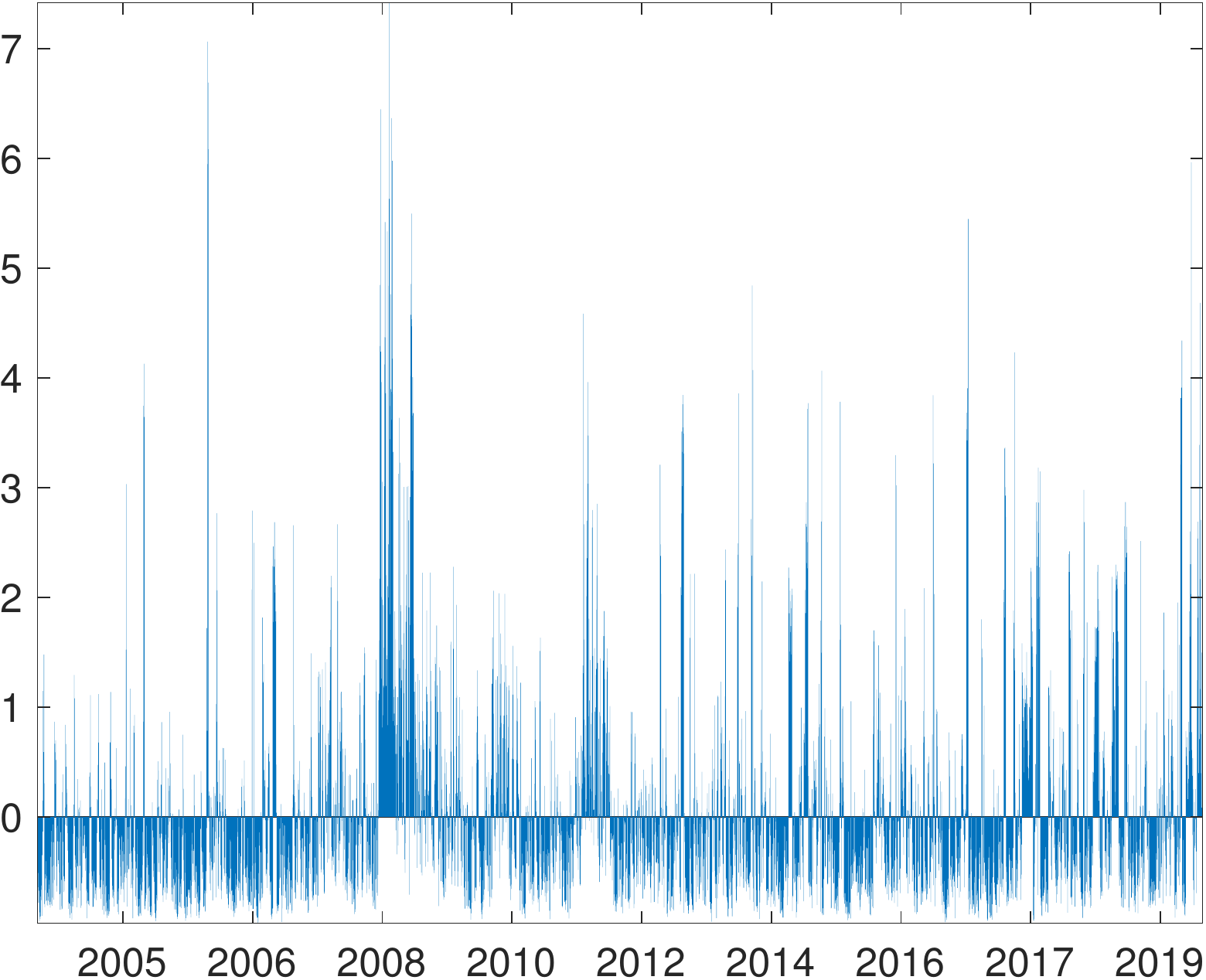}} 
\includegraphics[width=0.32\columnwidth]{{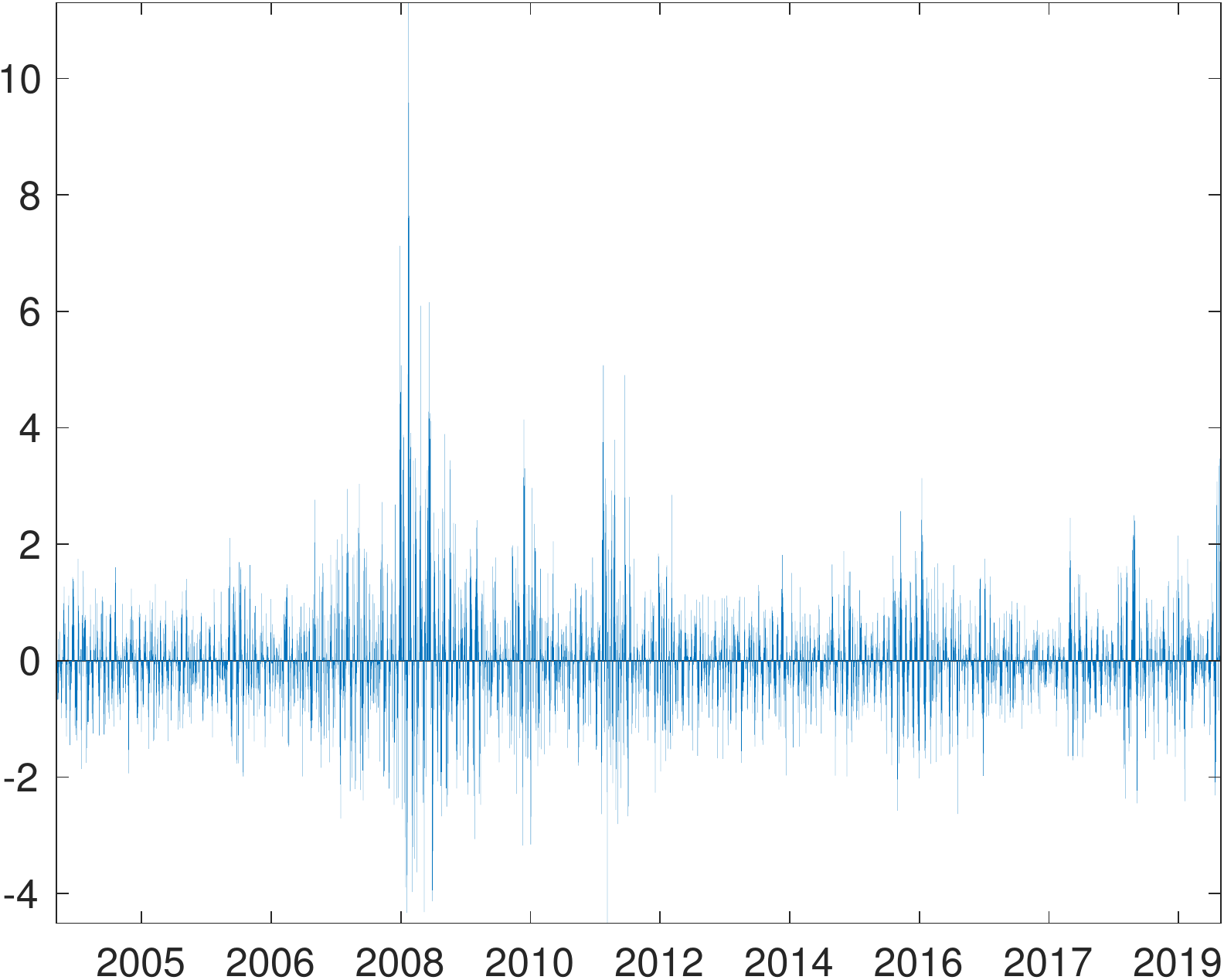}}
\end{center}
\vspace{-2mm}
\caption{Normalized total distance for each day, as computed via $\LTLp$ , annotated with the major market events. More explicitly, we compute the distance matrix between all pairs of days (where the data for a given by is given by the previous $m=20$ days, including the day of), construe this as a distance network, and compute the total degree of each node (which we show in the above figure, after standardization, for ease of visualization).}
\label{fig:TotalDistanceDegree_LTPL_annotated}
\end{figure}

Figure \ref{fig:PS_knn_100_MaxHist_Inf_BARS} shows portfolio statistics for the various methods, across different target future horizons  $h \in \{1,5,10\}$, for both raw-returns (RR) and market-excess returns (MR). The corresponding cumulative PnL plots across time are given in Figure~\ref{fig:PS_knn_100_MaxHist_Inf_EvoQR5} 
for future horizons. 
Here, we fixed the number of nearest neighbors to $k=100$, and allow the knn search to span back until the start of the available history $T=1$. When forecasting raw returns (left column in Figure \ref{fig:PS_knn_100_MaxHist_Inf_BARS}), all methods perform rather poorly, with COR and $\LWp$ showing the best performance for $h=1$, while for $h \in \{5,10\}$,  $\LTLp$ clearly outperforms all other methods.  
This supports the assumption that $\TLp$ is better able to model similarities in financial time series.

In the market-excess returns setting, for $h=1$, all methods return a similar performance in terms of Sharpe Ratio (SR) around 1, except for $\Wp$ which has a SR of around 0.5; in terms of PnL, most methods achieve a PPT of 1-3 basis points (bpts). However, for longer horizons, $\LTLp$ clearly outperforms all other methods, both in terms of Sharpe Ratio and PPT.


 
\begin{figure}
\newcommand{\wid}{0.22\textwidth}
\newcommand{\widh}{1.4in}
\newcolumntype{C}{>{\centering\arraybackslash}m{\wid}}
\begin{center}
\begin{tabular}{l*4{C}@{}}
& COR  & $\Wp$  & $\LWp$  & $\LTLp$  \\ 
& \includegraphics[width=\wid]{{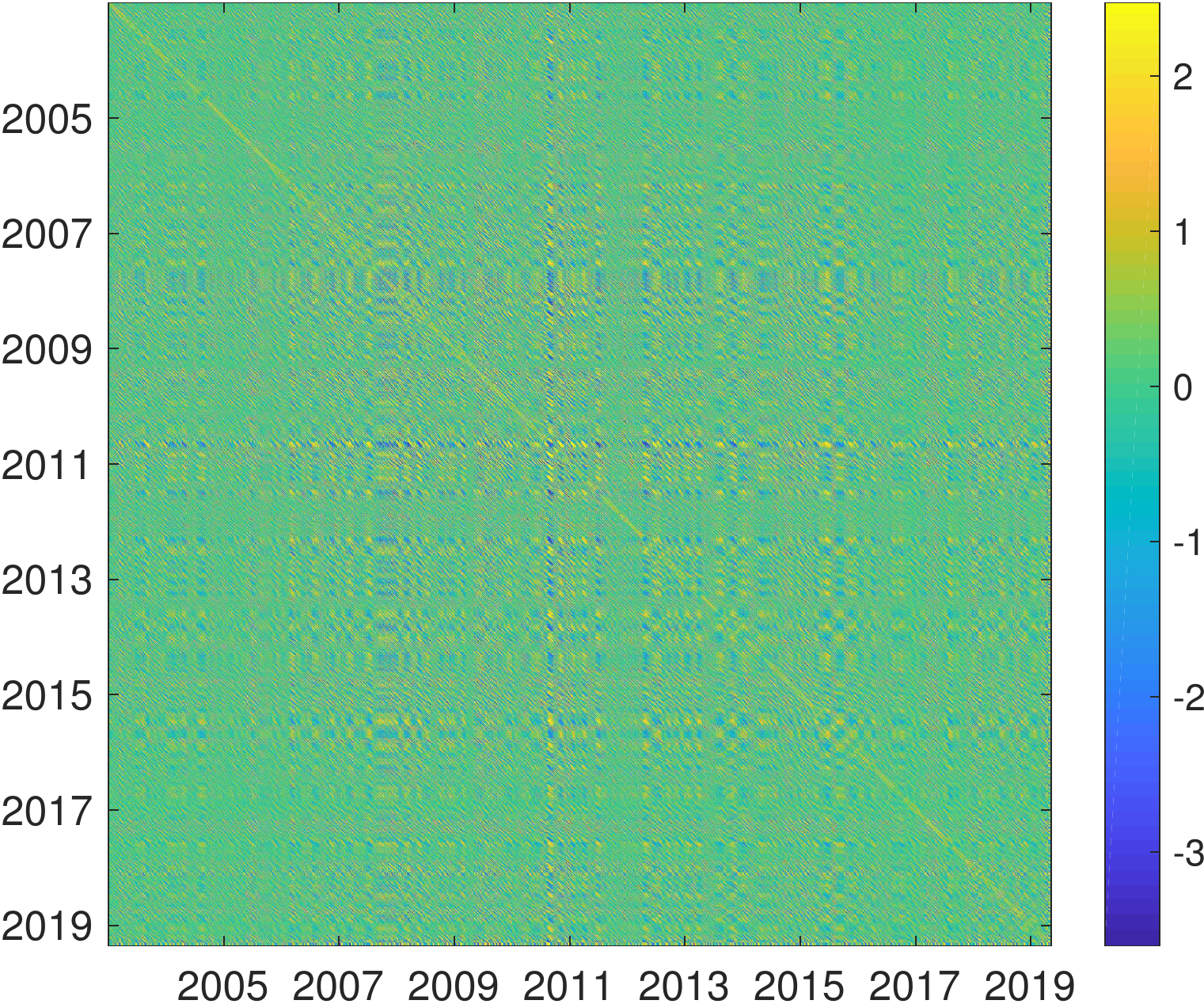}} 
& \includegraphics[width=\wid]{{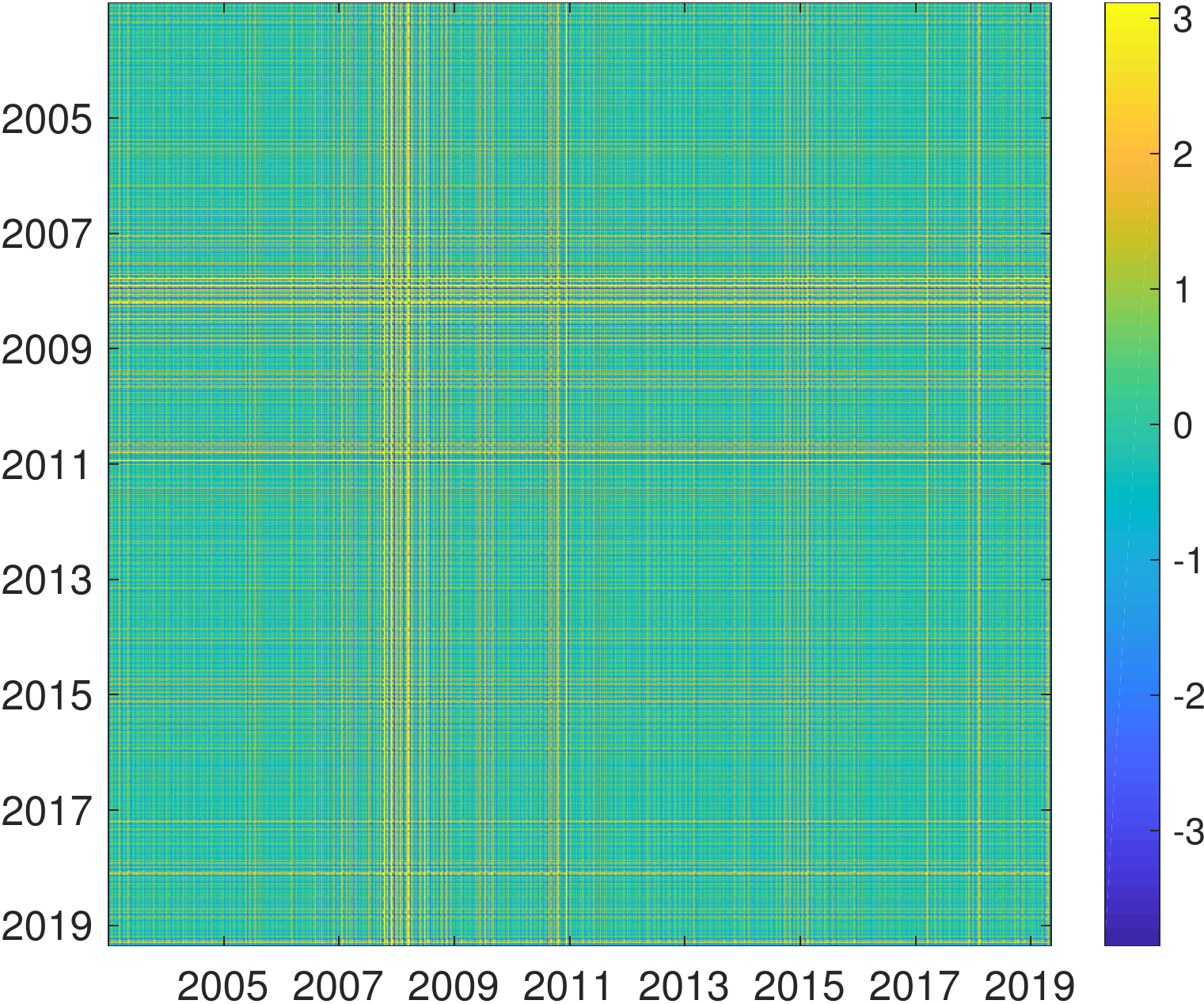}} 
& \includegraphics[width=\wid]{{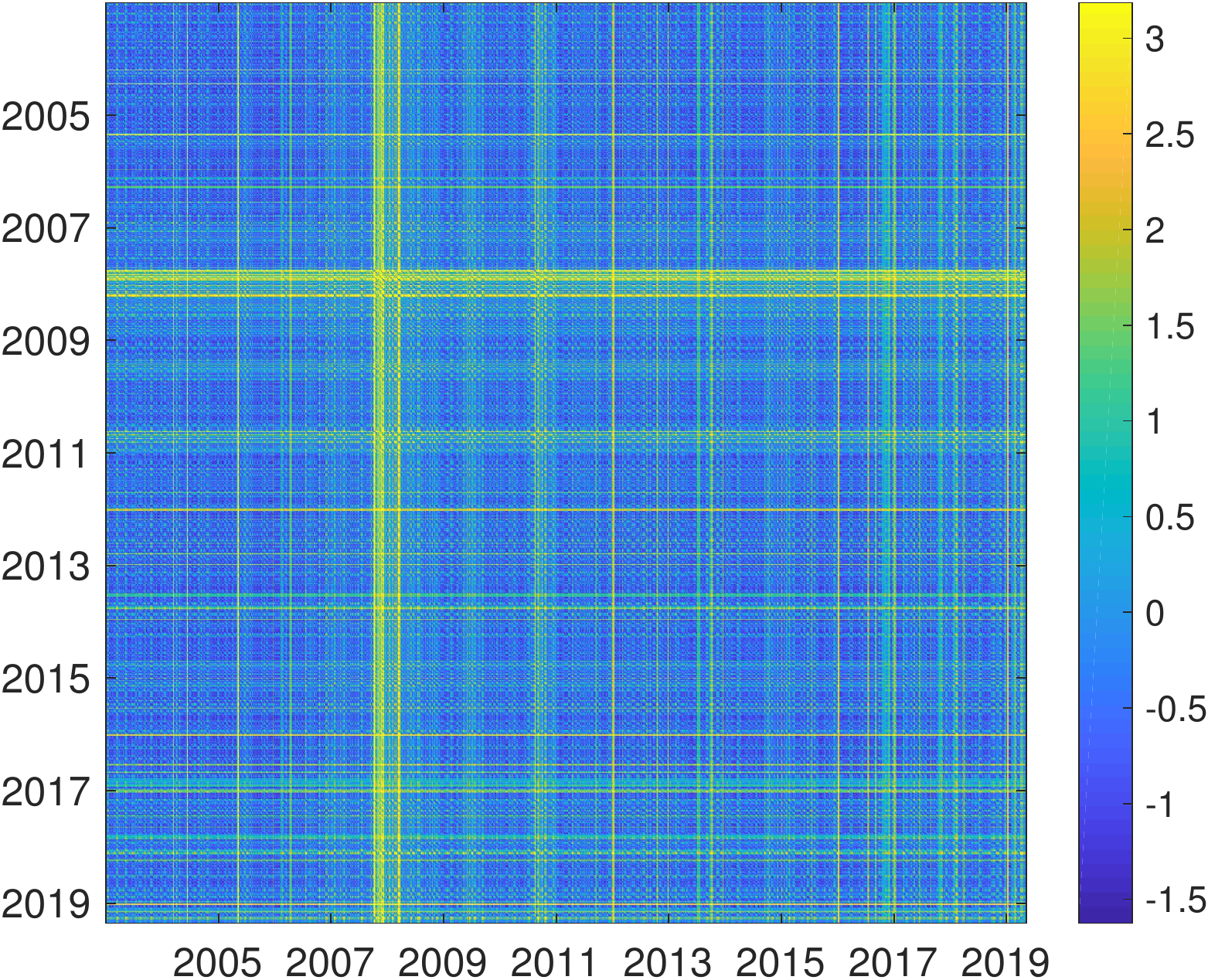}} 
& \includegraphics[width=\wid]{{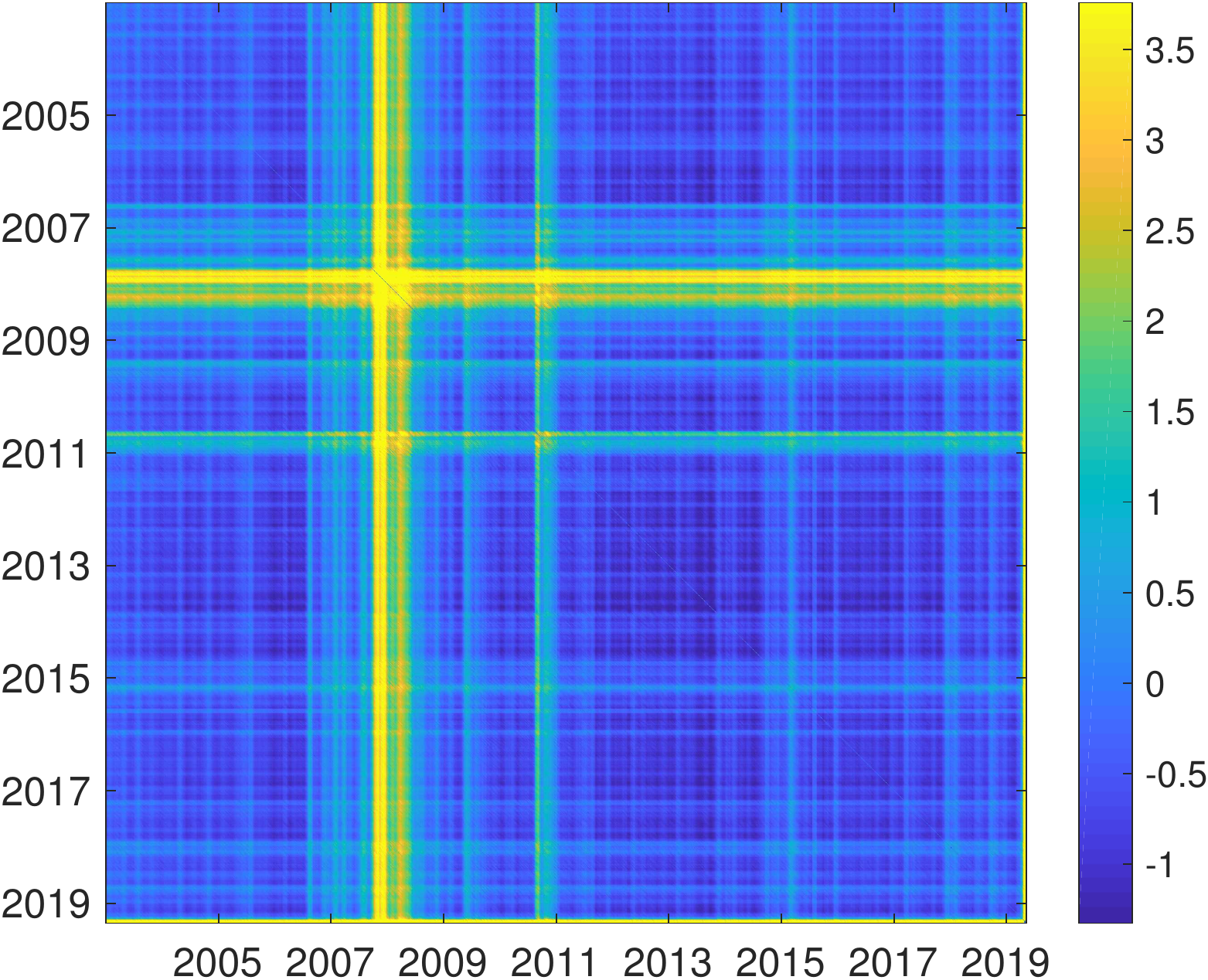}}\\
\end{tabular}
\end{center}
\caption{Pairwise distance matrices corresponding to each of the methods considered.}
\label{fig:DistHistBar}	
\end{figure}

\begin{figure}
\newcommand{\widd}{0.45\textwidth}
\newcolumntype{D}{>{\centering\arraybackslash}m{\widd}}
\begin{center}
\begin{tabular}{l*2{D}@{ }}
    &  Raw Returns (RR) & Market-Excess Returns (MR)  \\ 
\hline 
\rotatebox{90}{1-day}  
& \includegraphics[width=\widd]{{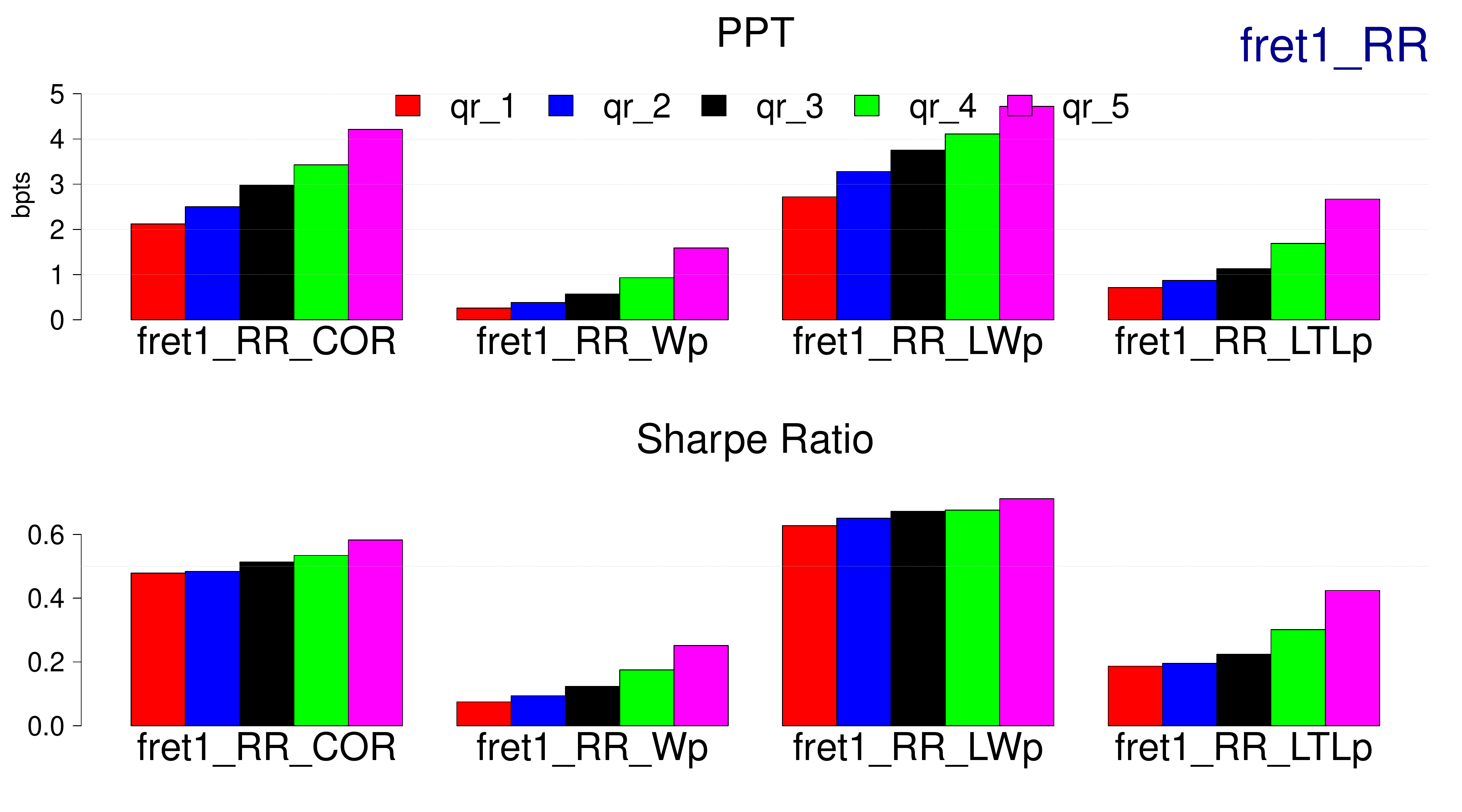}} 
& \includegraphics[width=\widd]{{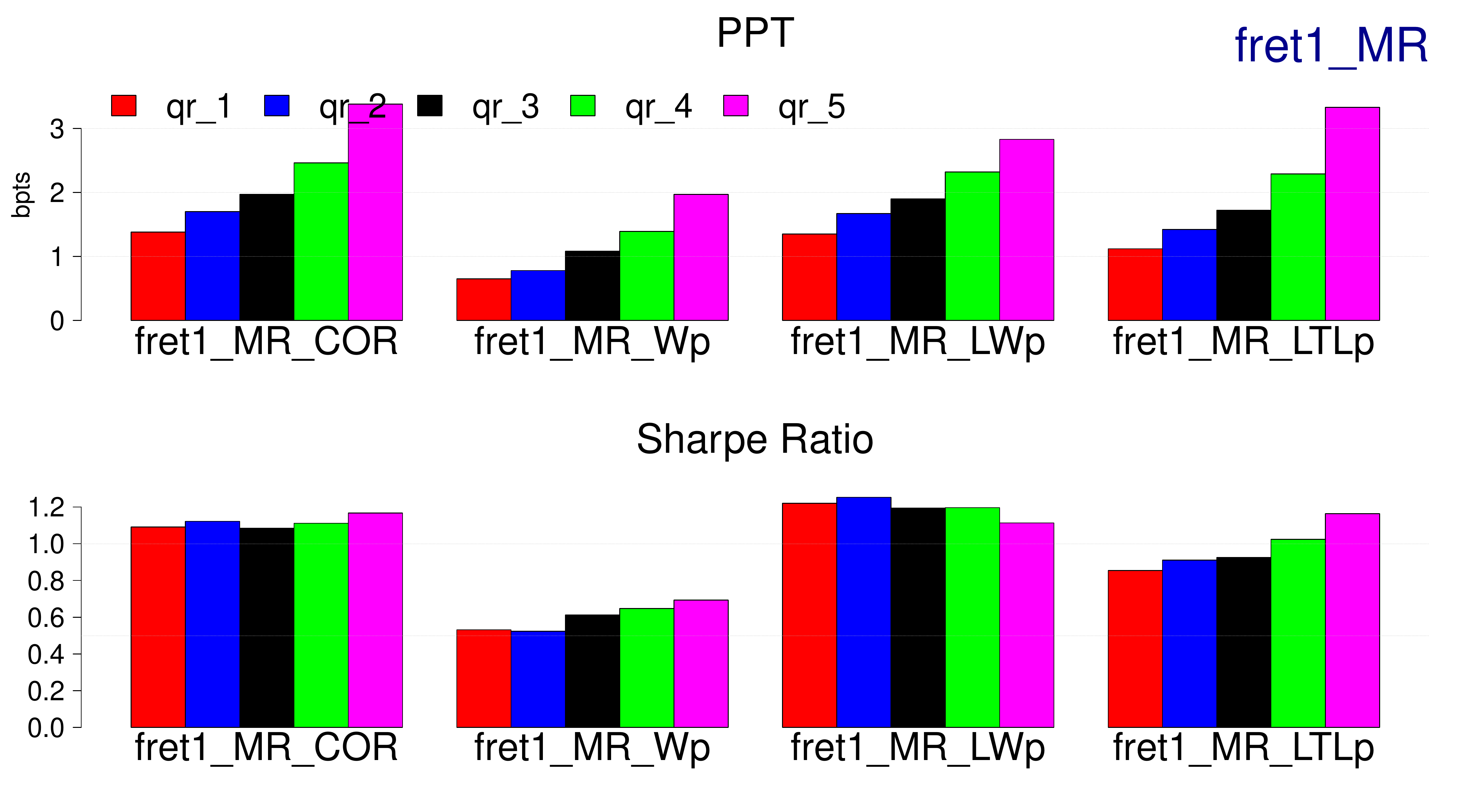}}  \\
%
\hline 
\rotatebox{90}{5-day}  
& \includegraphics[width=\widd]{{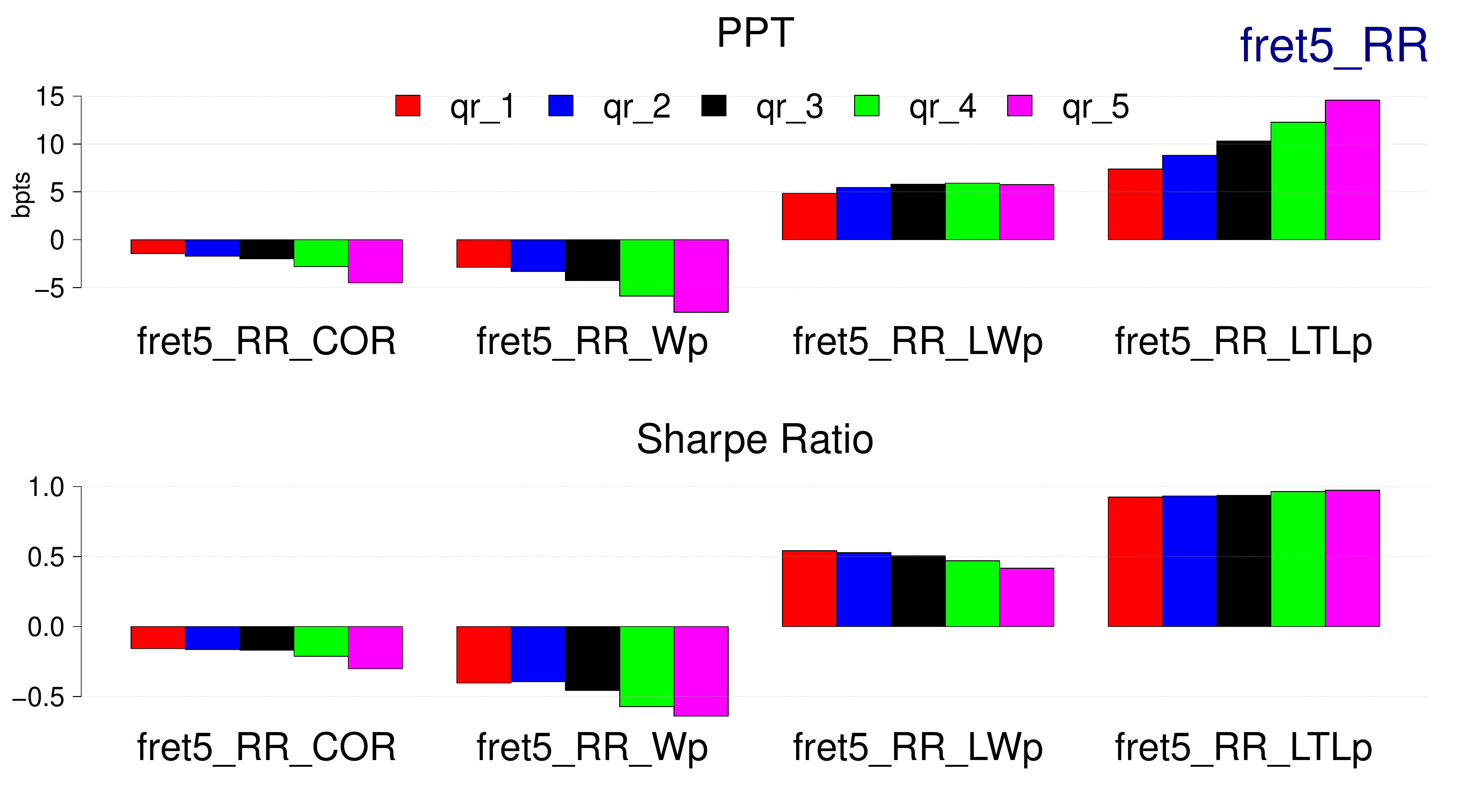}} 
& \includegraphics[width=\widd]{{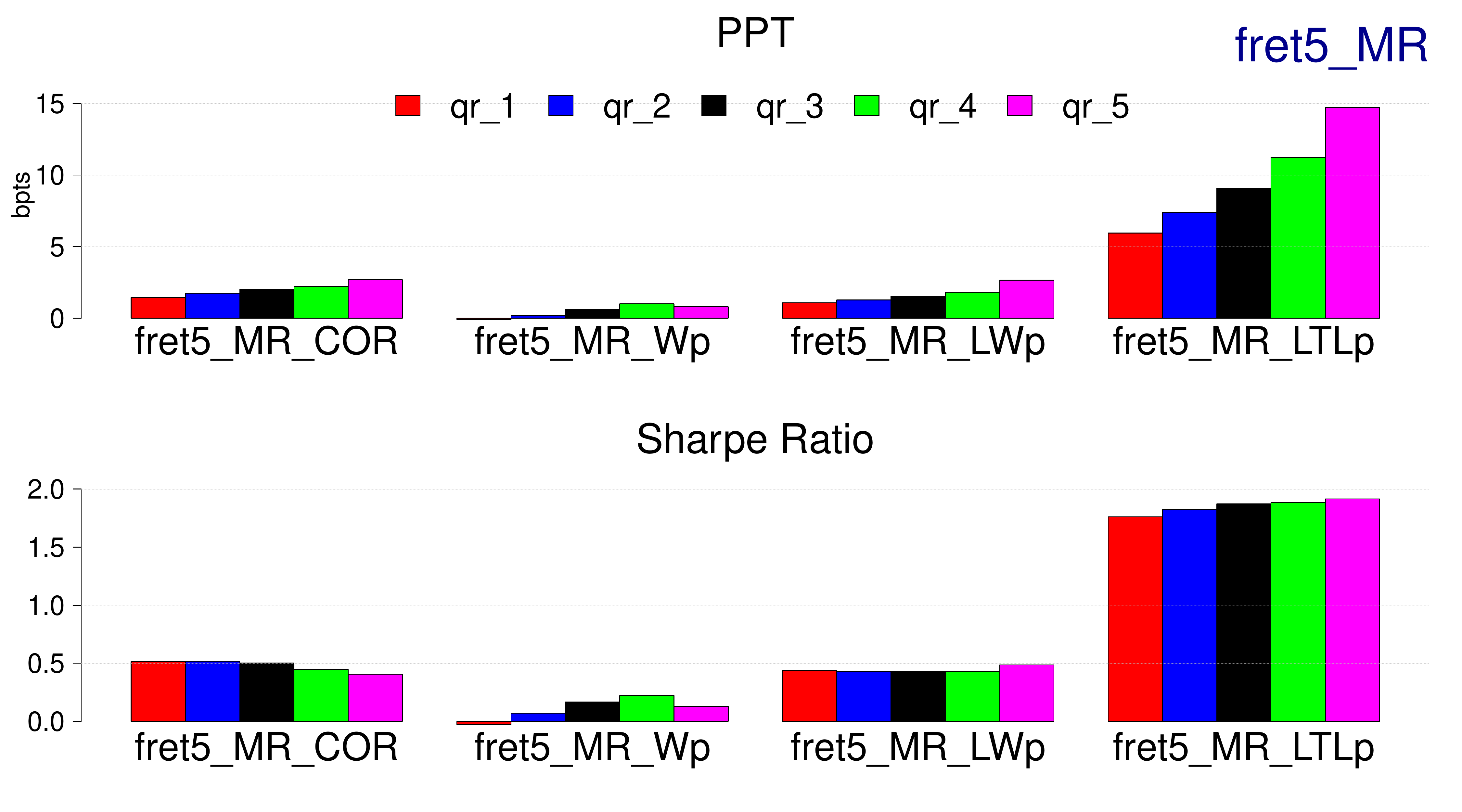}}  \\
\hline 
\rotatebox{90}{10-day}   
& \includegraphics[width=\widd]{{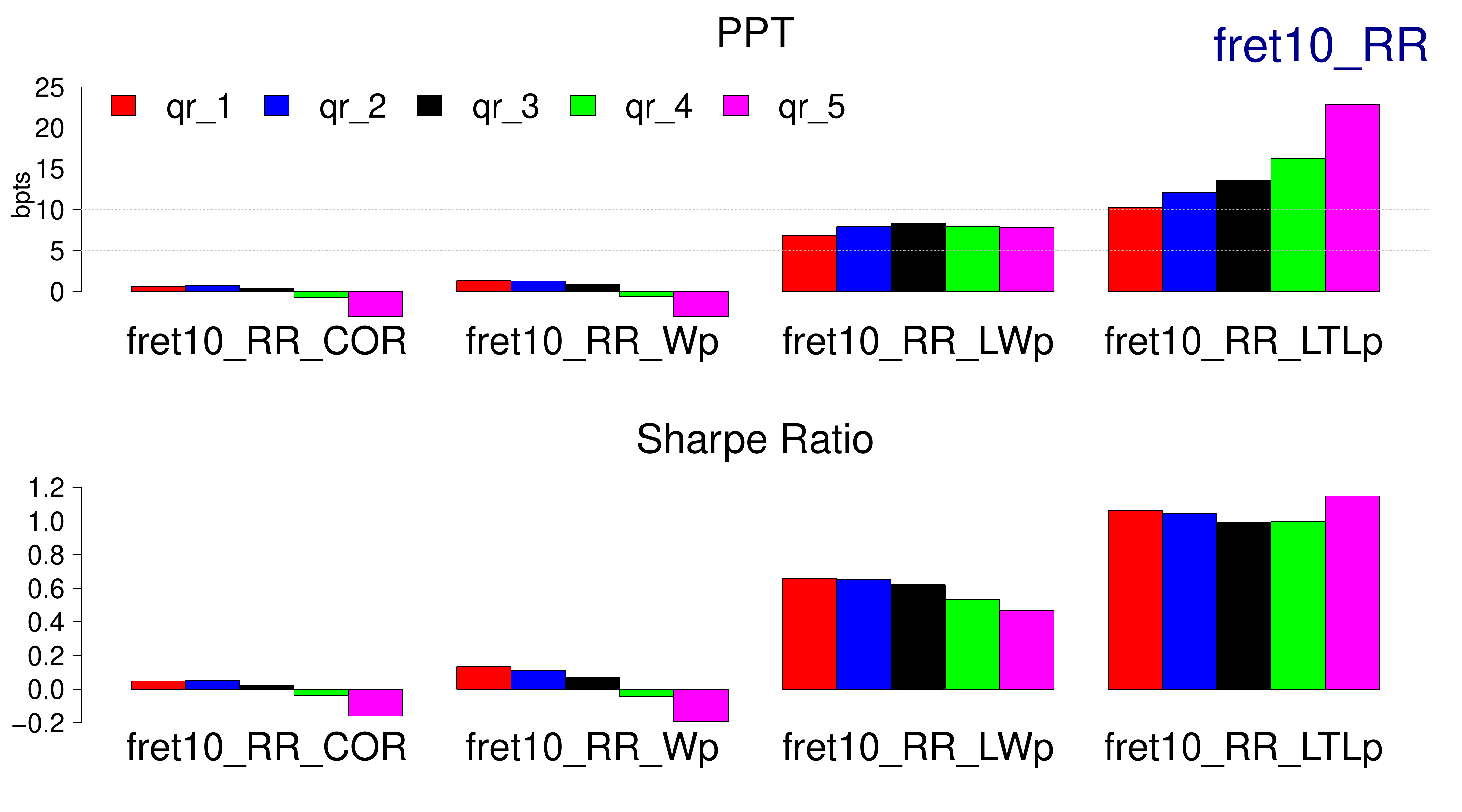}} 
& \includegraphics[width=\widd]{{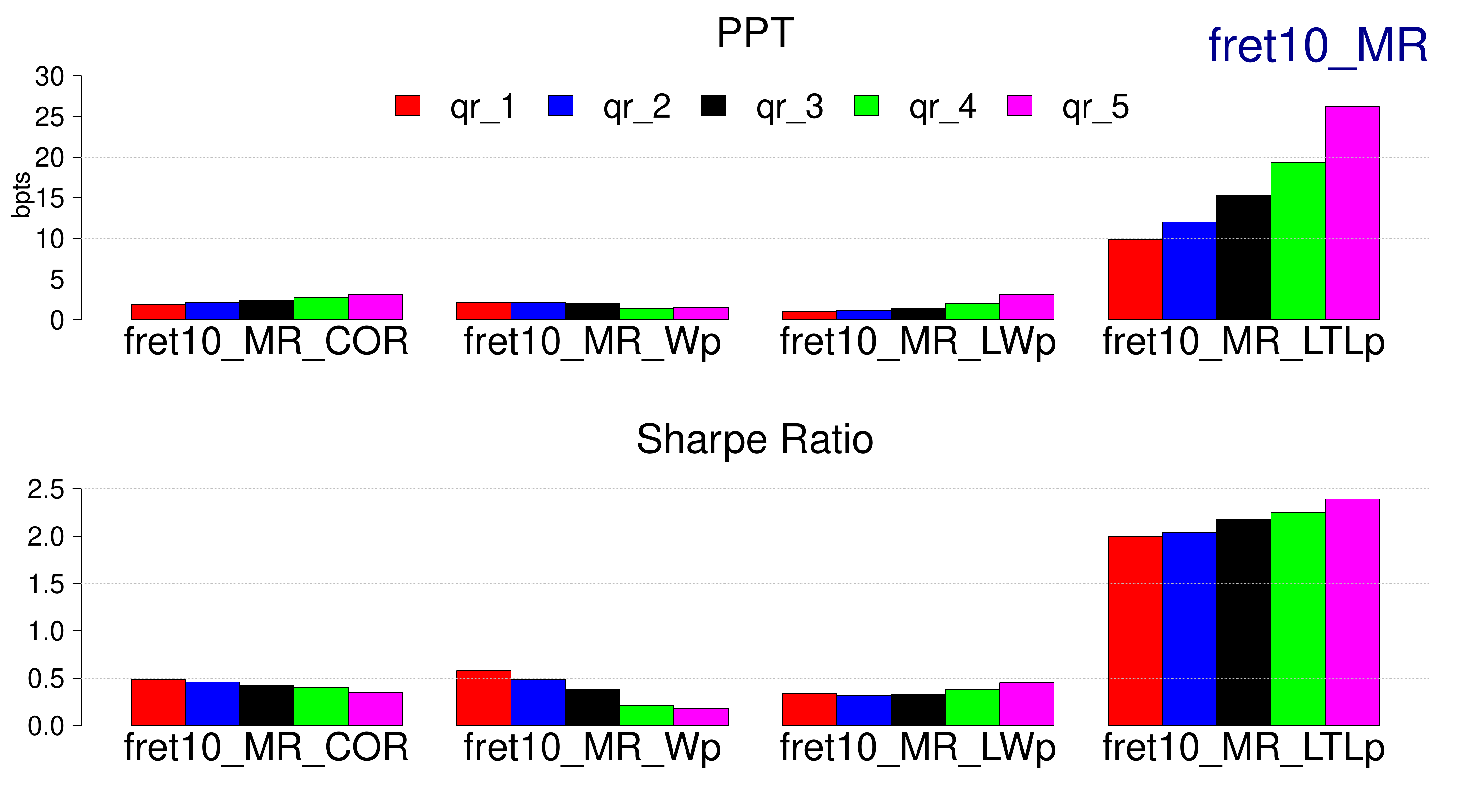}}  \\
\hline 
\end{tabular}
\end{center}
\captionsetup{width=0.99\linewidth}
\caption{Portfolio statistics  
across various future horizons $ h \in \{1,5,10\} $   
across two types of returns: raw returns and market-excess returns. 
The colors denote quintile portfolios, and the $x$-axis denotes the forecasts made by each method \{COR, $\Wp$, $\LWp$, $\LTLp$\}. 
}   
\label{fig:PS_knn_100_MaxHist_Inf_BARS}	
\end{figure}

\begin{figure}
\newcommand{\widd}{0.45\textwidth}
\newcolumntype{D}{>{\centering\arraybackslash}m{\widd}}
\begin{center}
\begin{tabular}{l*2{D}@{ }}
    &  Raw Returns (RR) & Market-Excess Returns (MR)  \\ 
\hline 
\rotatebox{90}{1-day}  
& \includegraphics[width=\widd]{{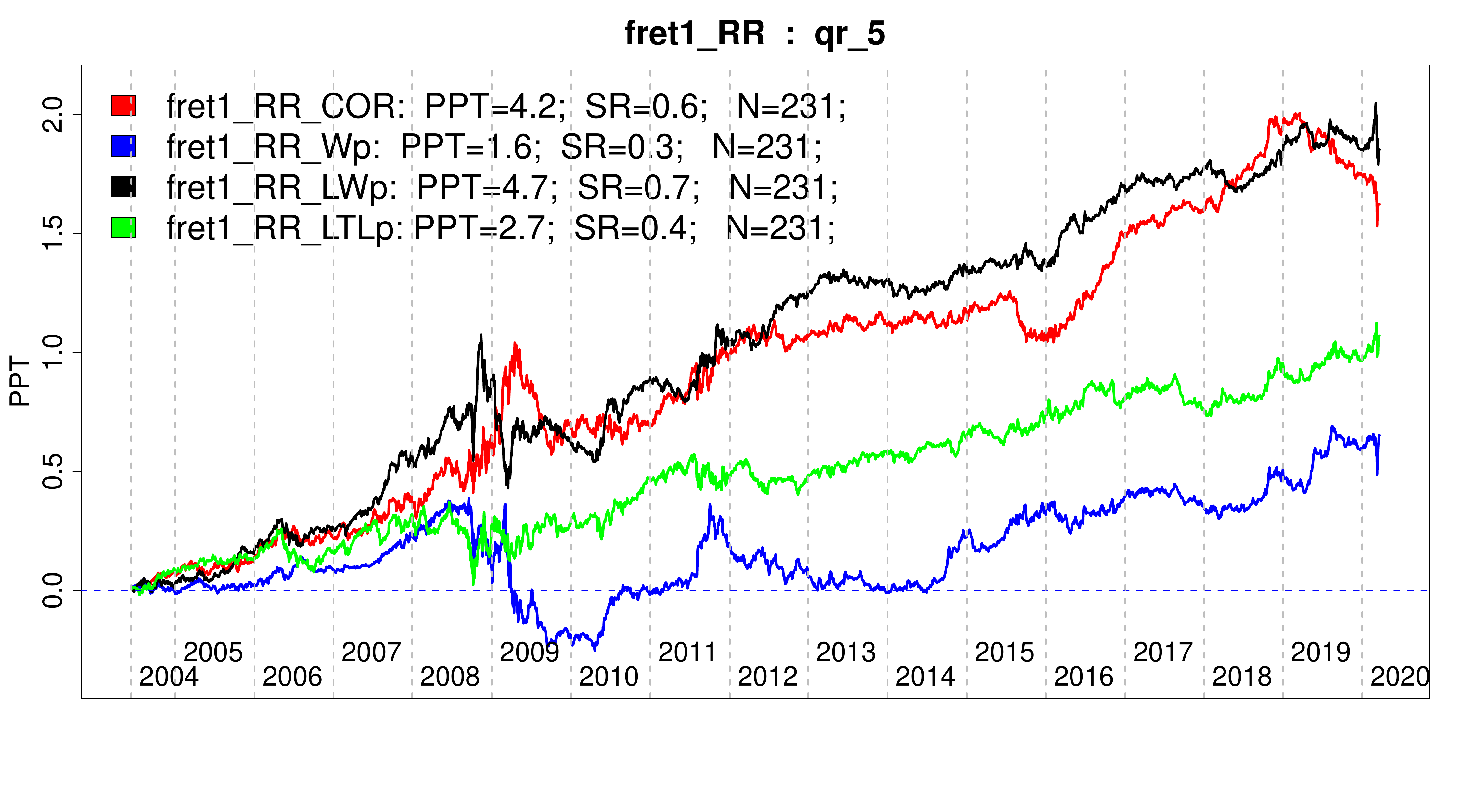}} 
& \includegraphics[width=\widd]{{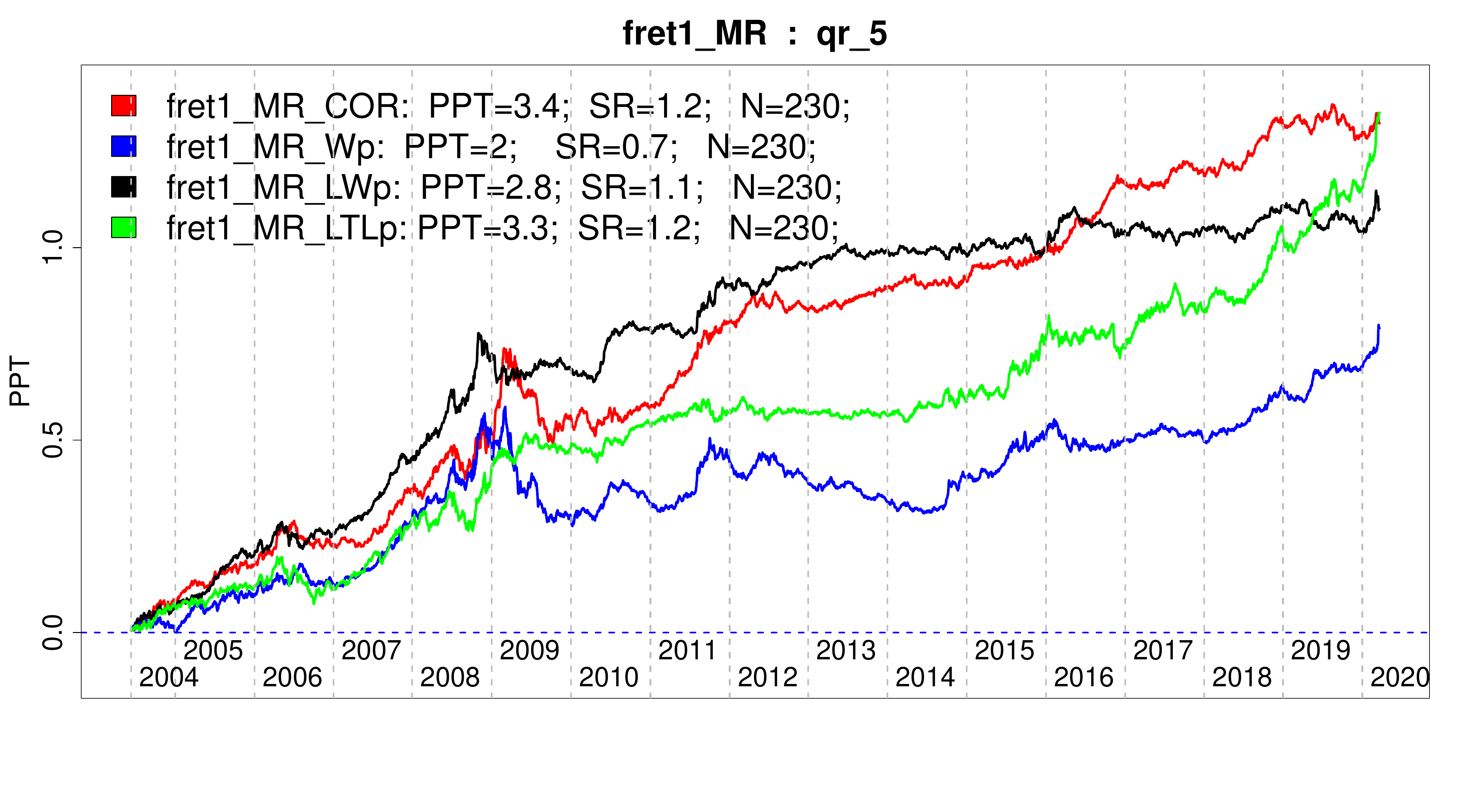}}  \\
\hline 
\rotatebox{90}{5-day}  
& \includegraphics[width=\widd]{{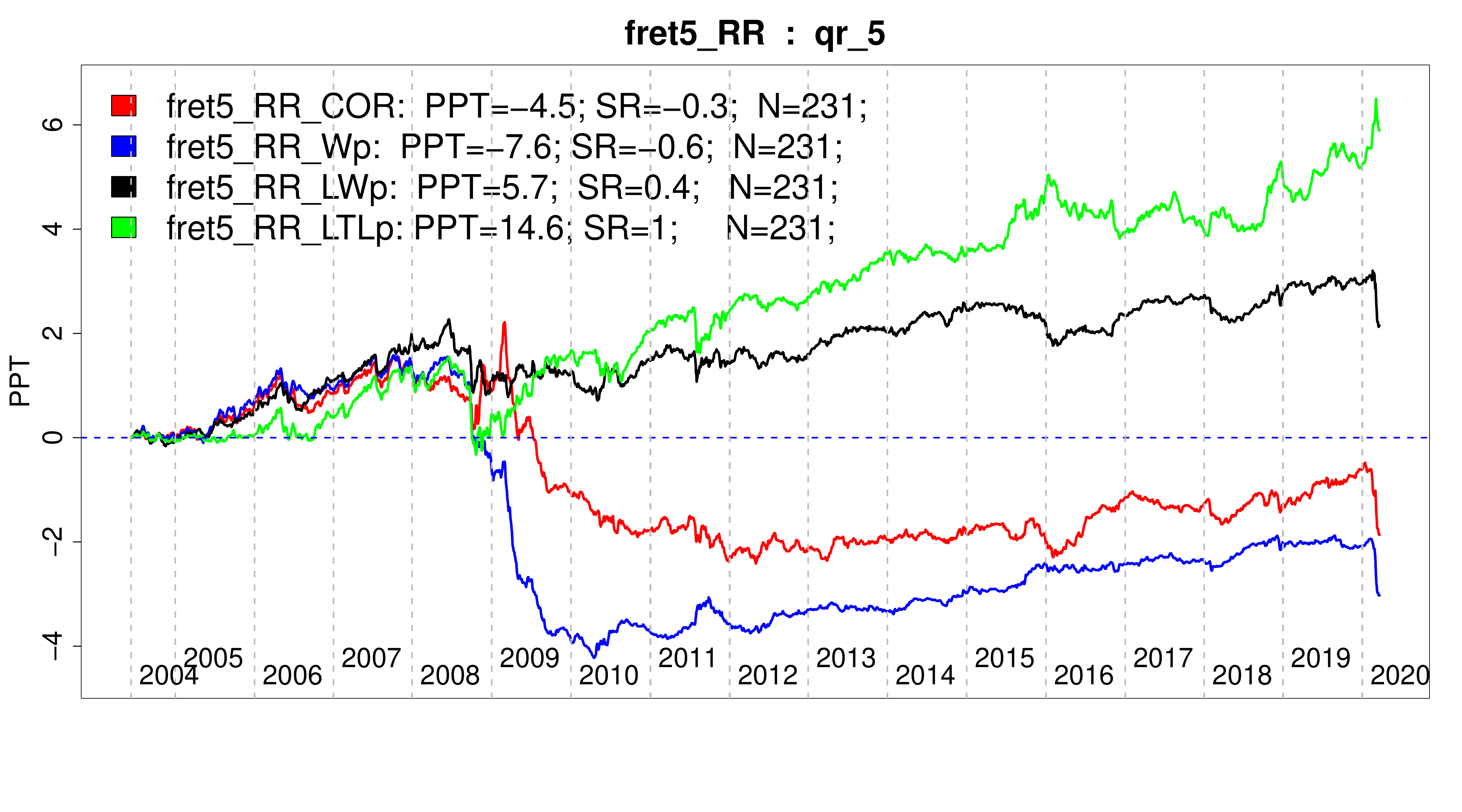}} 
& \includegraphics[width=\widd]{{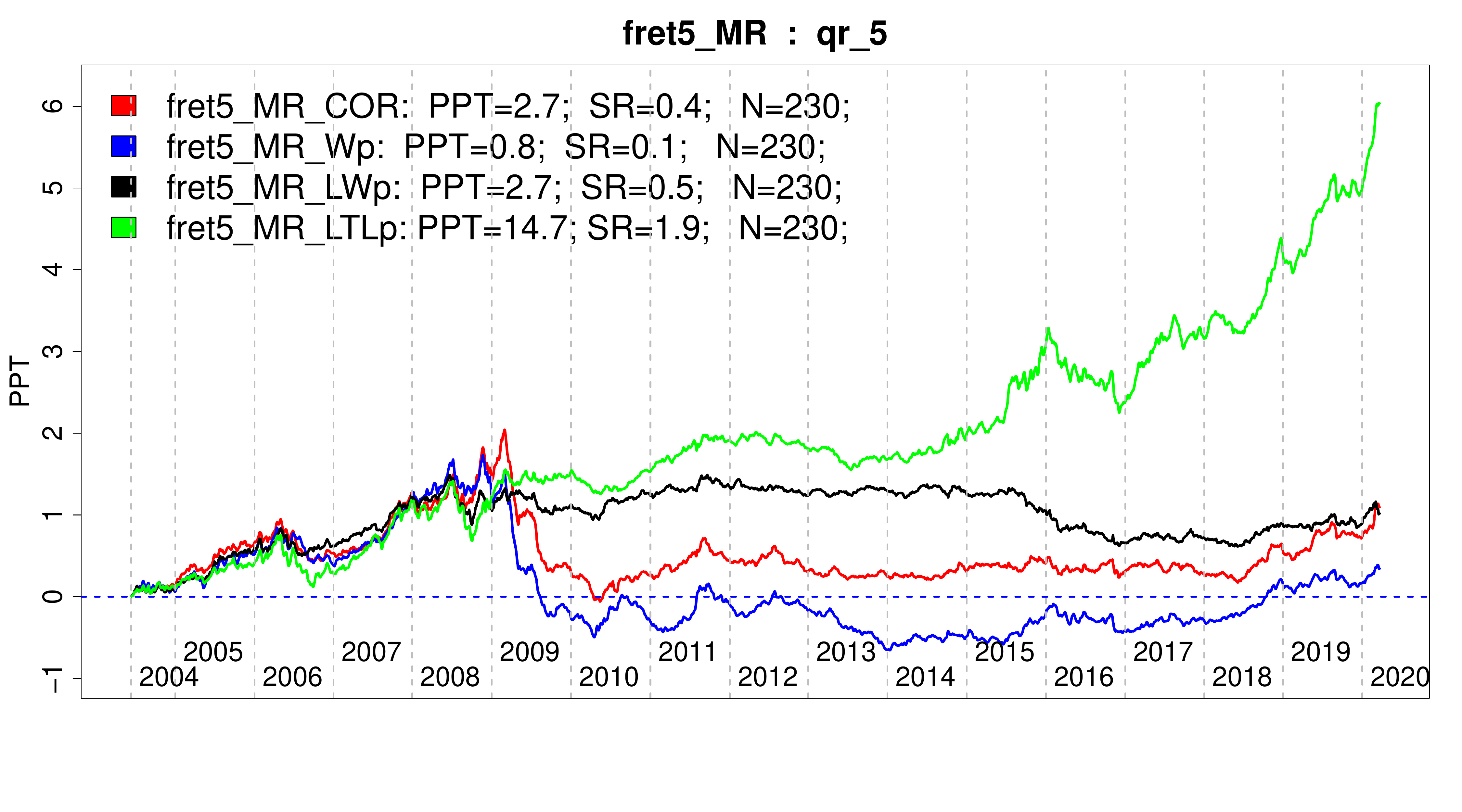}}  \\
\hline 
\rotatebox{90}{10-day}   
& \includegraphics[width=\widd]{{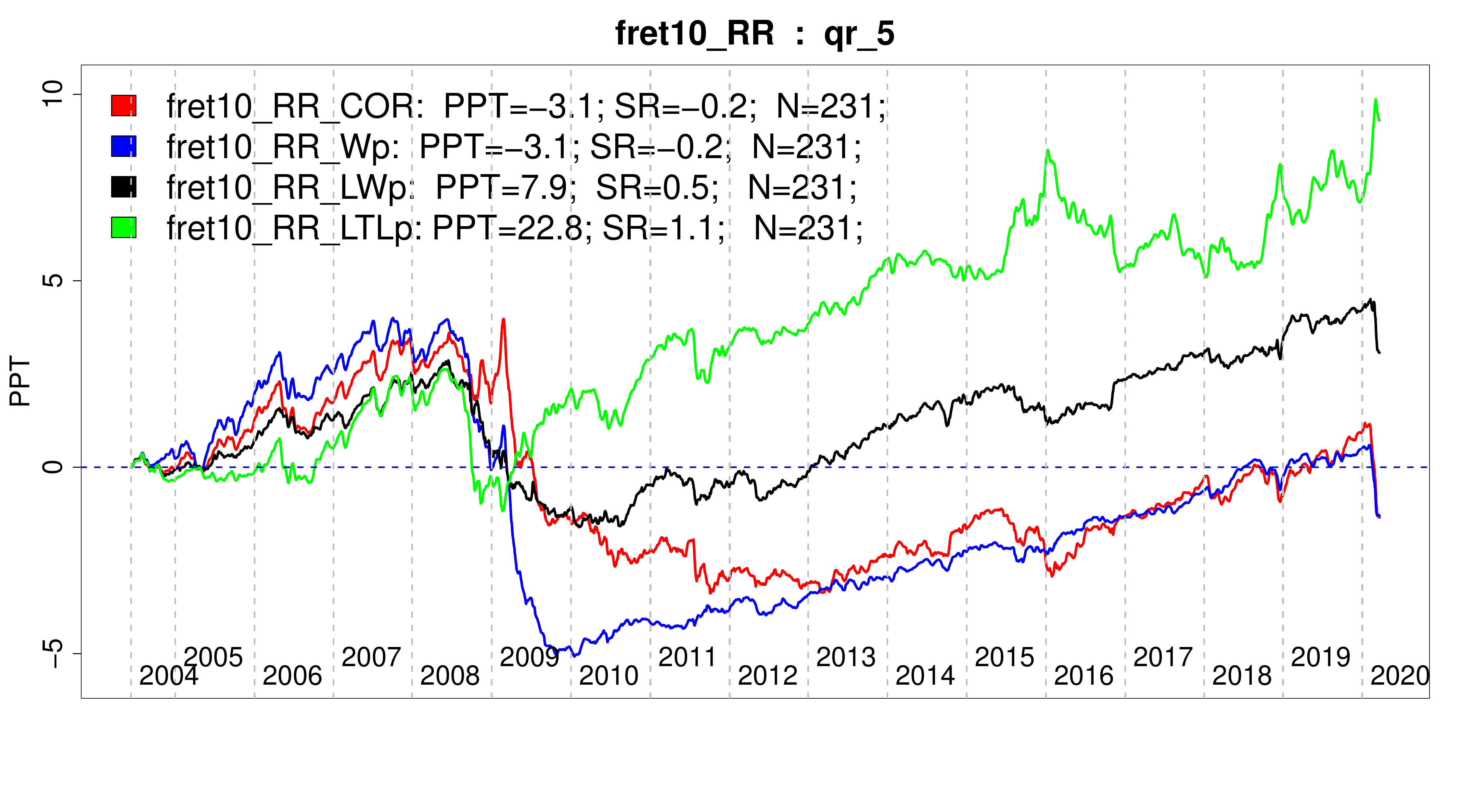}} 
& \includegraphics[width=\widd]{{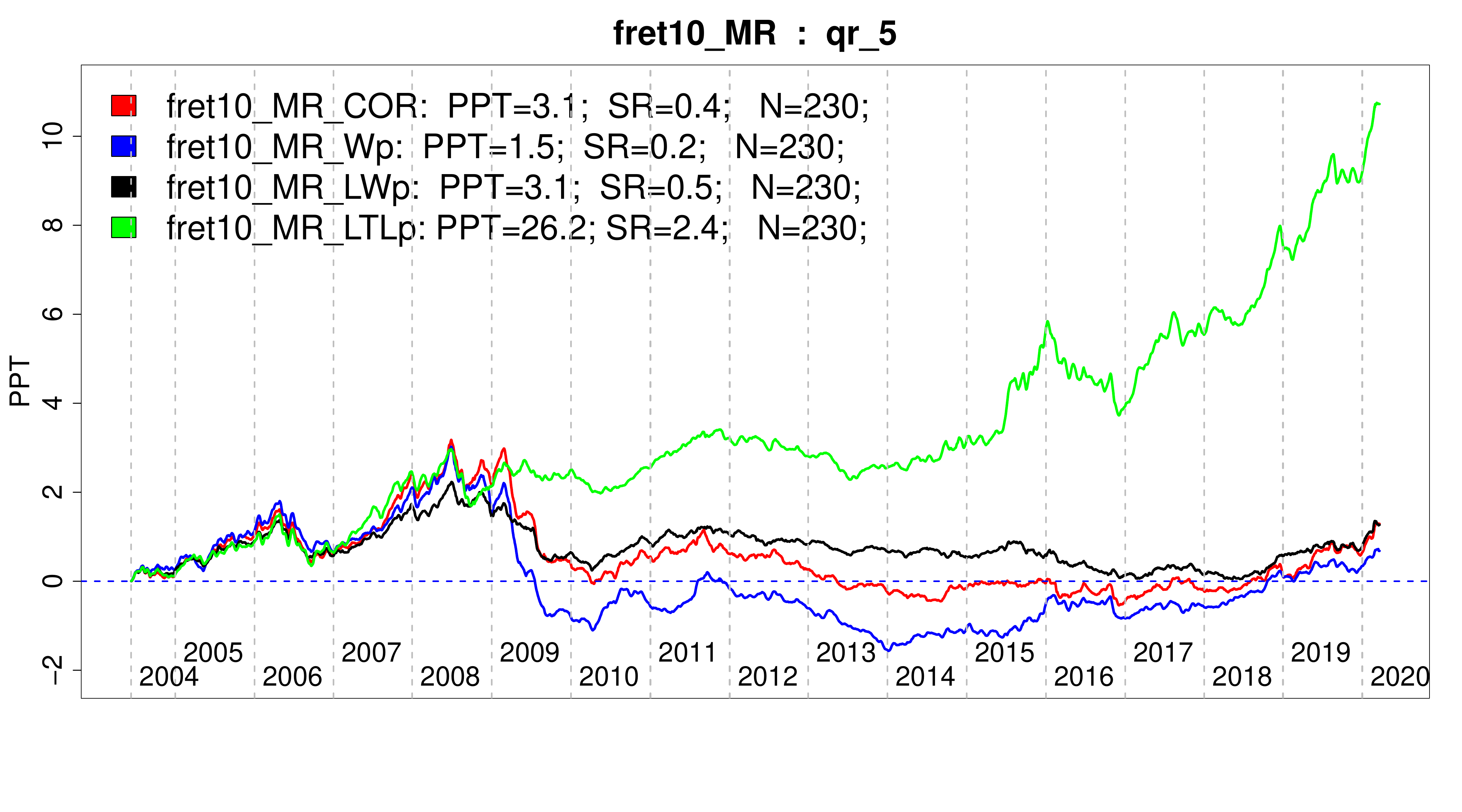}} \\
\hline 
\end{tabular}
\end{center}
\captionsetup{width=0.99\linewidth}
\caption{Cumulative PnL of the top quintile portfolio $qr_{5}$ for the  1,5,10-day future horizons, for raw returns and market-excess returns, across different methods. The legend contains performance statistics of each forecast method: PnL Per Trade (PPT) in basis points, Sharpe Ratio (SR), while $N$ denotes the size of the portfolio (since $qr_{5}$ contains the top 20\% largest magnitude forecasts, this amount to $N \approx 1150/5 =  230$ stocks.)
}
\label{fig:PS_knn_100_MaxHist_Inf_EvoQR5}	
\end{figure}

\section{Discussion}

We have presented a linear optimal transport distance for use in pattern recognition tasks based on the $\TLp$ distance. The $\TLp$ distance generalises the Wasserstein distance such that it can be applied to multi-channelled data, for example colour images and multivariate time-series. However, when pairwise distances are needed the computation of the $\TLp$ distance can make its routine application in pattern recognition tasks computationally infeasible. We proposed a method to alleviate this problem by extending the $\LWp$ framework to the $\TLp$ setting and call this the $\LTLp$ framework. In a dataset of $N$ signals, these linear transportation approaches need to only compute $N$ transport maps/distances and thus the cost scales linearly in the number of signals. In contrast, if pairwise distance are needed then the cost scales quadratically with the number of signals, which can render the problem infeasible.
\vspace{\baselineskip}

We developed the theory required to allow these transportation methods to be applied. Showing the preservation of these transportation distances with respect to a reference signal, as well as showing that the $\LTLp$ defines a bona fide metric on the $\TLp$ space. We additionally showed that the linear transportation methods facilitates a linear embedding of our signals allowing simple statistical methods to be applied to the data.
\vspace{\baselineskip}


We compared our methods on Auslan, Breast Cancer Histopathology and Financial times series problems. 
In each case, the $\LTLp$ approach significantly outperformed the $\LWp$ approach. 
Furthermore, this came at minimal increased computational cost.
Even though $\TLp$ often outperformed $\LTLp$ on this task this improved performance comes at an unreasonable additional computational cost.
\vspace{\baselineskip}

Our method still relies on the computation of transport maps and this comes at a cost. We have found entropy regularised methods to perform well but suffer from instability. Fast and stable algorithms which are memory efficient are still needed by the community for reliable computation of transport maps.
\vspace{\baselineskip}

Extensions of the $\TLp$ framework could be to the manifold setting \cite{Feldman:2002}, which would allow computation of the $\TLp$ distance on a graph. The $\TLp$ barycentre and other transformations have also yet to be explored. Considering $\TLp$ has a spatial penalty it could also be combined with other penalty terms for more complex applications, for example one could also include a penalty on derivatives of signals as in the $\TWkp$ distance, see~\cite{Thorpe:2017}. 
\vspace{\baselineskip}


\section*{Acknowledgements}

This work was supported by The Alan Turing Institute under EPSRC grant EP/N510129/1.
In addition the authors are grateful for discussions with Elizabeth Soilleux, whose interest in machine learning methods for diagnosing coeliac disease motivated this work, and Dejan Slep\v{c}ev.
OMC is a Wellcome Trust Mathematical Genomics and Medicine student and is grateful for generous funding from the Cambridge school of clinical medicine.
MC acknowledges support from the EPSRC grant EP/N510129/1 at The Alan Turing Institute.
TDH was supported by The Maxwell Institute Graduate School in Analysis and its Applications, a Centre for Doctoral Training funded by the EPSRC (grant EP/L016508/01), the SFC, Heriot-Watt University and the University of Edinburgh. 
CBS acknowledges support from the Leverhulme Trust project on `Breaking the non-convexity barrier', the Philip Leverhulme Prize, the Royal Society Wolfson Fellowship, the EPSRC grants EP/S026045/1 and EP/T003553/1, the EPSRC Centre Nr. EP/N014588/1, the Wellcome Innovator Award RG98755, European Union Horizon 2020 research and innovation programmes under the Marie Skodowska-Curie grant agreement No. 777826 (NoMADS) and No. 691070 (CHiPS), the Cantab Capital Institute for the Mathematics of Information and the Alan Turing Institute.
MT is grateful for the support of the Cantab Capital Institute for the Mathematics of Information (CCIMI) and Cambridge Image Analysis (CIA) groups at the University of Cambridge, and is supported by the European Research Council under the European Union's Horizon 2020 research and innovation programme grant agreement No 777826 (NoMADS) and grant agreement No 647812.
KCZ was supported by the Alan Turing Institute under the EPSRC grant EP/N510129/1.
\vspace{2\baselineskip}

\bibliographystyle{plain}
\bibliography{optimaltransport}

\appendix
\appendixpage

In the appendix we include some additional background on computing the $\TLp$ distance and further applications.

\section{Computing the \texorpdfstring{$\LTLp$}{LTLp} Embedding} \label{app::section::ComputingTLp}

In this section we review some methods for computing the $\LTLp$ embedding.
In principle, any algorithm that can compute optimal transport distances can be adapted to compute $\TLp$ by either interpreting $\TLp$ as a Wasserstein distance on the graphs of functions, or as an optimal transport problem with cost $c(x,y;f,g) = |x-y|_p^p + |f(x)-g(y)|_p^p$.
We refer to~\cite{peyre18} for a thorough review of computational methods for optimal transport.
Here, we review entropy regularised optimal transport and flow minimisation in the setting of $\TLp$.
We note that as the Kantorovich problem is a linear program then one can use algorithms such as the \emph{simplex} or \emph{interior-point} methods.
Although there are multi-scale approaches~\cite{Oberman:2015} these are typically not state-of-the-art for high dimensional images/signals and so we omit them from this review.

Once we have obtained optimal $\TLp$ maps $\tilde{T}^{\tilde{\mu}_i}:\Omega\times \bbR^m\to\Omega\times \bbR^m$ for each of the transportation problems between $\tilde{\sigma}\in \cP(\Omega\times\bbR^m)$ and $\tilde{\mu}_i\in\cP(\Omega\times\bbR^m)$ for $i = 1,..., N$ we can embed into Euclidean space by~(\ref{eq:Methods:LTLpPd}-\ref{eq:Methods:LTLptildePd}) where $\tilde{T}^{\tilde{\mu}_i} = (T^{\mu_i},f_i\circ T^{\mu_i})$. 
Hence, linear statistical methods can be applied. 



\subsection{An Entropy Regularisation Approach} \label{subsubsec:Methods:ComputingTLp:Ent}

We assume two pairs $(\mu,f),(\nu,g)\in\TLp$ can be written in the form
\[ \mu = \sum_{i=1}^m p_i \delta_{x_i}, \quad \nu = \sum_{j=1}^n q_j \delta_{y_j}, \quad f_i=f(x_i) \quad \text{and} \quad g_j = g(y_j). \]
It was proposed in~\cite{Cuturi:2013} to consider the entropy regularised problem
\begin{equation} \label{eq:Methods:EntReg}
S_\eps((\mu,f),(\nu,g)) = \min_{\pi} \l \sum_{i=1}^m \sum_{j=1}^n \l |x_i - y_j|_p^p + |f_i - g_j| \r \pi_{ij} - \eps H(\pi) \r
\end{equation}
where $\eps>0$ is a positive parameter that controls the amount of regularisation, $H$ is entropy and defined by
\[ H(\pi) = - \sum_{i=1}^n \sum_{j=1}^m \pi_{ij} \log \pi_{ij} \]
and the minimum in~\eqref{eq:Methods:EntReg} is taken over matrices $\pi\in \bbR^{n\times m}_+$ such that the row sums are $\boldsymbol{p} = (p_1,\dots, p_m)$ and the column sums are $\boldsymbol{q} = (q_1,\dots, q_n)$.
When $\eps\to 0$, the results of~\cite{carlier17} imply that $S_\eps((\mu,f),(\nu,g))\to \dTLp^p((\mu,f),(\nu,g))$.
Subsequent developments of the entropy regularised approach have appeared in~\cite{Cuturi:2014,Benamou:2015}. 
The measure $S_{\varepsilon}$ is referred to as the Sinkhorn distance.
It is easy to see that
\[ S_{\varepsilon}((\mu,f),(\nu,g)) = \varepsilon \inf_{\pi} \{\KL(\pi|K)\}, \]
where $K_{ij} = \exp\left(- \frac{C_{ij}}{\varepsilon}\right)$ is the Gibbs distribution, $C_{ij} = |x_i-y_j|^p_p + |f_i-g_j|_p^p$, and $\KL$ denotes the \emph{Kullback-Leibler} divergence.
The minimisation is taken over the same set as in~\eqref{eq:Methods:EntReg}
The optimal choice for $\pi$ can be written in the following form:
\[ \pi^{\dagger} = \diag(u)K\diag(v), \]
where $u,v$ are the limits, as $r \to \infty$, of the sequence
\[ v^{(0)} = \boldsymbol{1},\,\,u^{(r)} = \frac{\boldsymbol{p}}{Kv^{(r)}},\,\,v^{(r+1)} = \frac{\boldsymbol{q}}{K^{T}u^{(r)}}, \]
see~\cite{Benamou:2015}.
The entropy regularisation means the optimal $\pi$ for $S_\eps$ cannot be written as a transport map.
To obtain an approximation to the optimal transport map one can use Barycentric projections as in~\cite[Section 2.3]{Oberman:2015}.

\subsection{A Flow Minimisation Approach}

Following~\cite{Haker:2004} we derive a flow minimization method for finding the transportation map in $\TLtwo$. Let $\Omega \subset \bbR^d$ be a compact domain with smooth boundary and let $(\mu, f)$ and $(\nu, g)$ be signals in $\TLtwo(\Omega, \mathbb{R}^m)$, where $f,g: \Omega \to \mathbb{R}^m$ are square-integrable functions. Furthermore, we assume that the measures $\mu$ and $\nu$ admit densities with respect to the Lebesgue measure. Abusing notation we write $\dd\mu(x) = \mu(x)\dd x$. The variational $\TLtwo$ problem is finding the diffeomorphic map $T:\Omega \to \Omega$, which minimises the following energy
\begin{align}
\eps(T) & = \int_{\Omega} \l |T(x) - x|^2_2 + | g(T(x)) - f(x)|^2_2 \r \mu(x)\,\dd x \label{eq:Methods:epsT} \\
&\text{subject to }T_*\mu = \nu. \label{eq:Methods:epsTCons}
\end{align}
We assume the following polar factorization of $T$.
Let $s:\Omega \times [0,\infty) \to \Omega$  and assume the second coordinate is time. We further assume for any fixed $t$, $[s(\cdot,t)]_*\mu = \mu$.
That is, $s(\cdot,t):\Omega\to \Omega$ is a mass preserving rearrangement of $\mu$.
Let $T^0:\Omega\to\Omega$ be an initial mass preserving map between $\mu$ and $\nu$, i.e. $T^0_* \mu = \nu$, for example the Knothe-Rosenblatt coupling \cite{Villani:2008}.
We assume that $s(\cdot,t)$ is invertible in $x$ for every $t$ and with an abuse of notation we write $s^{-1}$ for this inverse, i.e.
\begin{equation} \label{eq:Methods:sInv}
s^{-1}(s(x,t),t) = x = s(s^{-1}(x,t),t) \quad \text{for all } x\in \Omega \text{ and for all } t\in [0,\infty).
\end{equation}
We require that $T = T^0 \circ s^{-1}$.
The strategy in~\cite{Haker:2004} is to evolve $s(\cdot,t)$ using a gradient descent step such that it converges to a minimiser of~\eqref{eq:Methods:epsT} satisfying the constraint~\eqref{eq:Methods:epsTCons} as $t\to \infty$. 
We first consider sufficient conditions on $s$ in order to guarantee that~\eqref{eq:Methods:epsTCons} holds for all $t>0$.
The proof of the proposition can be found in~\cite[Section A.2]{Haker:2004}.

\begin{proposition}
\label{prop:Methods:sEvol}
Let $\Omega\subset\bbR^d$ be a compact domain with a smooth boundary and $\mu,\nu\in\cP(\Omega)$.
Assume that $\mu$ and $\nu$ have $C^1$ densities with respect to the Lebesgue measure on $\Omega$ and with an abuse of notation write $\dd\mu(x) = \mu(x) \, \dd x$ and $\dd\nu(x) = \nu(x) \, \dd x$
Let $\chi$ be a $C^1$ vector field on $\Omega$ satisfying $\Div(\chi) = 0$ on $\Omega$ and $\chi\cdot \n = 0$ on $\partial \Omega$ where $\n$ is the normal to the boundary of $\Omega$.
Assume $s:\Omega\times [0,\infty)\to \Omega$ is differentiable and invertible in the sense of~\eqref{eq:Methods:sInv}, and $T^0$ satisfies $T^0_*\mu = \nu$.
If, $s(\cdot,0) = \Id$ and for all $t>0$
\begin{equation} \label{eq:Methods:sEvol}
\frac{\partial s}{\partial t}(x,t) = \frac{1}{\mu(s(x,t))} \chi(s(x,t))
\end{equation}
then $[s(\cdot,t)]_*\mu = \mu$.
Furthermore $\frac{\partial T^t}{\partial t} = - \frac{1}{\mu}\nabla T^t\chi$ and $T^t_*\mu = \nu$ where
\begin{equation} \label{eq:Methods:Tt}
T^t = T^0(s^{-1}(\cdot,t)).
\end{equation}
\end{proposition}

If we restrict ourselves to look for transport maps of the form~(\ref{eq:Methods:sEvol}-\ref{eq:Methods:Tt}) then we must decide how to choose $\chi$.
Let us define $s_\chi:\Omega\times [0,\infty)\to \Omega$ by~\eqref{eq:Methods:sEvol} with $s_\chi(\cdot,0) = \Id$ and $T^t_\chi:\Omega\to \Omega$ by~\eqref{eq:Methods:Tt} with $s = s_\chi$.
An obvious criterion is to choose $\chi$ so that $\eps(T^t_\chi)$ decreases quickest over all choices of $\chi$.
To this end we compute the derivative of $\eps(T^t_\chi)$ with respect to $t$.

\begin{lemma}
In addition to the assumptions and notation of Proposition~\ref{prop:Methods:sEvol} let $f\in C^1(\Omega;\bbR^m)$ and $g\in L^2(\nu)$ and define $\eps$ by~\eqref{eq:Methods:epsT}.
Define $s_\chi:\Omega\times [0,\infty)\to \Omega$ by~\eqref{eq:Methods:sEvol} with $s_\chi(\cdot,0) = \Id$ and $T^t_\chi:\Omega\to \Omega$ by~\eqref{eq:Methods:Tt} with $s = s_\chi$.
Then we have
\[ \frac{\dd}{\dd t} \eps(T^t_\chi) = - \int_\Omega Q(x,t) \cdot \chi(x)  \, \dd x \] 
where
\begin{equation} \label{eq:Methods:Q}
Q(t,x) = 2T^t_\chi(x) + 2\sum_{i=1}^m g_i(T_\chi^t(x)) \nabla f_i(x).
\end{equation}
\end{lemma}

\begin{proof}
We define
\[ \tilde{\eps}(T;f,g) = \int_\Omega |g(T(x)) - f(x)|_2^2 \, \dd \mu(x) \]
which we can also write as
\[ \tilde{\eps}(T;f,g) = \int_\Omega |g(y)|_2^2 \, \dd \nu(y) + \int_\Omega |f(x)|_2^2 \, \dd \mu(x) - 2\int_\Omega g(T(x)) \cdot f(x) \, \dd \mu(x). \]
By a change of variables $y = s_\chi^{-1}(x,t)$, and since $[s_\chi(\cdot,t)]_*\mu = \mu$ we have
\begin{align*}
\tilde{\eps}(T_\chi^t;f,g) & =  \int_\Omega |g(y)|_2^2 \, \dd \nu(y) + \int_\Omega |f(x)|_2^2 \, \dd \mu(x) - 2 \int_\Omega g(T^0(s_\chi^{-1}(x,t))) \cdot f(x) \, \dd \mu(x) \\
 & =  \int_\Omega |g(y)|_2^2 \, \dd \nu(y) + \int_\Omega |f(x)|_2^2 \, \dd \mu(x) - 2 \int_\Omega g(T^0(y)) \cdot f(s_\chi(y,t)) \, \dd \mu(y).
\end{align*}
Differentiating the above we obtain,
\begin{align*}
\frac{\dd}{\dd t} \tilde{\eps}(T_\chi^t;f,g) & = -2 \sum_{i=1}^m \int_\Omega g_i(T^0(y)) \nabla f_i(s_\chi(y,t)) \cdot \frac{\partial s_\chi}{\partial t}(y,t) \, \dd \mu(y) \\
 & = -2 \sum_{i=1}^m \int_\Omega g_i(T_\chi^t(x)) \nabla f_i(x) \cdot \chi(x) \, \dd x.
\end{align*}
We note that $\eps(T_\chi^t) = \tilde{\eps}(T_\chi^t;\Id,\Id) + \tilde{\eps}(T^t_\chi;f,g)$ hence $\frac{\dd}{\dd t} \eps(T^t_\chi) = - \int_\Omega Q(x,t) \cdot \chi(x)  \, \dd x$. 
\end{proof}

When $d=2$ by the Helmholtz decomposition (in 2D) we can find, for each $t>0$, two scalar fields $w:\Omega\to \bbR$ and $\alpha:\Omega\to \bbR$ such that $Q(\cdot,t) = \nabla w + \nabla^\perp \alpha$ (where the $t$ dependence on $\alpha$ and $w$ is supressed) and $\alpha = 0$ on $\partial \Omega$ where $\nabla^\perp f = \l-\frac{\partial f}{\partial x_2},\frac{\partial f}{\partial x_1}\r$ for a function $f(x)=f(x_1,x_2)$. 
To find the direction of steepest descent we let $\psi = \nabla^\perp \alpha$ and $\chi = \nabla^\perp \beta$ and compute
\begin{align}
\frac{\dd }{\dd t} \eps(T^t_\chi) & = -\int_\Omega \l \nabla w(x) + \psi(x) \r \cdot \chi(x) \, \dd x \notag \\
 & = - \int_\Omega \l \Div(w \chi)(x) - w(x)\Div(\chi)(x) \r\, \dd x - \int_\Omega \psi(x) \cdot \chi(x) \, \dd x \notag \\
 & = - \int_{\partial\Omega} w(x) \chi(x) \cdot \n(x) \, \dd S(x) - \int_\Omega \psi(x) \cdot \chi(x) \, \dd x \notag \\
 & = - \int_\Omega \psi(x) \cdot \chi(x) \, \dd x \notag \\
 & = - \int_\Omega \nabla \alpha(x) \cdot \nabla \beta(x) \, \dd x \label{eq:Methods:depsdt}
\end{align}
where the third line follows from the divergence theorem and since $\Div(\chi) = 0$ on $\Omega$, and the fourth line follows from $\chi(x) \cdot \n(x) = 0$ on $\partial \Omega$.
It follows that the direction of steepest descent is $\alpha = \beta$.

To find $\alpha$, we need to observe that $\nabla \alpha = -Q^\perp -\nabla^\perp w$ where $\perp$ is rotation clockwise by $\pi/2$, i.e. $Q^\perp  = (-Q_2,Q_1)$.
Taking the divergence we have
\[ \Delta \alpha = \Div(\nabla \alpha) = \Div(- Q^\perp - \nabla^\perp w) = - \Div (Q^\perp). \]
Hence, $\alpha$ solves the Poisson equation with Dirichlet boundary conditions:
\begin{align}
\Delta \alpha & = - \Div(Q^{\perp}) & \text{in } \Omega \label{eq:Methods:alphaInt} \\
\alpha & = 0 & \text{on } \partial\Omega. \label{eq:Methods:alphaBound}
\end{align}

To summarise, the flow minimization scheme for $\TLp$, given a step size $\tau$ is as follows.
\begin{enumerate}
\item Construct $T^0$ and set $t=0$.
\item Compute $Q(\cdot,t)$ defined by~\eqref{eq:Methods:Q}.
\item Find $\alpha$ by solving~(\ref{eq:Methods:alphaInt}-\ref{eq:Methods:alphaBound}). 
\item Update $T^{t+\tau} = T^t - \frac{\tau}{\mu} \nabla T^t \nabla^\perp \alpha$.
\item Set $t \mapsto t+\tau$.
\item Repeat 2-5 until convergence.
\end{enumerate}

\section{Additional Experiments}

To supplement the examples given in the main body of the paper we include three other applications.
The first two additional examples are to synthetic 1D and 2D data; the same examples were given in~\cite{Thorpe:2017}.
The final example is an application to cell morphometry which appeared in~\cite{Basu:2014} as an example of the LOT framework.

We repeat Table~\ref{table::timings} with the complete set of experiments.

\begin{table}[ht!]
	\begin{center}
		\begin{tabular}{ |c|c|c|c| } 
			\hline
			Application & $\LWp$ & $\LTLp$ & $\TLp$\\
			\hline
			\hline
			1-D Synthetic & 1.3 & 0.2 & 90.2\\ 
			\hline
			2-D Synthetic & 7.1 & 16.9 & 707.2\\  
			\hline
			Cell Morphometry & 242.1 & 512.7 & 161080\\
			\hline 
			Auslan & 12.1 & 13.0 & 91200\\
			\hline
			Breast Cancer Histopathology & 25407.0 & 2919.8 & >345600\\
			\hline
			Financial Time Series & 39.5 &  192.3 & - \\ 
			\hline
	    \end{tabular} 
\end{center} 
\caption{CPU times in seconds to compute each transportation method on each dataset. Computation was halted after 4 CPU days. 
}
\label{table::timings2} 
\end{table}


\subsection{One Dimensional Synthetic Signal Processing} \label{subsec:Results:Syn1d}

We first consider a one dimensional signal processing problem, to test the ability of $\Wp$ and $\TLp$ to discriminate between different signals. Throughout, we take $p = 2$.  We consider the task of discriminating between double hump and a high-frequency perturbation of the hump function: a chirp function. A double hump function is of the form:
\begin{equation}\label{equation:hump}
f = K_1 \cdot (\mathds{1}_{[l,l + r]} + \mathds{1}_{[l + b + r, l + b + 2r]}),
\end{equation} 
where $\mathds{1}_{[\alpha, \beta]}$ denotes the indicator function on the interval $[\alpha, \beta]$ and $l \in [0, 1 - b - 2r]$. The constant $K_1$ is chosen such that $f$ integrates to unity. A chirp-hump function is given as:
\begin{equation}\label{equation:chirp}
f = K_2 \cdot \left(\sum_{j = 0}^{\frac{r}{\gamma} - 1}\mathds{1}_{[l + j\gamma, l + \frac{(2j + 1)\gamma}{2}]} + \frac{1}{4}\mathds{1}_{[l + b + r, l + b + 2r]}\right),
\end{equation}
where $\gamma$ controls the high-frequency perturbation and $K_2$ is chosen so that $f$ integrates to unity. To generate our synthetic dataset we proceed as follows, fixing $l,r$ and $b$, we generate $f_1,...,f_{30}$ from~\eqref{equation:hump}. We corrupt each signal with standard Gaussian noise to obtain $30$ noisy double hump functions. We then obtain two separate classes from the chirp-hump functions by first randomly sampling $\gamma \in \{\gamma_1, \gamma_2\}$ with equal probability. Each chirp-hump function is then corrupted with standard Gaussian noise. We then obtain functions $f_{31},...,f_{60}$ as chirp-hump functions with proportion $R_1$ having perturbation parameter $\gamma_1$ and proportion $R_2$ having perturbation parameter $\gamma_2$. All functions are defined on $[0,1]$ discretized on a uniform gird of length $N = 150$.

To apply $\LWp$ to discriminate between these signals we first need to satisfy positivity and mass constraints. Thus, each function $f$ is normalised as follows $g = \frac{f + \chi}{\int (f + \chi) }$, for a small number $\chi$. Discrete measures $\mu_1,...,\mu_{60}$ are then defined to be the probability measures with density $g_1,\dots,g_{60}$  with respect to the uniform grid on $[0,1]$.
A reference measure $\sigma$ is constructed as an empirical average of all these measures. We then use entropy regularised methods, see~\cite{Cuturi:2013, Benamou:2015} or Section~\ref{subsubsec:Methods:ComputingTLp:Ent} to compute optimal transport plans between $\sigma$ and $\mu_i$, where $i = 1,...,60$; after which an optimal transport map is computed using barycentric projection. We then embed the measures into Euclidean space as described in Section~\ref{section:LOT}. Note that for this linear embedding $E \in \mathbb{R}^{150 \times 60}$. This method requires only the computation of $60$ transport plans, rather than $59 \times 30$ if all pairwise $\Wp$ distance were computed.

The $\LTLp$ framework can be applied directly without ad-hoc pre-processing and normalisation. The base measure is taken as the uniform measure on $[0,1]$. The reference measure $\sigma$ is taken to be the base measure and the reference signal $h$ is the empirical average of all signals. Optimal $\TLp$ plans are computed again using entropy regularised methods from $(\sigma, h)$ to $(\mu_i, f_i)$ for $i = 1,...,60$. Recall that this requires the computation of the optimal transport plan from $\tilde{\sigma}$ to $\tilde{\mu_i}$ for $i = 1,...,60$. The map is then obtained from the plan via barycentric projection. A linear embedding $U$ is obtained as detailed in Section~\ref{section:LTLP}. Note this linear embedding is higher dimensional and $U\in \mathbb{R}^{300\times 60}$.

To assess the discriminating ability of $\LWp$ and $\LTLp$ we apply $K$-means clustering with $K = 3$ to the linear embedding to see if we can recover the true underlying classes. Since $K$-means attempts to minimise the within class distance and maximise between class distance, we expect a distance which is able to detect the differences between the classes to have the best performance. We take the clustering returned from the $K$-means algorithm and compare it to the true clustering using the adjusted Rand index (ARI) \cite{rand:1971, Hubert:1985}. An ARI is a score with $1$ indicating perfect agreement, $0$ indicating the method performs as well as one would expect if random assignment where made and the ARI can be negative if the method is worse than random. We repeat our method $100$ times to produce a distribution of scores.

Figure \ref{figure:1D} demonstrate the improved performance of using the $\TLp$ distance to form a linear embedding of the data. The median ARI using the $\TLp$ distance was $1$, whilst the median for using the $\Wp$ distance was $0.8129$. The $\LTLp$ approach outperforms the $\LWp$ approach significantly (Kolmogorov-Smirnov (KS) Test, $p$-value less than $10^{-4}$). Furthermore, Figure \ref{figure:1D} panels (b) and (c) demonstrates that the linear embedding produce much tighter and therefore more interpretable clusters in $\LTLp$ as compared to $\LWp$.

\begin{figure}[ht]
	\begin{subfigure}[t]{0.33\textwidth}
		\centering
		\includegraphics[width=0.95\textwidth]{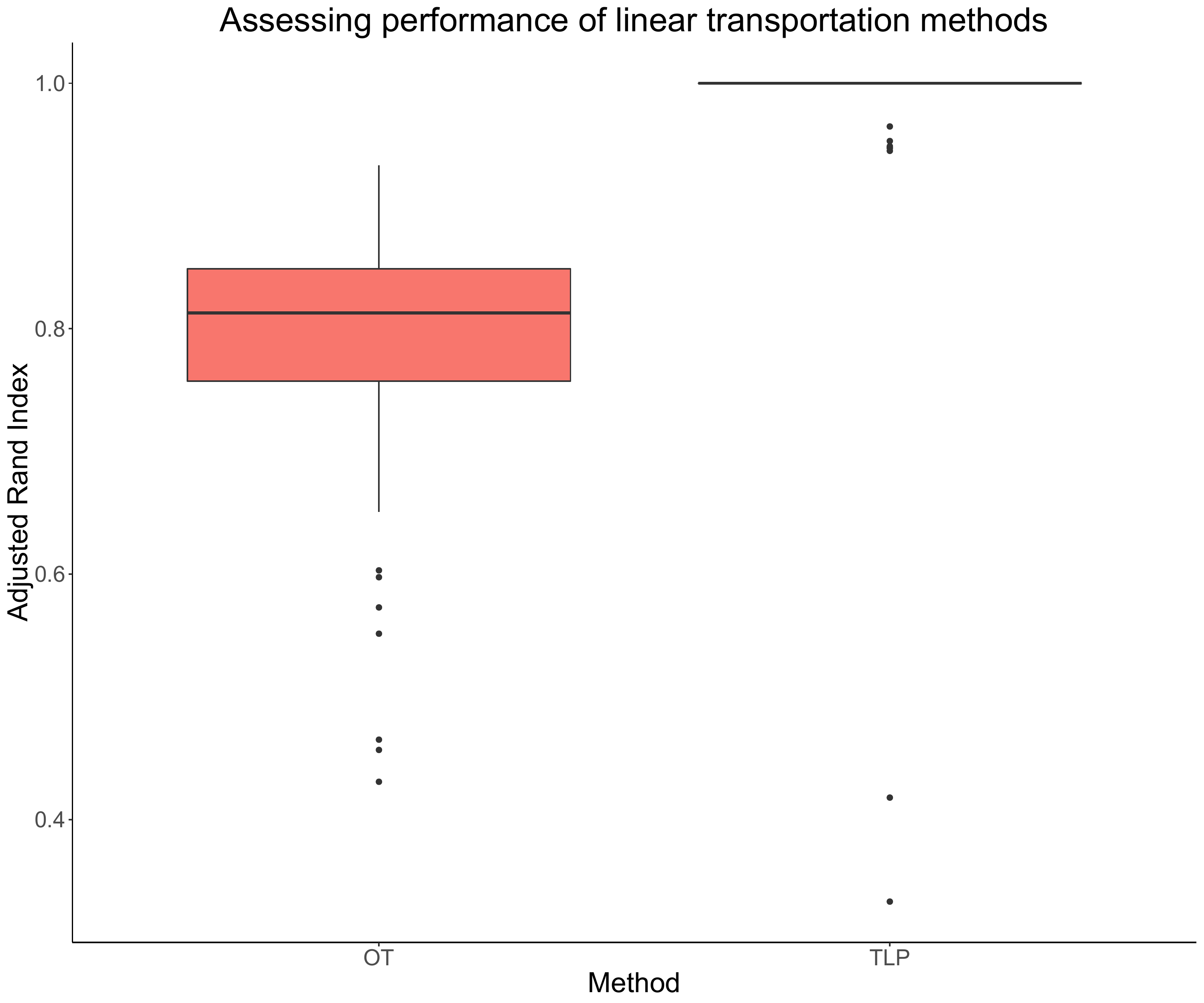}
		\caption{}
	\end{subfigure}
	\begin{subfigure}[t]{0.33\textwidth}
		\centering
		\includegraphics[width=0.95\textwidth]{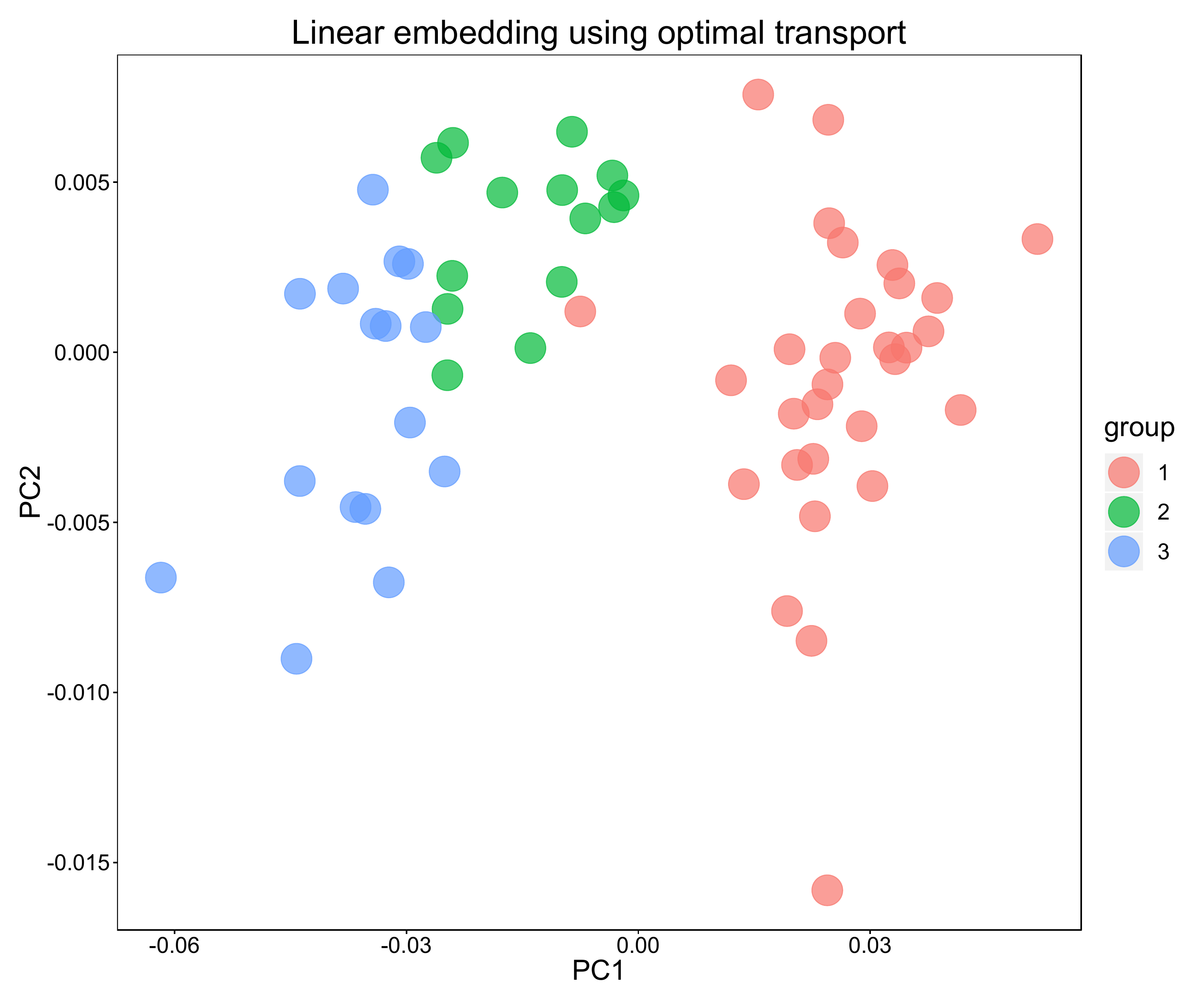}
		\caption{}
	\end{subfigure}%
	\begin{subfigure}[t]{0.33\textwidth}
		\centering
		\includegraphics[width=0.95\textwidth]{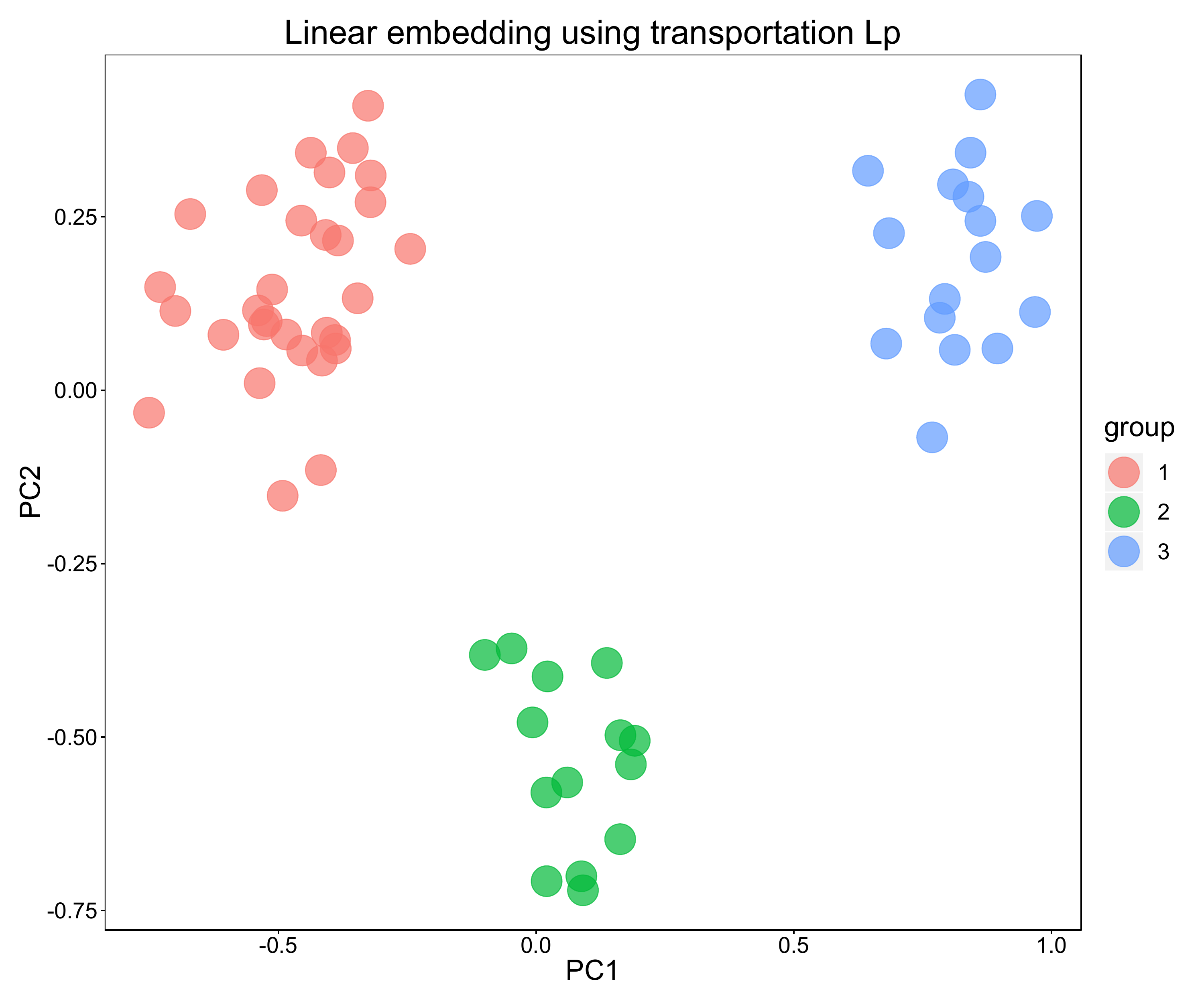}
		\caption{}
	\end{subfigure}%

\begin{subfigure}[t]{0.49\textwidth}
		\centering
		\includegraphics[width=0.646\textwidth]{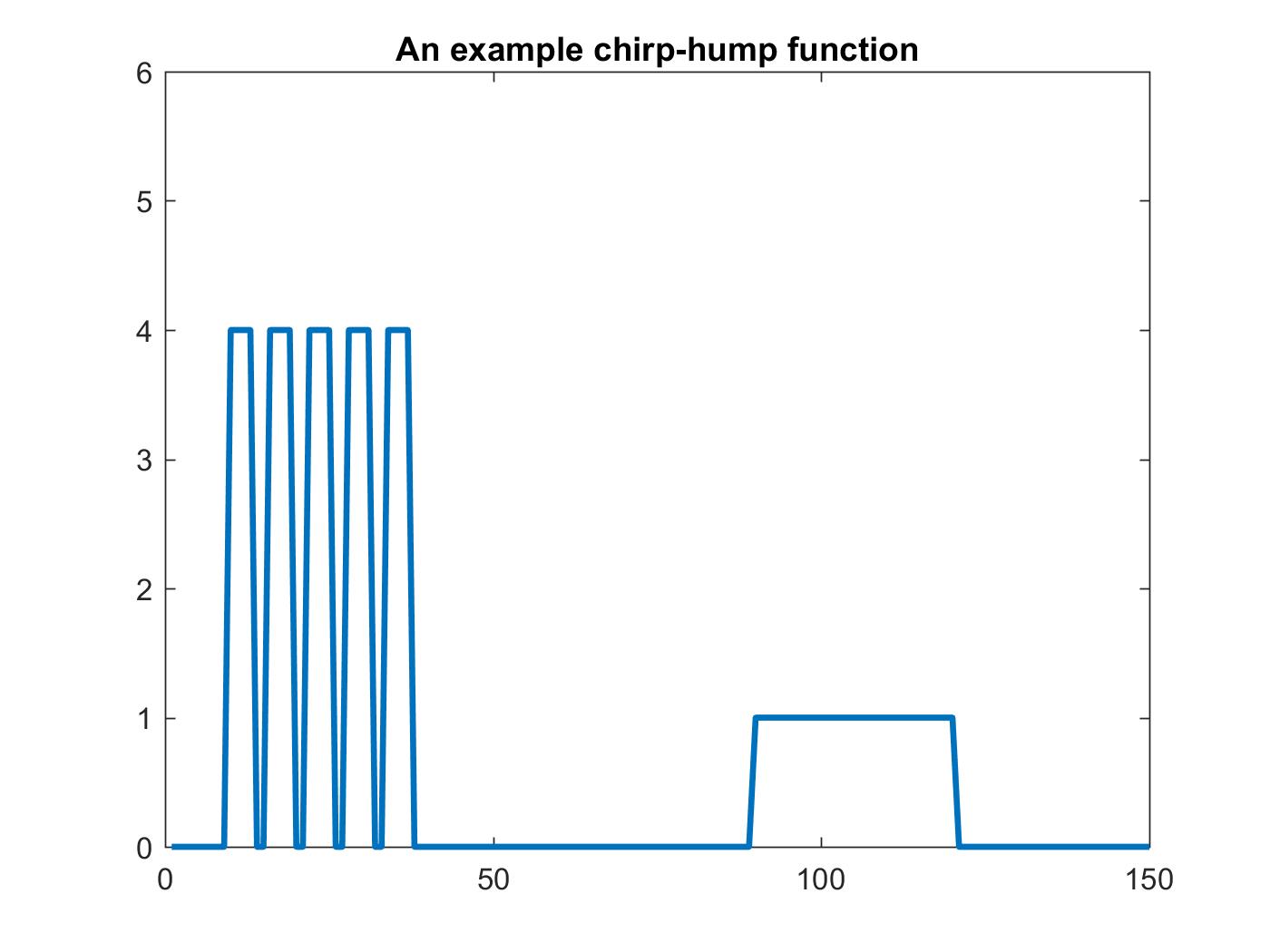}
		\caption{}
\end{subfigure}
\begin{subfigure}[t]{0.49\textwidth}
	\centering
	\includegraphics[width=0.646\textwidth]{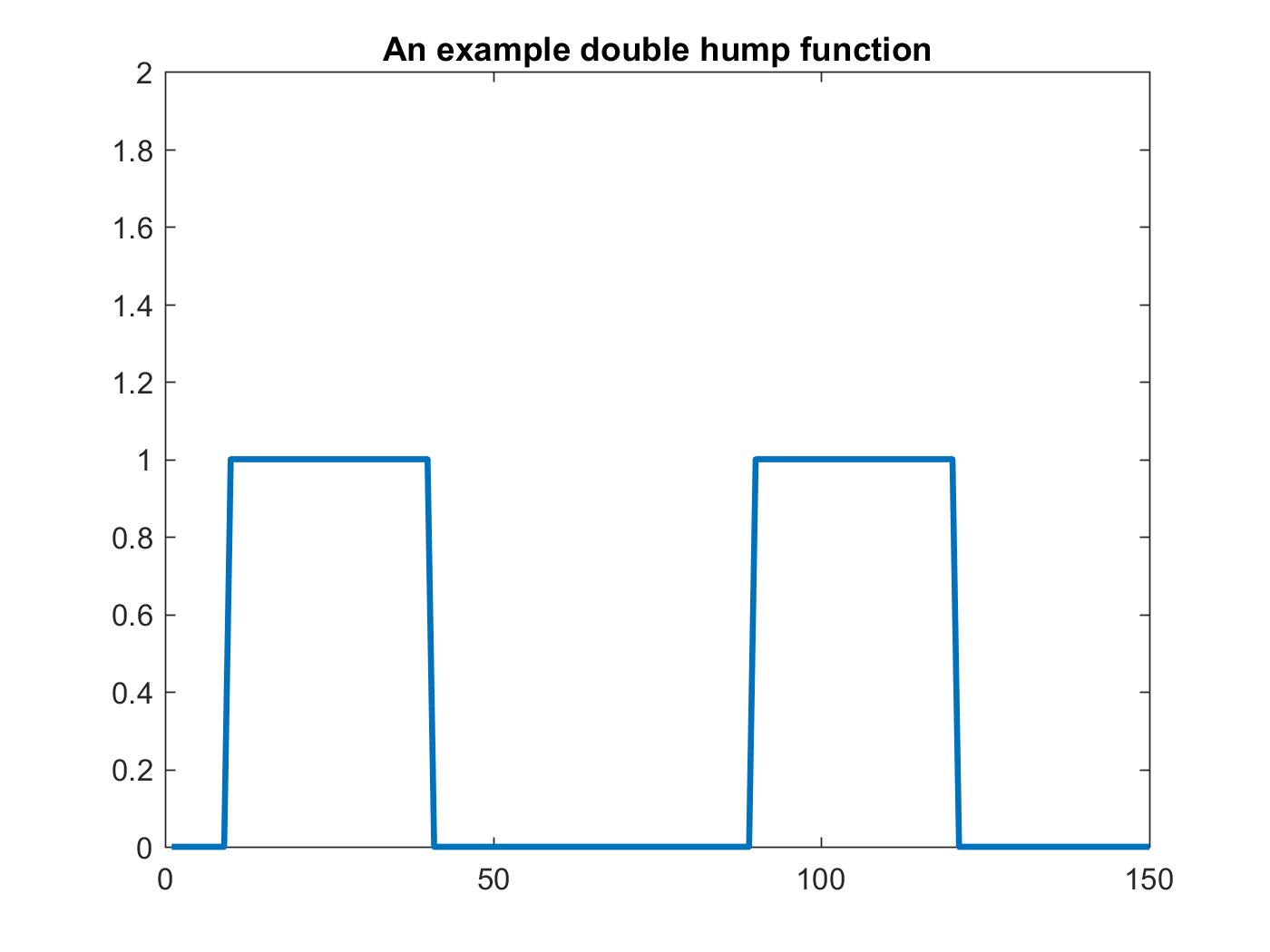}
	\caption{}
\end{subfigure}

	\caption{(a) A boxplot showing the distribution of adjusted Rand index over different runs of the algorithm. The $\LTLp$ method clearly outperforms the $\LWp$ approach. (b) An example of a $\LWp$ embedding showing classes that appear dispersed, and between class distance is smaller than within class distance. (c) An example of a $\LTLp$ embedding, between class distance is clearly greater than within class distance. This results in tighter clusters. 
	(d) An example chirp-hump function (without noise). (e) An example double hump function (without noise).}
	\label{figure:1D}
\end{figure}

\subsection{Two Dimensional Synthetic Signal Processing} \label{subsec:Results:Syn2d}

We consider a more challenging synthetic two dimensional signal processing problem. As in the previous section, we take $p = 2$ as the exponent of the cost function in both settings. Consider the following class of functions, defined on a grid on $[0,1]^2$:
\[ \mathcal{M}_1 = \left\{f:[0,1]^2 \to \bbR \,|\, f(x_i,y_j) = \alpha x_i \mathrm{e}^{-x_i^2-y_j^2} + \sigma_{ij}, \alpha \sim \mathcal{N}(0,1),\sigma_{ij} \sim \mathcal{N}(0,1) \right\}, \]
and
\[ \mathcal{M}_2 = \left\{f:[0,1]^2 \to \bbR \,|\, f(x_i,y_j) = \alpha x_i \mathrm{e}^{-x_i^2-y_j^2} + \sigma_{ij}, \alpha \sim \mathcal{N}(-4, 1.5), \sigma_{ij} \sim \mathcal{N}(0,1) \right\}. \]
Furthermore, we introduce a perturbation to functions in $\mathcal{M}_1$, by first sampling an integer $n \in \{10,...,20\}$ each with equal probability and then setting $f = -2$ for $n$ randomly chosen coordinates on the grid, with each coordinate having equal probability of being chosen.

We note that these signals take positive and negative values and thus to apply $\Wp$ to this problem we perform normalisation. We add a constant to each (random) function and then ensure that mass still integrates to unity. This ad-hoc normalisation procedure introduces signal compression into the problem causing important features to become suppressed. The $\TLp$ distance can be applied without normalisation or pre-processing.

We generate $25$ random functions from each class. We then apply both the $\LWp$ and $\LTLp$ methods to the resulting dataset. We performed PCA and $K$-means clustering, with $K = 2$, on the linear embedding 
to see if we can discriminate between the two classes. As in the previous section, we compare the resultant clustering with ground truth using the ARI. We repeat this process $100$ times to obtain a distribution of scores.

Figure \ref{figure:2D} show that the $\LTLp$ embedding outperforms the $\LWp$ embedding. The median ARI for the $\LTLp$ approach is $0.92$ and the median ARI for the $\LWp$ approach is $0.5689$. The $\LTLp$ distance produced significantly better results (KS-test, $p < 10^{-4}$) and the PCA plots in Figure~\ref{figure:2D} demonstrate that the two classes overlap when using $\Wp$ but separate when using $\TLp$.

\begin{figure}[ht]
	\begin{subfigure}[t]{0.33\textwidth}
		\centering
		\includegraphics[width=0.95\textwidth]{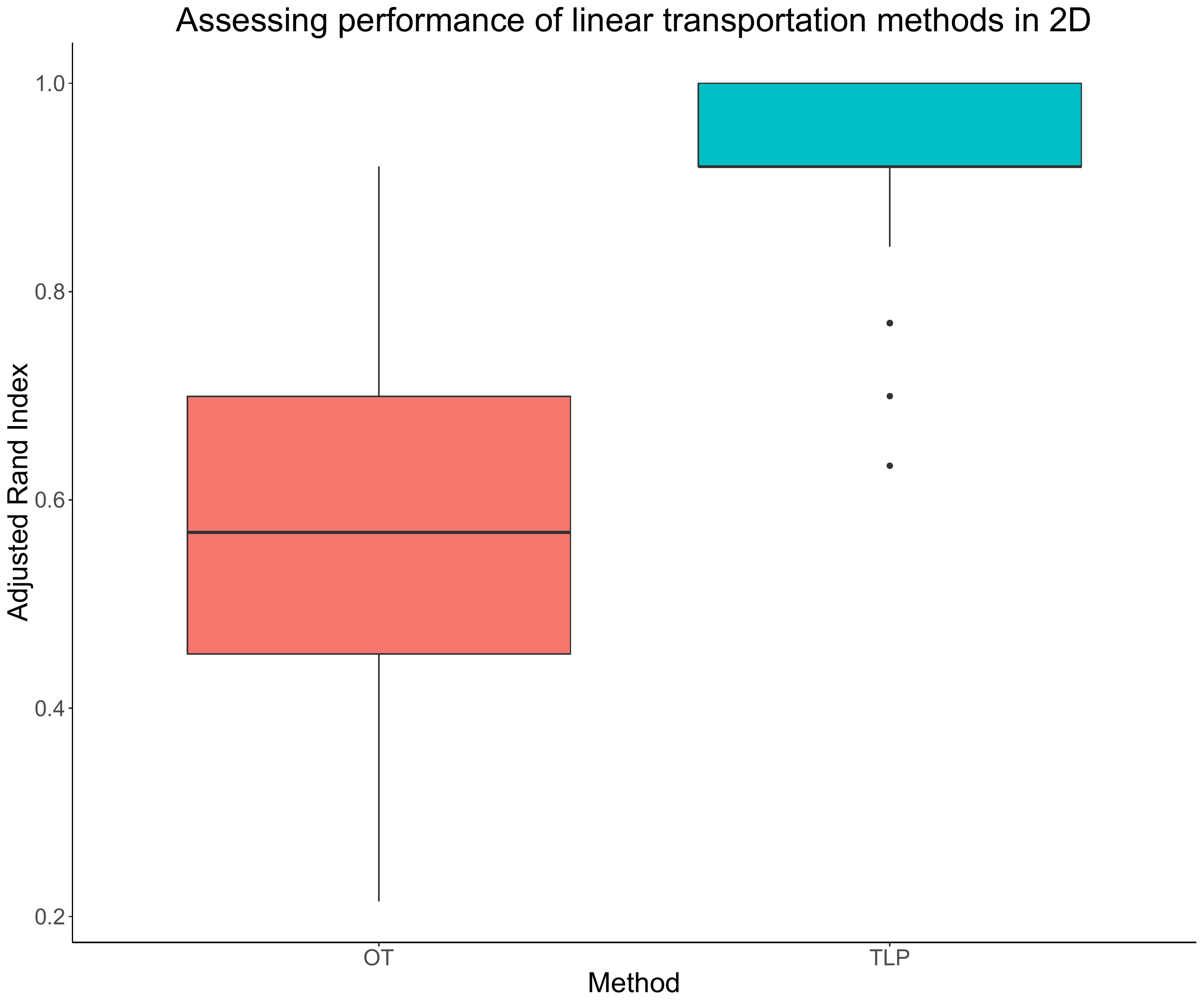}
		\caption{}
	\end{subfigure}
	\begin{subfigure}[t]{0.33\textwidth}
		\centering
		\includegraphics[width=0.95\textwidth]{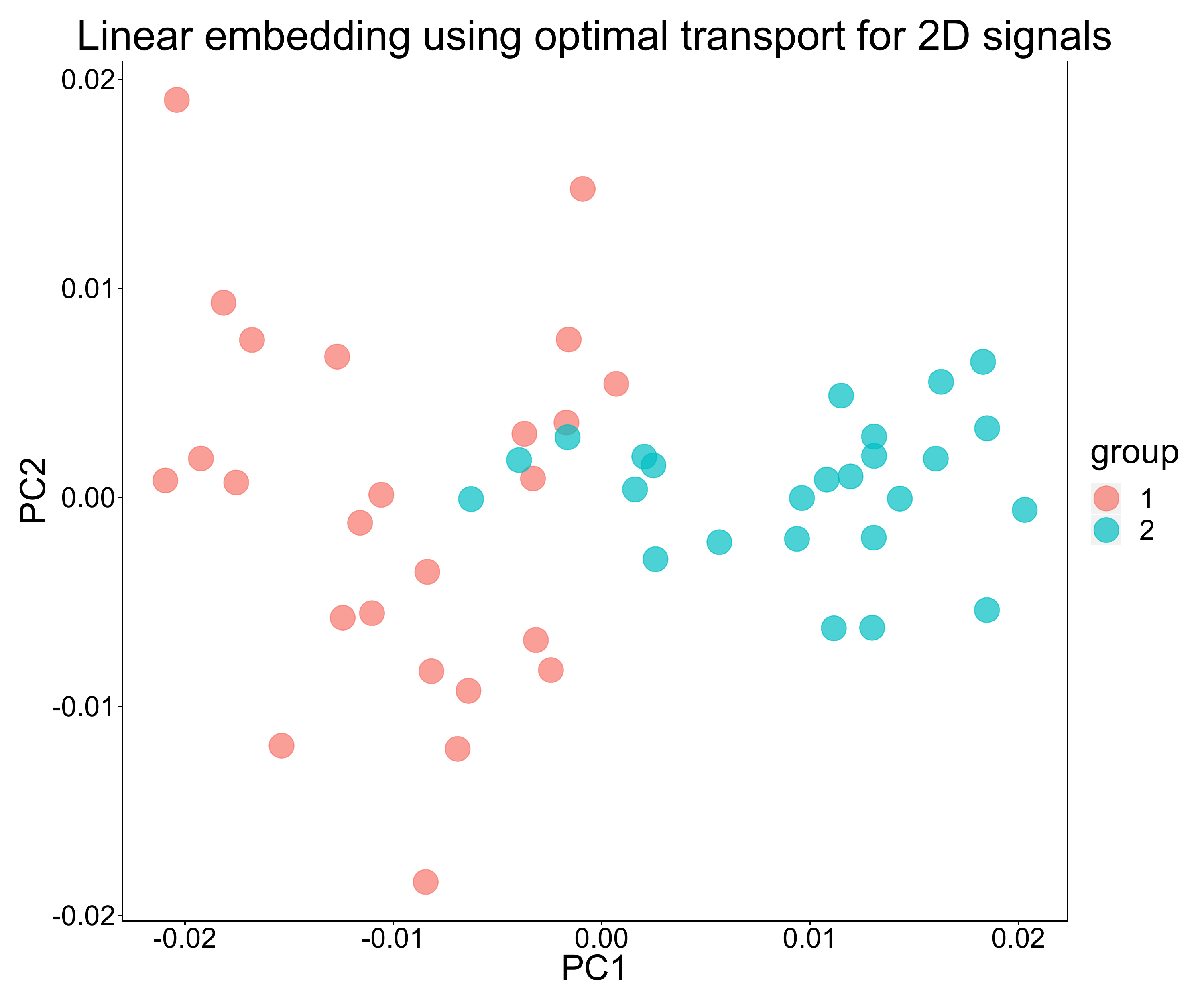}
		\caption{}
	\end{subfigure}%
	\begin{subfigure}[t]{0.33\textwidth}
		\centering
		\includegraphics[width=0.95\textwidth]{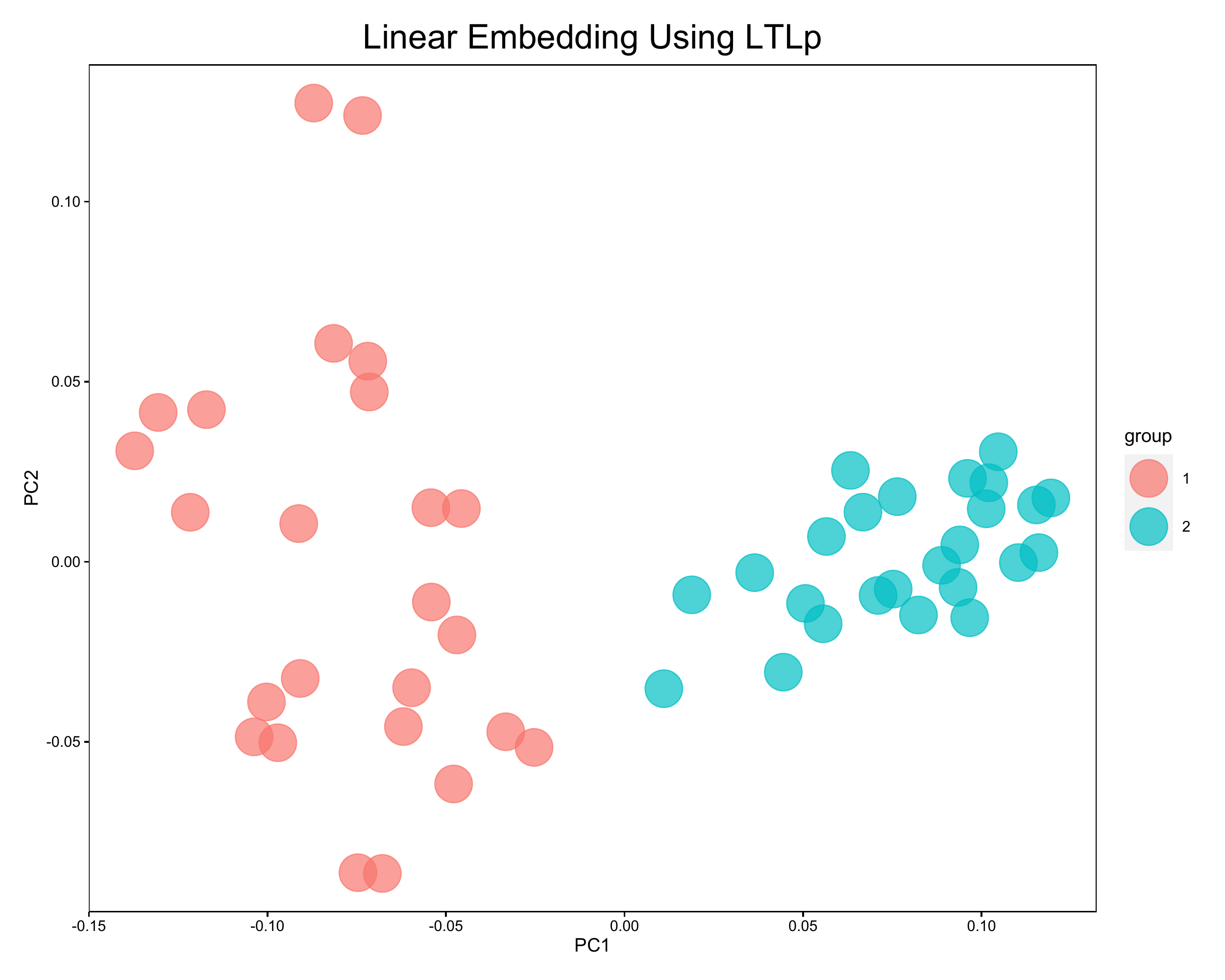}
		\caption{}
	\end{subfigure}%

	\begin{subfigure}[t]{0.49\textwidth}
	\centering
	\includegraphics[width=0.646\textwidth]{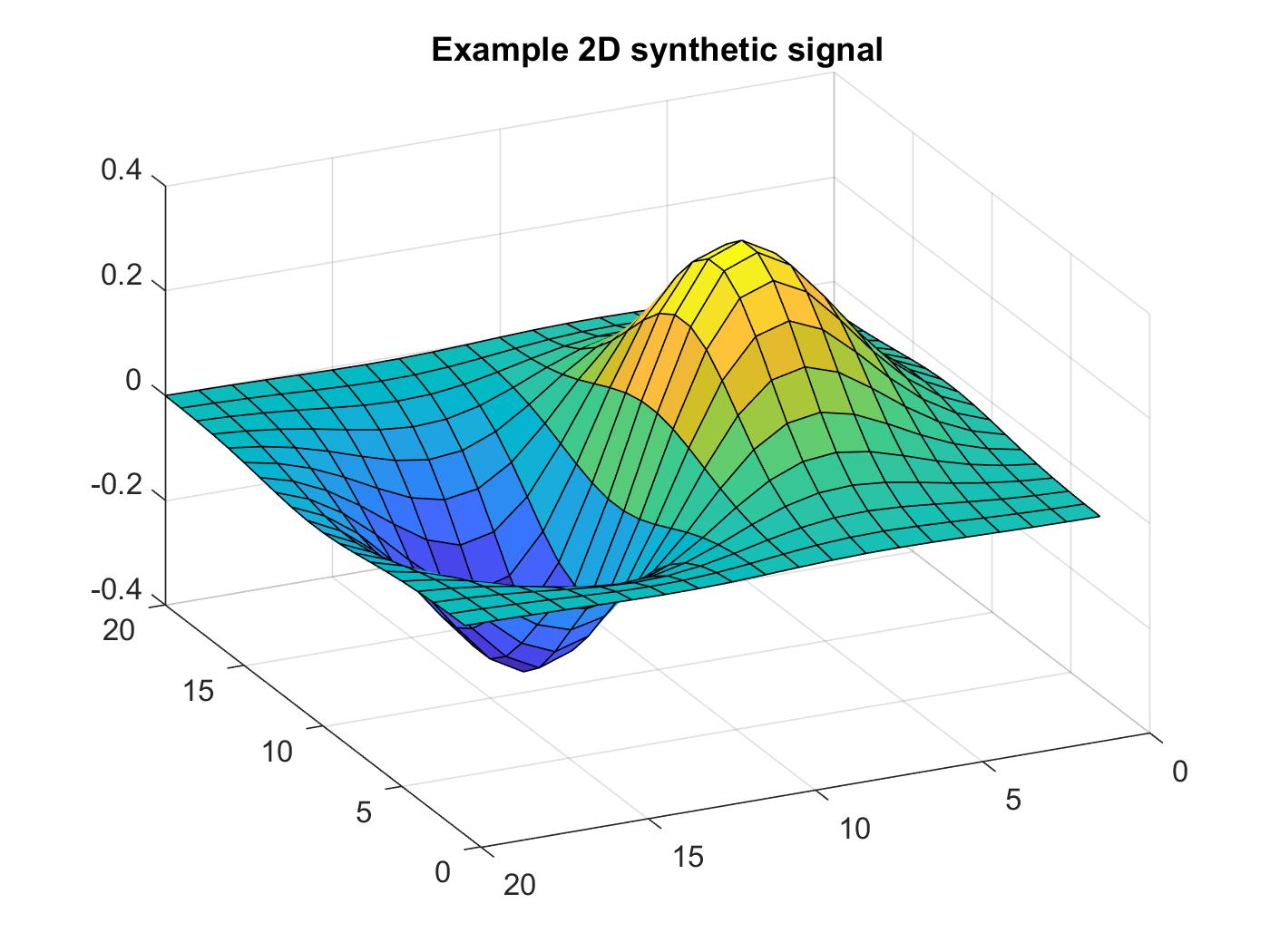}
	\caption{}
\end{subfigure}%
\begin{subfigure}[t]{0.49\textwidth}
	\centering
	\includegraphics[width=0.646\textwidth]{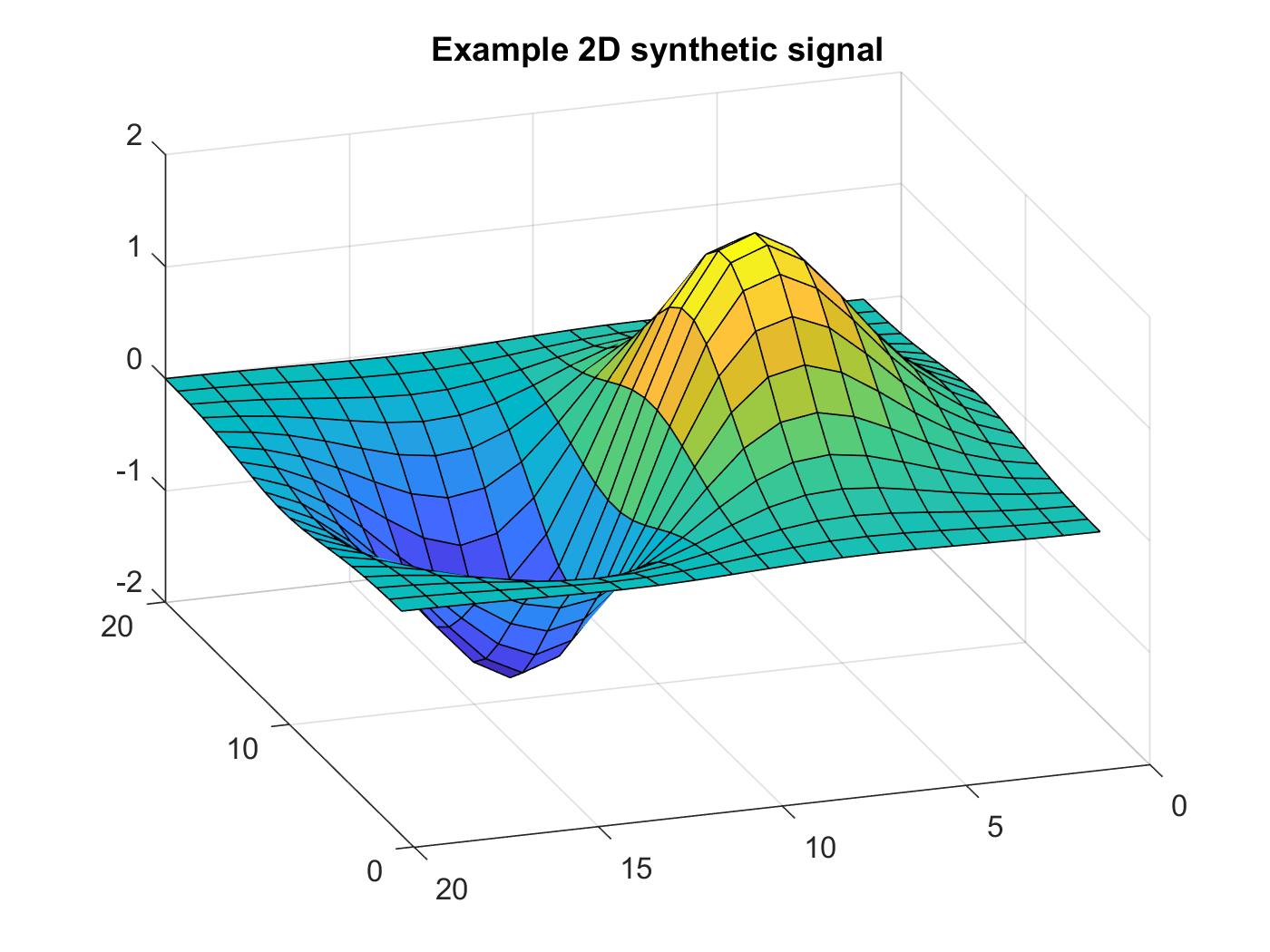}
	\caption{}
\end{subfigure}%
	
	\caption{(a) A boxplot showing the distribution of adjusted Rand index over different runs of the algorithm. The $\TLp$ based method clearly outperforms the $\Wp$ approach. (b) An example of a $\LWp$ embedding where we see overlap between the classes as a result of signal compression. (c) An example of a $\LTLp$ embedding where there is a clear separation between classes. 
	(d) An example signal (without noise) from $\mathcal{M}_1$. (e) An example signal (without noise) from $\mathcal{M}_2$. 
	}
	\label{figure:2D}
\end{figure}

\begin{figure}[ht]
	\begin{subfigure}[t]{0.33\textwidth}
		\vskip 0pt
		\centering
		\includegraphics[width=0.95\textwidth]{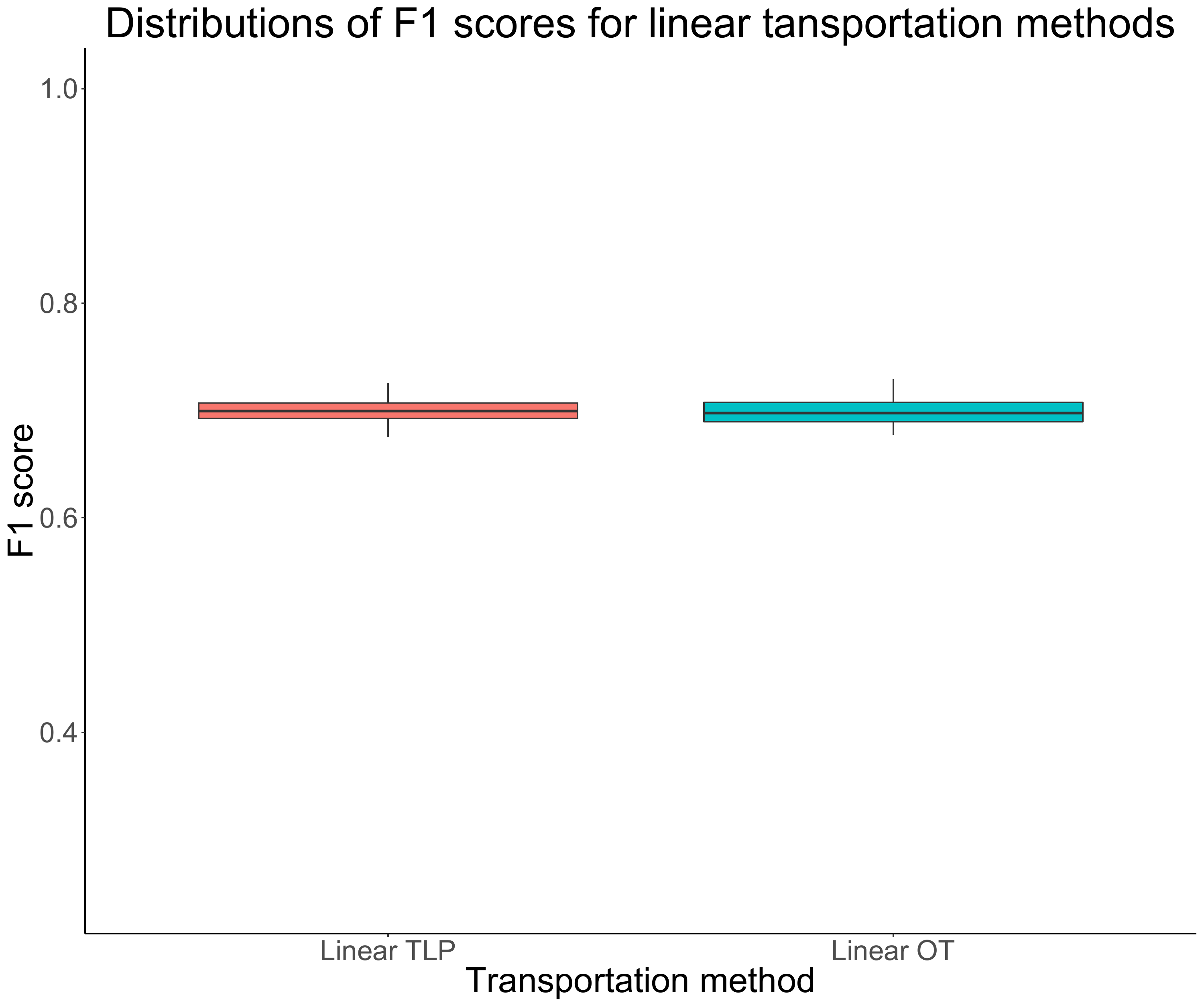}
		\caption{}
	\end{subfigure}
	\begin{subfigure}[t]{0.33\textwidth}
		\vskip 0pt
		\centering
		\includegraphics[width=0.95\textwidth]{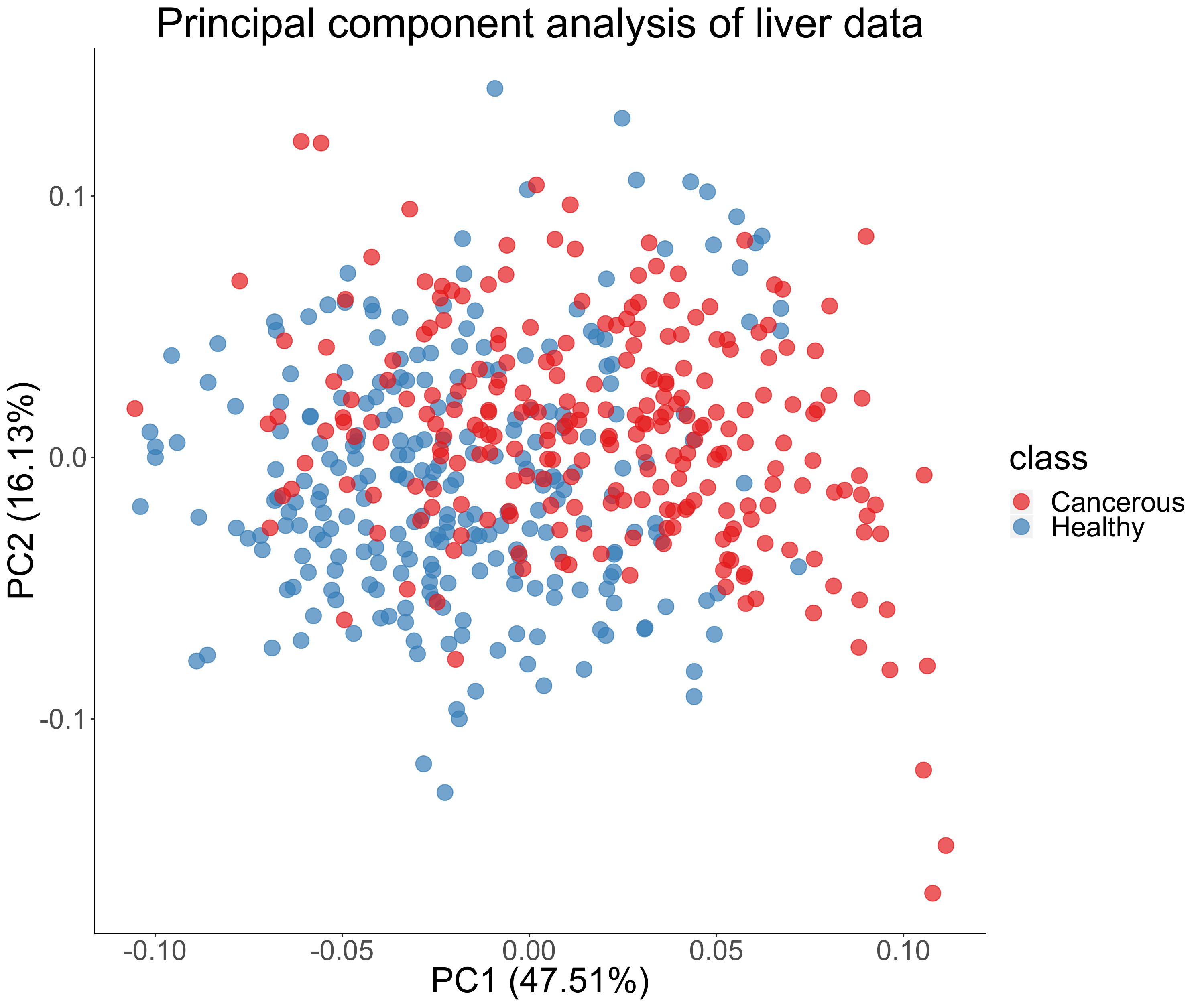}
		\caption{}
	\end{subfigure}%
	\begin{subfigure}[t]{0.33\textwidth}
	    \vskip 0pt
		\centering
		\includegraphics[width=0.95\textwidth]{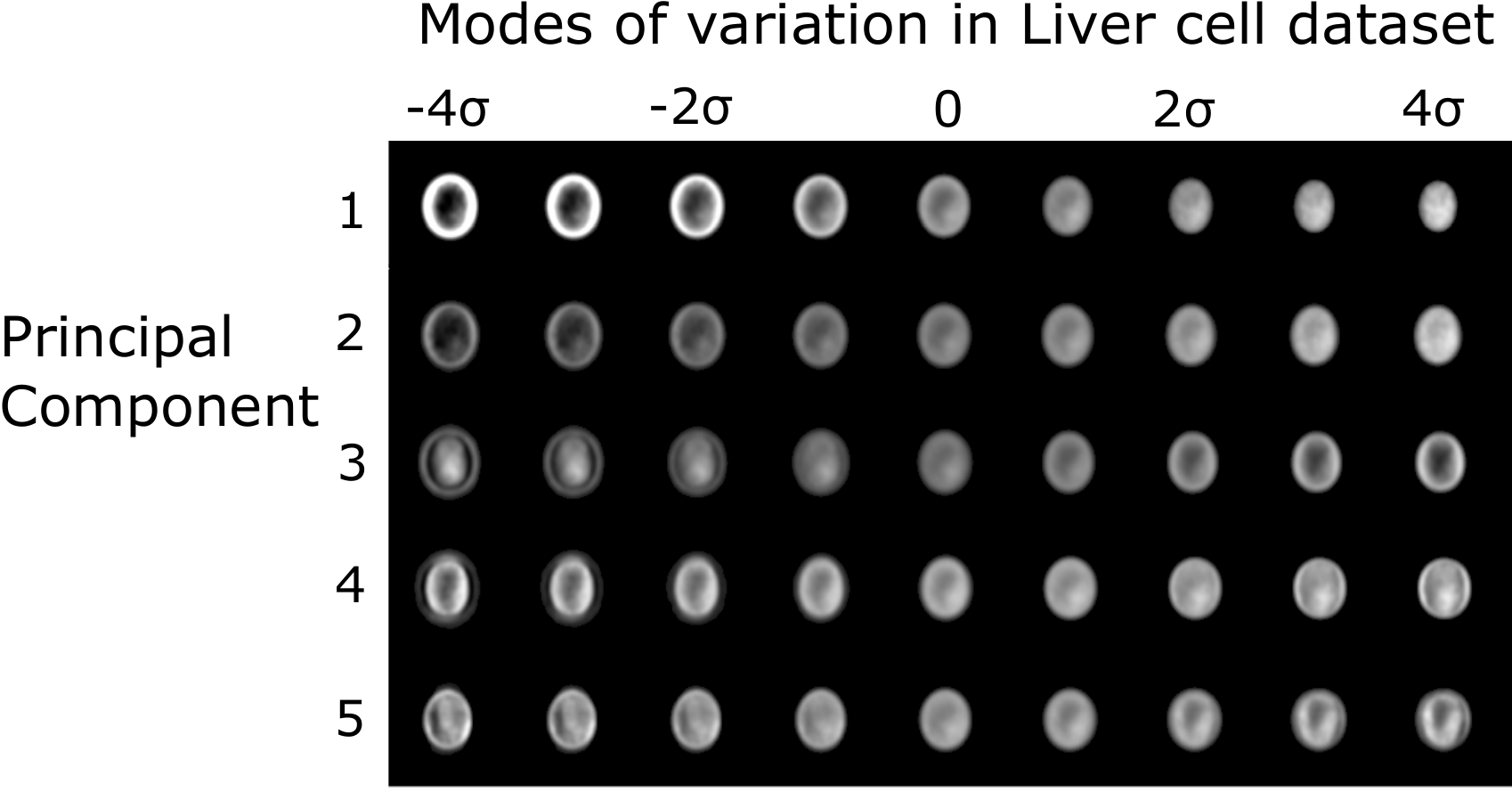}
		\caption{}
	\end{subfigure}%

	\begin{subfigure}[t]{0.49\textwidth}
	\centering
	\includegraphics[width=0.646\textwidth]{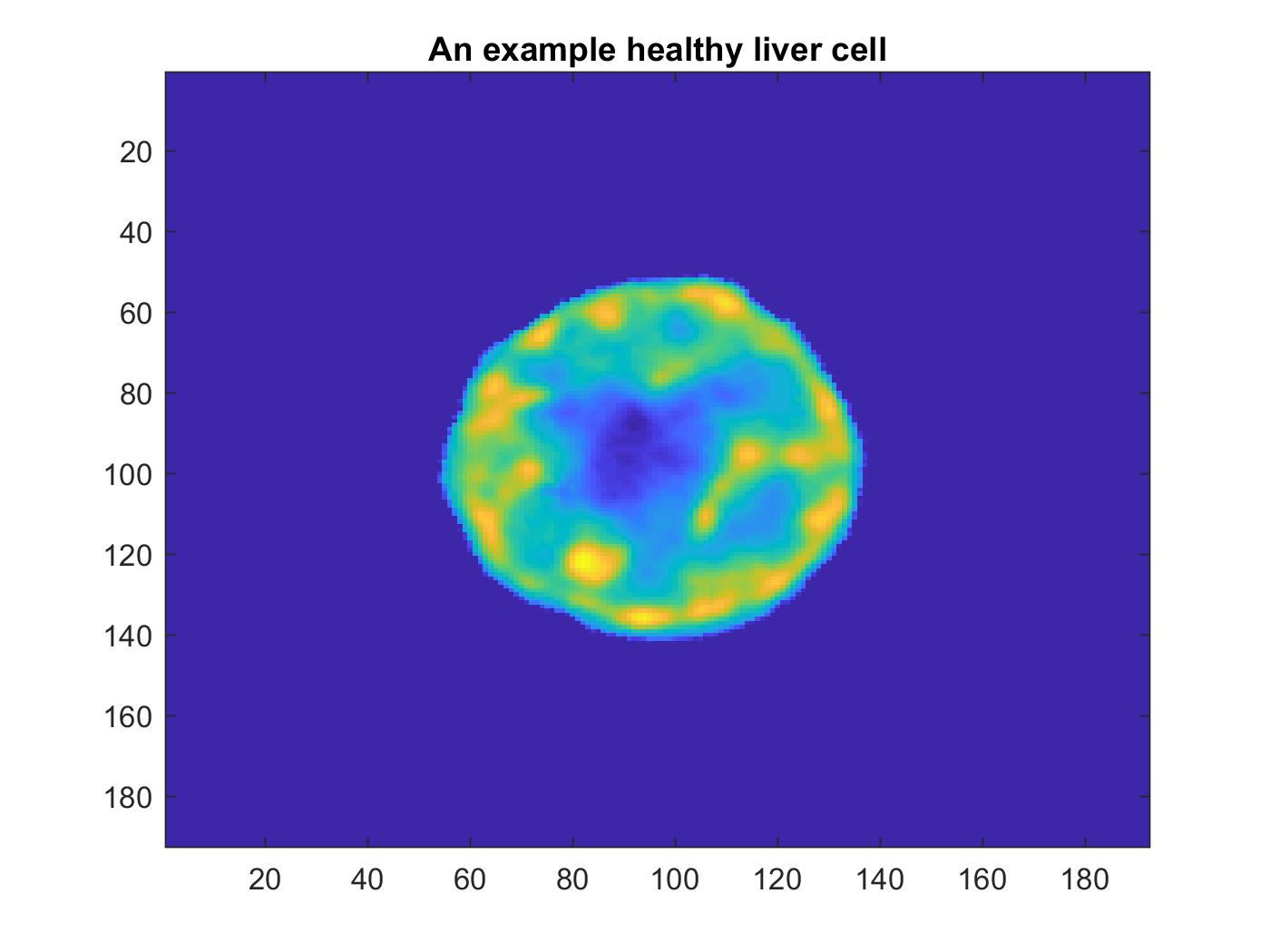}
	\caption{}
\end{subfigure}%
\begin{subfigure}[t]{0.49\textwidth}
	\centering
	\includegraphics[width=0.646\textwidth]{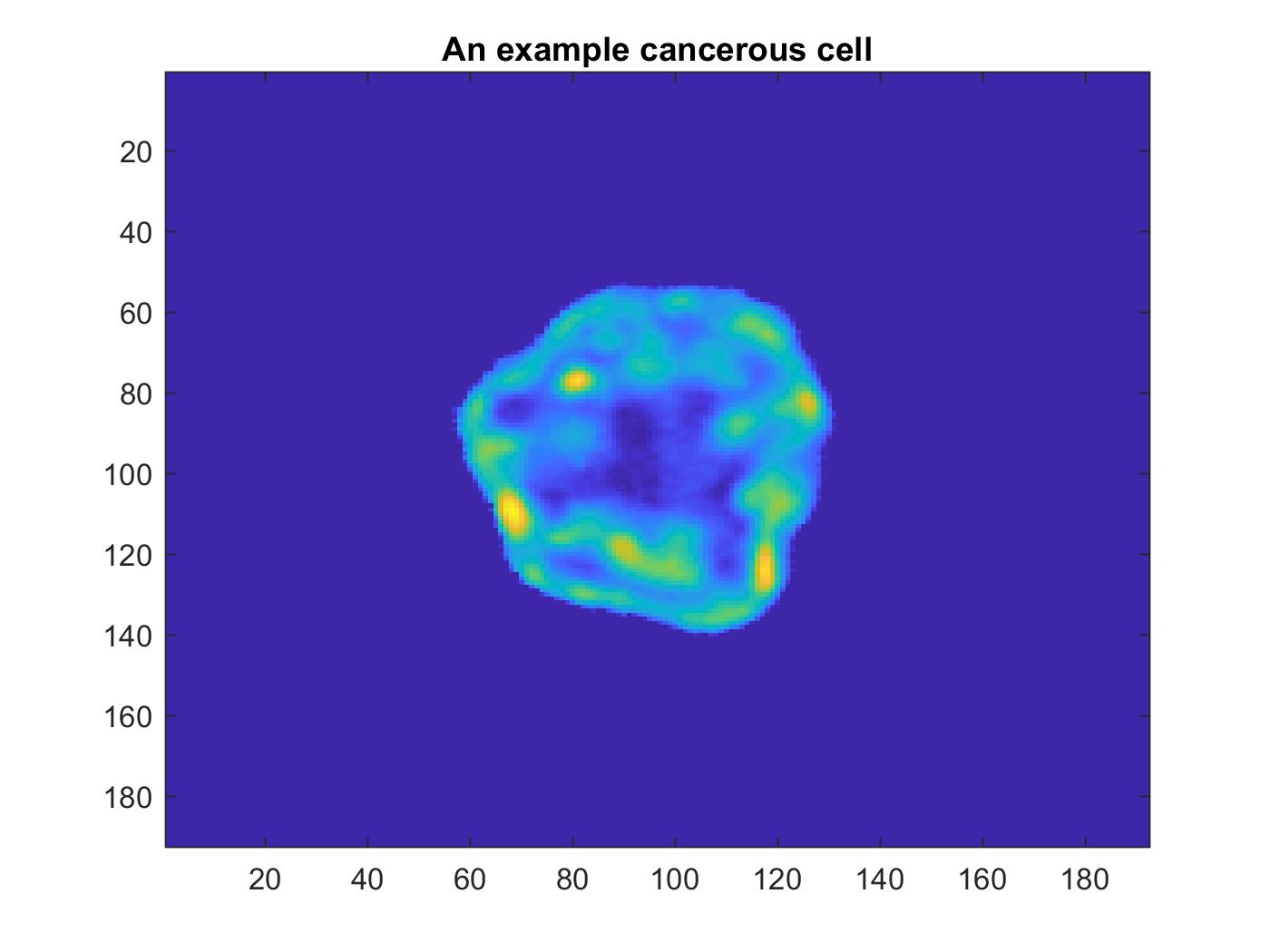}
	\caption{}
\end{subfigure}%
	
	\caption{(a) A boxplot showing the distribution of adjusted Rand index using the 1NN classifier on the $\LWp$ and $\LTLp$ embeddings. Unsurprisingly, both methods perform equally well on this dataset.  (b) A PCA plot of the $\LTLp$ embedding where we see overlap between the classes but distinct class distributions (c) The first $5$ principal modes of variation using linear interpolation along the eigenvectors in PCA space. (d) An example healthy liver cell. (e) An example cancerous liver cell.
}
	\label{figure:LiverCells}
\end{figure}

\subsection{Application to Cell Morphometry} \label{subsec:Results:CM}

In this section, we analyse the liver dataset of \cite{Basu:2014} containing $250$ normal and $250$ cancerous liver cells. \cite{Basu:2014} proposed a transportation based morphometry analysis and this facilitated high accuracy classification. Furthermore, the generative nature of optimal transport allowed them to visualise the modes of variation in the dataset, allowing for superior interpretation of the data. Each cell image is defined on a $192 \times 192$ pixel grid with a single intensity channel. Thus both $\LWp$ and $\LTLp$ are applicable. 

For consistency we apply flow minimisation techniques to compute the transport maps in each setting \cite{Kolouri:2016b}. To alleviate numerical issues, as in \cite{Kolouri:2016b}, we apply a Gaussian low-pass filter with standard deviation $2$ to smooth the data. We perform mass normalisation so that optimal transport can be applied and these signals are also used for $\TLp$, so that the spatial and intensity features are on the same scale. A linear embedding is obtained from the transport maps.

Once this linear embedding is obtained we use the 1 nearest neighbour (1NN) algorithm to predict normal or cancerous from the linear embedding of the signals. We assess performance with a $5$-fold cross-validation framework; that is, an $80/20$ split between training and testing partitions. 

\begin{figure}
\newcommand{\widd}{0.45\textwidth}
\newcolumntype{D}{>{\centering\arraybackslash}m{\widd}}
\begin{center}
\begin{tabular}{l*2{D}@{ }}
    &  Raw Returns (RR) & Market-Excess Returns (MR)  \\ 
\hline 
\rotatebox{90}{1-day}  
& \includegraphics[width=\widd]{{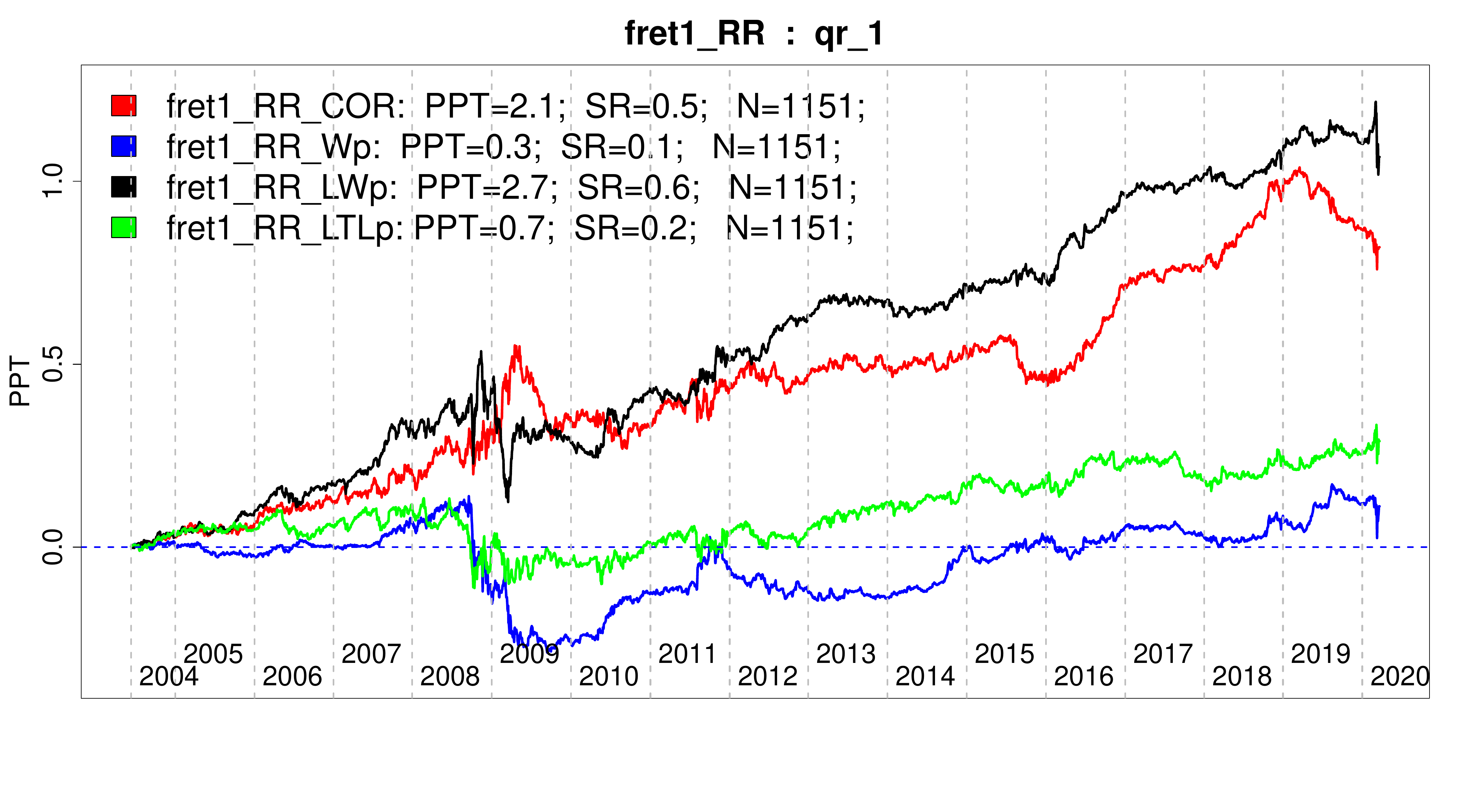}} 
& \includegraphics[width=\widd]{{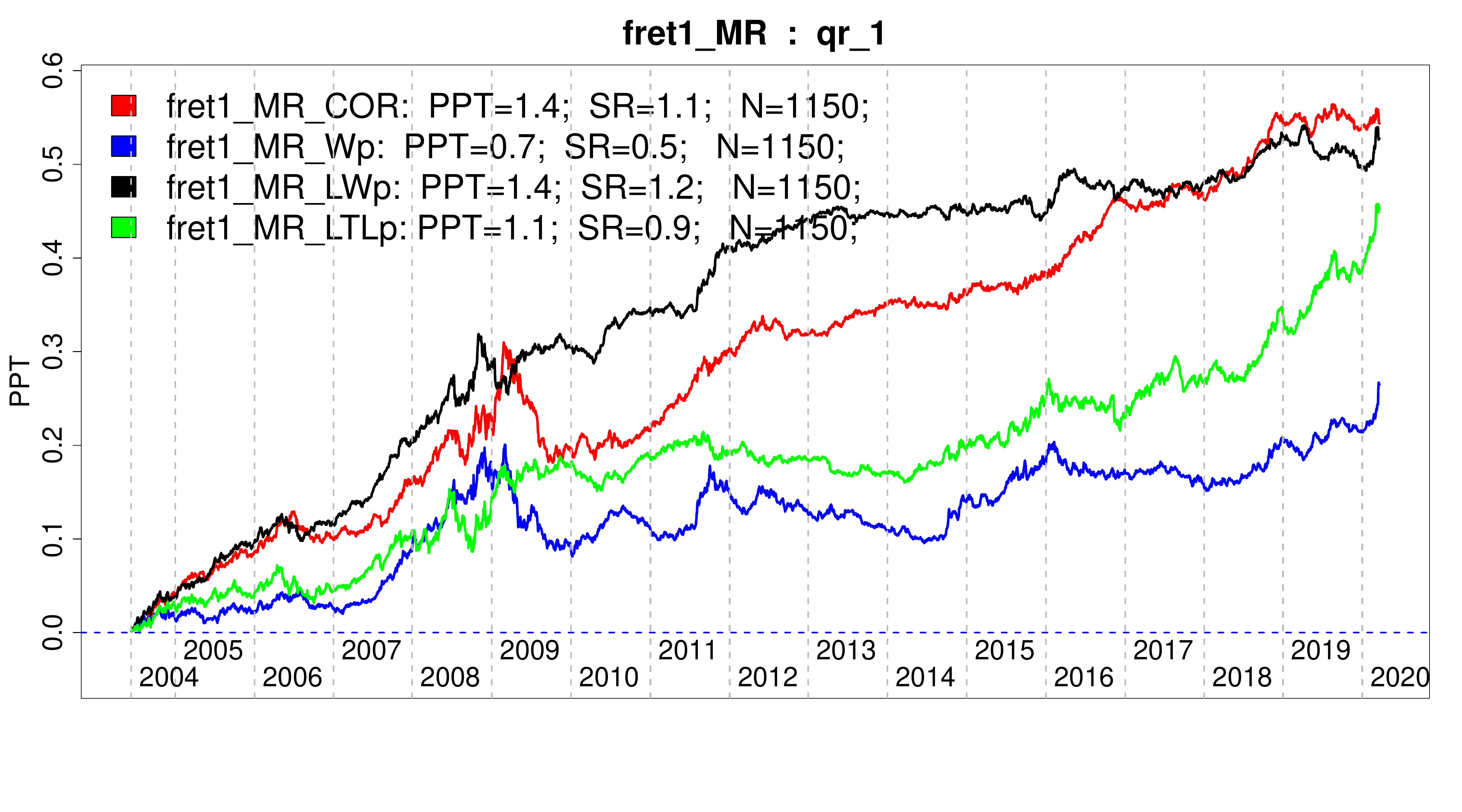}}  \\
%
\hline 
\rotatebox{90}{5-day}  
& \includegraphics[width=\widd]{{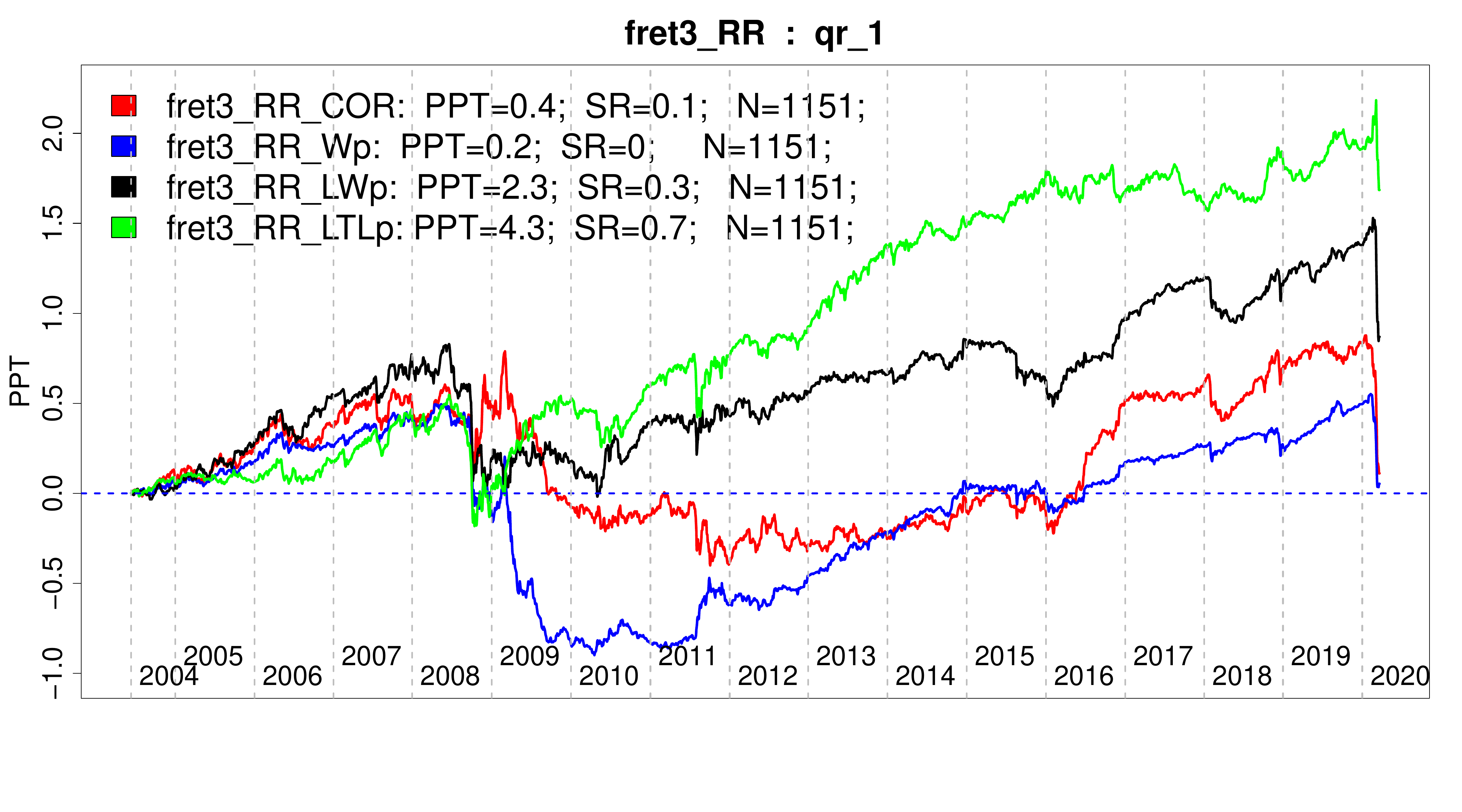}} 
& \includegraphics[width=\widd]{{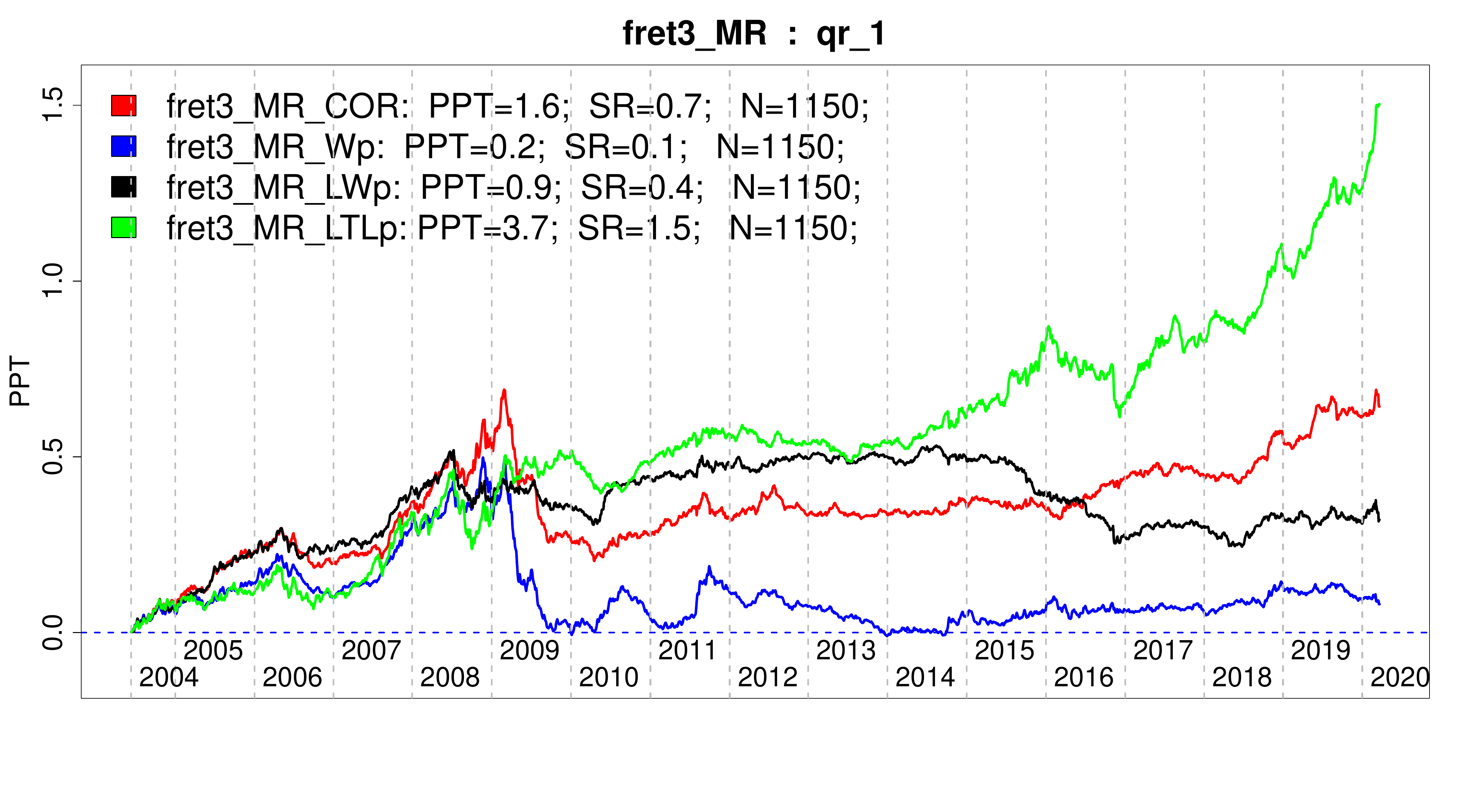}}  \\
\hline 
\rotatebox{90}{5-day}  
& \includegraphics[width=\widd]{{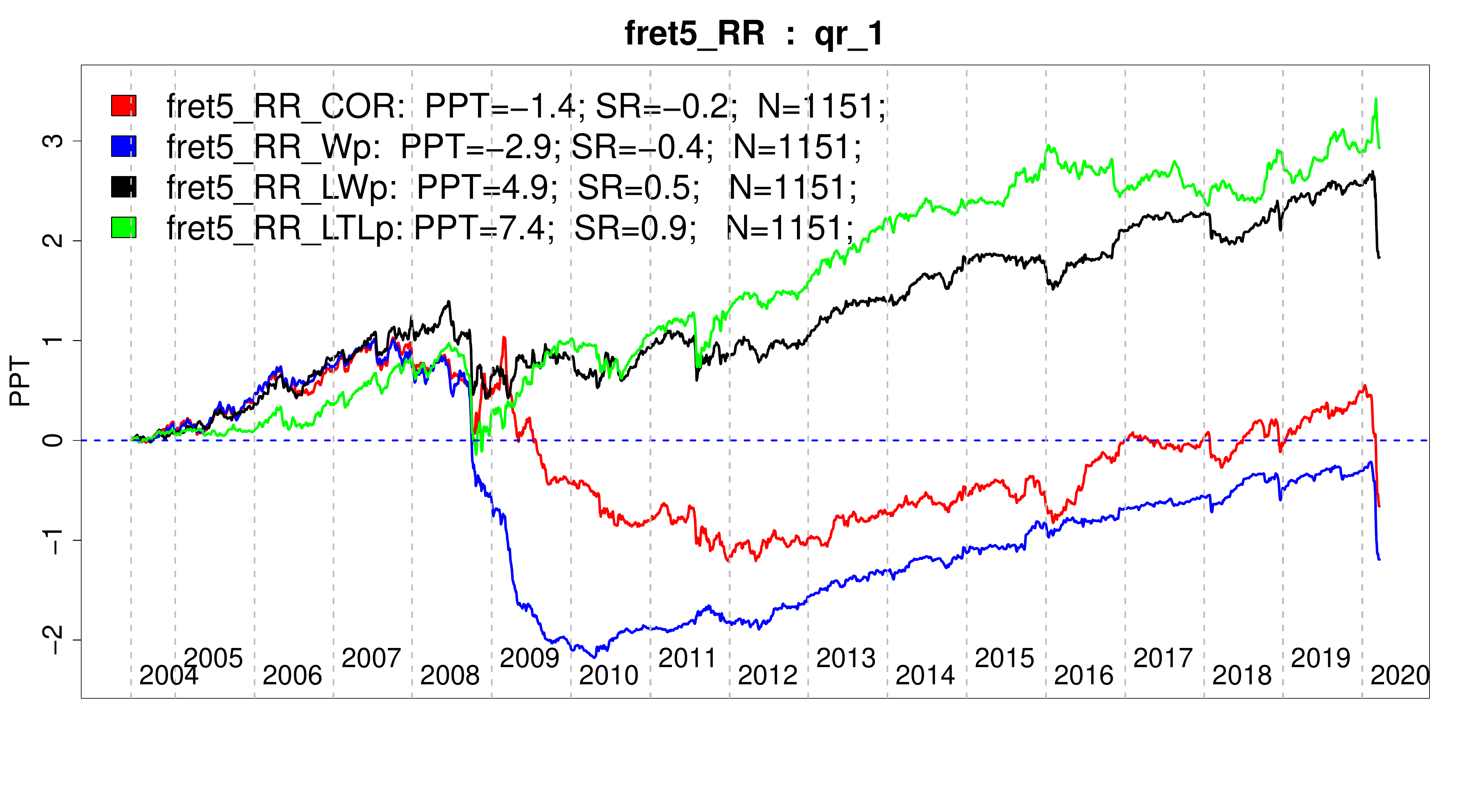}} 
& \includegraphics[width=\widd]{{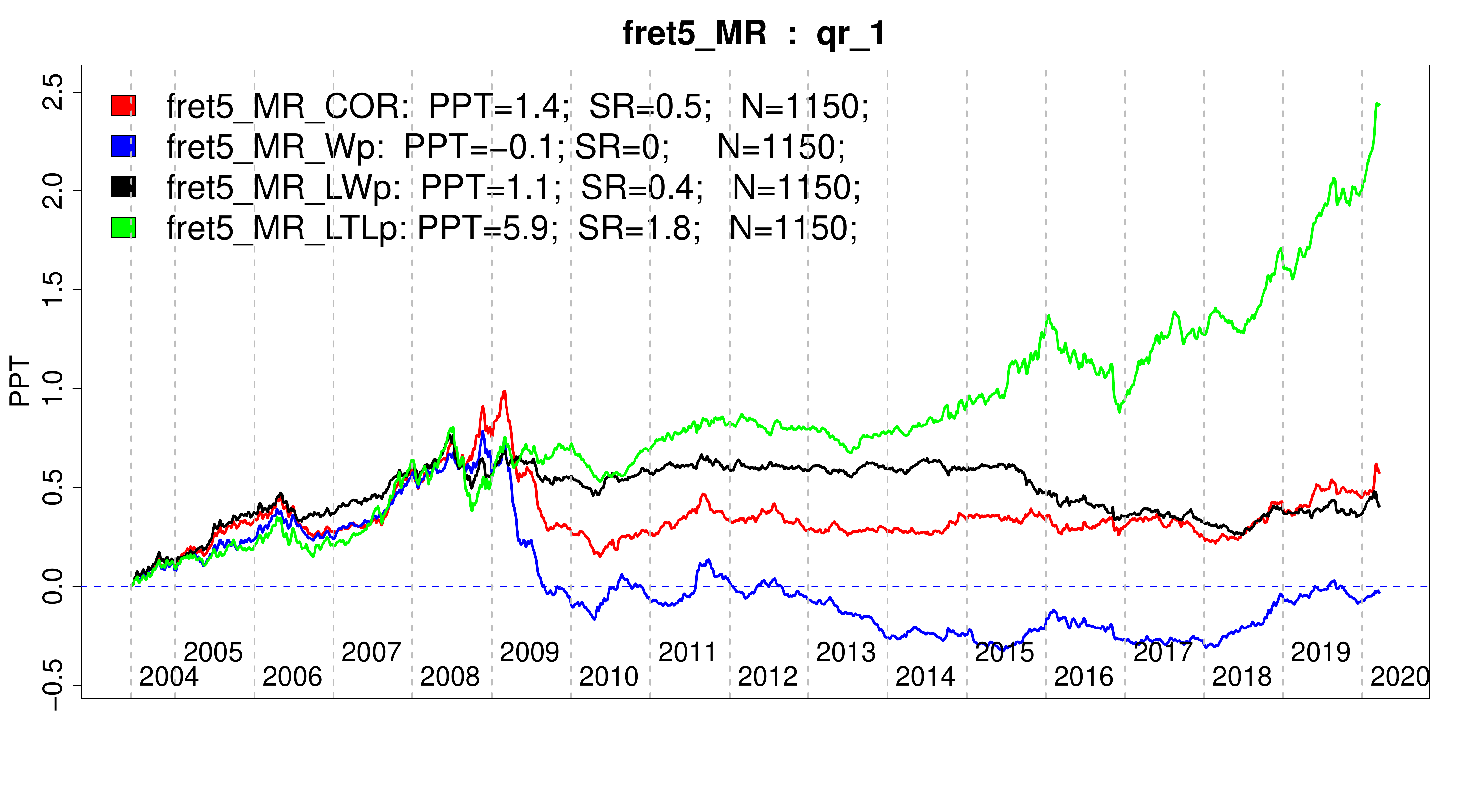}}  \\
\hline 
\rotatebox{90}{10-day}   
& \includegraphics[width=\widd]{{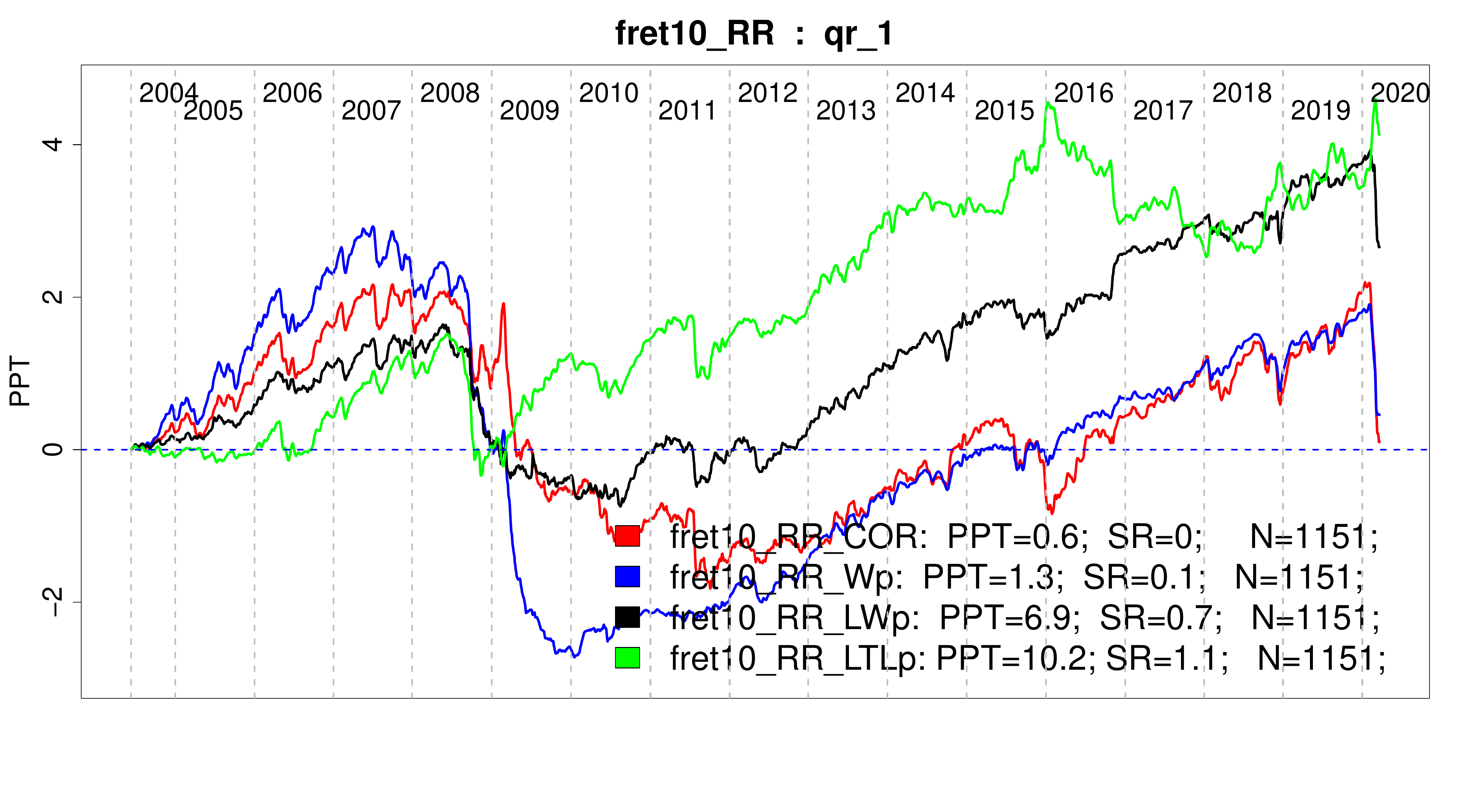}} 
& \includegraphics[width=\widd]{{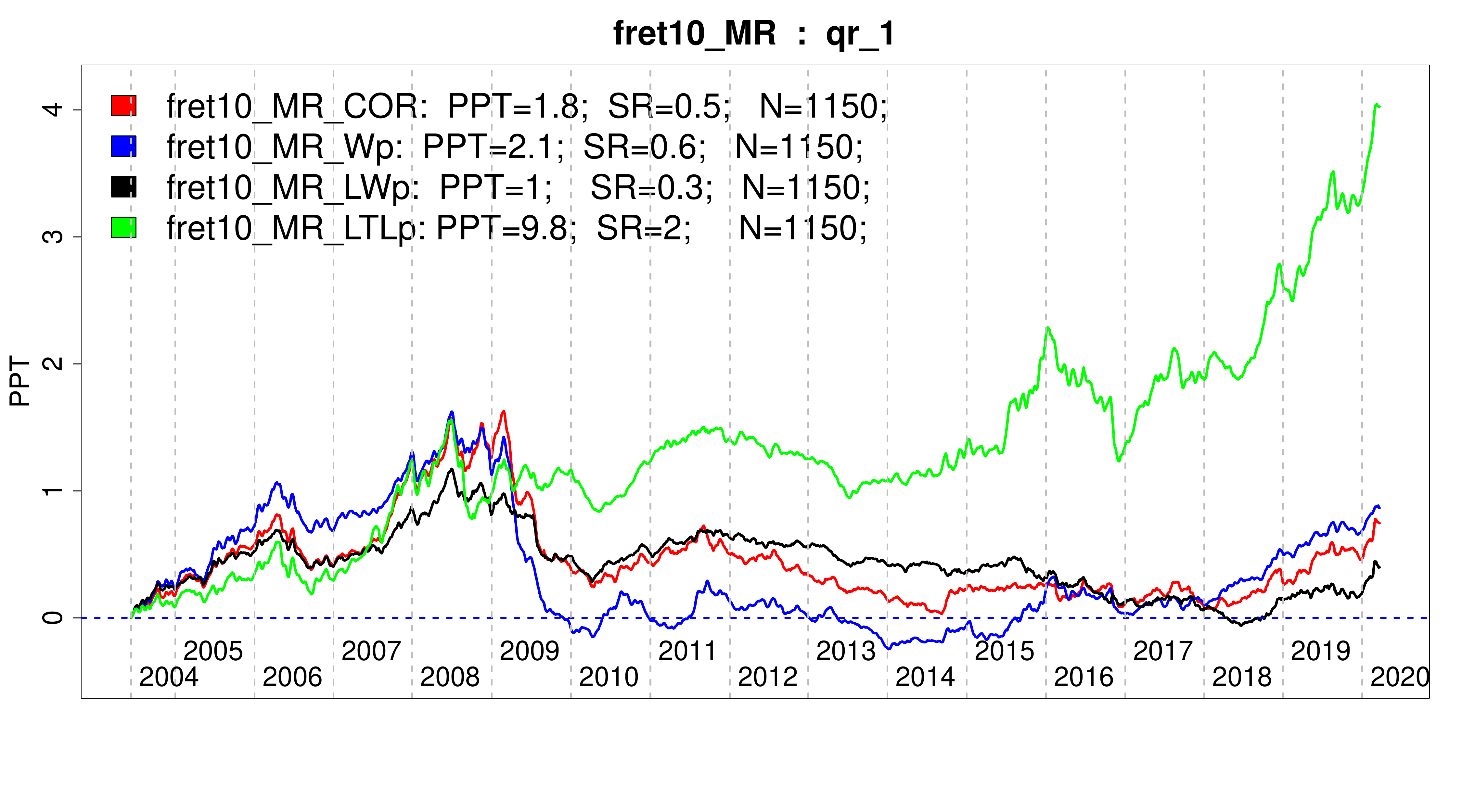}} \\ 
\hline 
\end{tabular}
\end{center}
\captionsetup{width=0.99\linewidth}
\caption{Cumulative PnL of the entire portfolio $qr_{1}$ for the  1,3,5,10-day future horizons, for raw returns and market-excess returns, across different methods. The legend contains performance statistics of each forecast method: PnL Per Trade (PPT) in basis points, Sharpe Ratio (SR), while $N$ denotes the size of the portfolio (since $qr_{1}$ contains the entire portfolio, this amount to $N \approx 1150$ stocks). Figure \ref{fig:PS_knn_100_MaxHist_Inf_EvoQR5} in the main text shows the top most quintile portfolio $qr_{5}$. 
}
\label{fig:PS_knn_100_MaxHist_Inf_evoQR1}	
\end{figure}

Figure \ref{figure:LiverCells} shows that on this particular task the $\LWp$ and $\LTLp$ frameworks have similar classification performance when differentiating between cancerous and normal cells. This is unsurprising as these images only have a single intensity channel. We also visualise, in a PCA plot, the $\LTLp$ embedding showing the variability across the dataset.
Figure \ref{figure:LiverCells} panel (c) visualises the modes of variation in image space for the first $5$ principal components of the data.

\subsection{Further Results on the Application to Financial Time Series}

We show here additional numerical results for the financial time series applications. 
Figure \ref{fig:PS_knn_100_MaxHist_Inf_evoQR1} shows the PnL curves corresponding the full portfolio of stocks, attained by each of the methods, for the  1,3,5,10-day future horizons, for both the raw returns and market-excess returns. As similarly observed earlier for the top quintile portfolio shown in Figure \ref{fig:PS_knn_100_MaxHist_Inf_EvoQR5}, $\LTLp$ outperforms all other methods at the longer horizons {3,5,10}, especially in the more realistic scenario of using market excess returns (corresponding to a hedged portfolio).
Figure \ref{fig:eigenLTLP} shows the top eigenvectors of the $\LTLp$ distance matrix, which localize on known crises, especially those during 2007-2008 and 2020.

\begin{figure}[h!]
\begin{center}
\includegraphics[width=0.62\columnwidth]{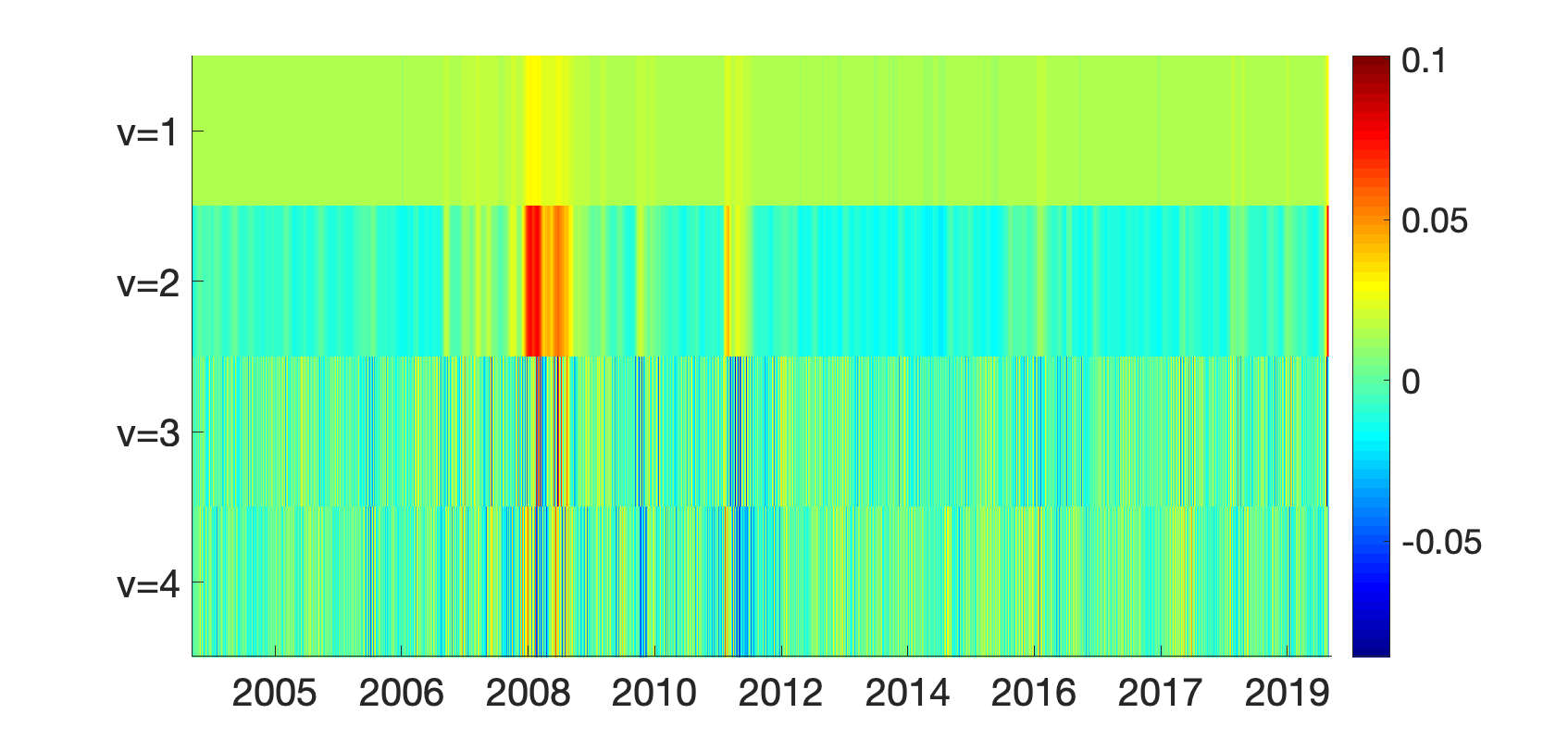}
\end{center}
\vspace{-2mm}
\caption{Barplot of the top $k=5$ eigenvectors of the $\LTLp$ distance matrix. The top eigenvector and eigenvalue typically correspond to the so-called market mode. The second eigenvector is strongly localized on the 2007-2008 financial crisis and the February-March 2020 Covid-19 crisis.}
\label{fig:eigenLTLP}
\end{figure}

\end{document}